\documentclass{article}

 \usepackage[preprint]{neurips_2026}

\usepackage{wrapfig}
\usepackage[utf8]{inputenc} 
\usepackage[T1]{fontenc}    
\usepackage{hyperref}       
\usepackage{url}            
\usepackage{booktabs}       
\usepackage{amsfonts}       
\usepackage{nicefrac}       
\usepackage{microtype}      
\usepackage{xcolor}         
\usepackage{titletoc}       
\usepackage{graphicx}
\usepackage{subcaption}
\usepackage{thmtools,thm-restate,mathtools,amsthm}
\usepackage{comment,algorithm}
\newtheorem{theorem}{Theorem}

\newtheorem{assumption}{Assumption}

\newtheorem{definition}{Definition}

\newenvironment{proofsketch}{%
  \proof}{\endproof}

\newcommand\independent{\protect\mathpalette{\protect\independenT}{\perp}}
    \def\independenT#1#2{\mathrel{\rlap{$#1#2$}\mkern2mu{#1#2}}}

\usepackage{tikz}

\newdimen\arrowsize
\pgfarrowsdeclare{arcsq}{arcsq}
{
	\arrowsize=0.2pt
	\advance\arrowsize by .5\pgflinewidth
	\pgfarrowsleftextend{-4\arrowsize-.5\pgflinewidth}
	\pgfarrowsrightextend{.5\pgflinewidth}
}
{
	\arrowsize=1.5pt
	\advance\arrowsize by .5\pgflinewidth
	\pgfsetdash{}{0pt} 
	\pgfsetroundjoin   
	\pgfsetroundcap    
	\pgfpathmoveto{\pgfpoint{0\arrowsize}{0\arrowsize}}
	\pgfpatharc{-90}{-140}{4\arrowsize}
	\pgfusepathqstroke
	\pgfpathmoveto{\pgfpointorigin}
	\pgfpatharc{90}{140}{4\arrowsize}
	\pgfusepathqstroke
}

\tikzset{
    block/.style={
        circle, draw, fill=white
        },
    myarrow/.style={
        single arrow,  
        draw, 
        single arrow head extend=2mm, minimum width=30pt,
        },    
    myar/.style={
        rounded corners=2pt, fill=black!20, 
        },
    mytri/.style={
        isosceles triangle, anchor=apex,
        isosceles triangle apex angle=90,
        minimum width=50pt
        },
    }

\title{Structural Extrapolated Data Generation}

\author{%
    Kun Zhang$^{\dagger,2,1}$ \quad
    Jiaqi Sun$^{\dagger,1,2}$ \quad
    Yiqing Li$^{2}$ \quad
    Ignavier Ng$^{1}$ \quad
    Namrata Deka$^{1}$ \quad
    Shaoan Xie$^{1,2}$ \\
    $^1$Carnegie Mellon University \\
    $^2$Mohamed bin Zayed University of Artificial Intelligence\\
    \texttt{\{kunz1,jiaqisun\}@cmu.edu} \\
}

\begin{document}

\maketitle

\footnotetext[2]{Equal contribution.}

\begin{abstract}
This paper aims to address the challenge of data generation beyond the training data and proposes a framework for Structural Extrapolated Data GEneration (SEDGE) based on suitable assumptions on the underlying data-generating process. We provide conditions under which  data satisfying novel specifications can be generated reliably, together with the approximate identifiability of the distribution of such data under certain ``conservative" assumptions, as well as the inherent non-identifiability of this distribution without such assumptions.  On the algorithmic side, we develop practical methods to achieve extrapolated data generation, based on a structure-informed optimization strategy or diffusion posterior sampling, respectively.  We verify the extrapolation performance on synthetic data and also consider extrapolated image generation as a real-world scenario to illustrate the validity of the proposed framework.
\end{abstract}
\section{Introduction}
\vspace{-2mm}
Generative AI, designed to generate new content in the form of images, text, audio, or videos, has seen major breakthroughs and attracted much attention in recent years. For instance, in the image domain, well-known generative AI tools, including the latent diffusion model \cite{Rombach2021HighResolutionIS}, DALL-E by OpenAI \cite{Ramesh2021ZeroShotTG}, Google's Imagen \cite{Saharia2022PhotorealisticTD}, can generate high-quality images in a controllable manner.

However, in many real-world problems,  we expect generative AI tools to be able to produce novel content and automate creative tasks in a reliable way. That is, they are required not only to interpolate within known regimes, but also to extrapolate beyond the support of available data. For instance, in AI-generated images for creative design, the desired content or visual specifications often lie outside the support of the training data. Similarly, in drug discovery, candidate molecules are optimized toward target properties to address new diseases or previously unseen combinations of diseases, for which no direct training examples exist.  In such scenarios, the ability to generate reliable extrapolated data is essential, but as a classical problem in machine learning, especially in deep learning, it still remains a fundamental challenge.

Extrapolated data generation differs in theory from standard data augmentation \cite{Shorten2019ASO} or interpolation (estimating unknown data points within the range of a given dataset).  Interpolation can be achieved, for instance, by making use of suitable smoothness assumptions on the function to be estimated from data; in contrast, extrapolation often requires a deeper understanding of the process behind the given data and applying it to novel scenarios. Otherwise, if one directly uses the learned flexible, deep data-generation model for extrapolation, the value of the input of the model is out of support of the training data, and the output of the model often does not have correctness guarantees. In fact, existing generative models, despite their expressive power, frequently fail when tasked with producing data that lie outside the training distribution \cite{gokhale2022benchmarking, luo2024text_mol, vatsa2025right}.

Researchers have developed several types of assumptions or constraints to deal with the extrapolated generation issues in machine learning.  A typical one is to directly constrain the functional class of the prediction function to be linear (or parametric) or with an additive form. Some other methods exploit compositional properties of generative functions to reduce instability during generation \cite{liu2022compositional_diffusion, wiedemer2023compositional, lachapelle2023additive}.  Some work derives extrapolation guarantees from regularities on distribution shifts, for instance, smoothness \cite{kong2024towards}. These developments do not provide a general framework for extrapolated data generation given novel specifications when properties of the involved functions are unknown. 

In this paper, we study extrapolated data generation from a principled structural perspective, without directly imposing rather strong assumptions on the underlying functions.
First, we provide an open analysis of two different data-generating processes, represented by graphical models, to formulate the relationship between {\it specifications} and {\it features} of the data to be generated: modeling specifications as conditioning variables that generate features, as commonly adopted in conditional generation frameworks (as seen in Fig. \ref{fig:two_structures}(a), in which $Z_i$ and $X_i$ denote specifications and features, respectively), or modeling the process as features generating or satisfying specifications (see Fig. \ref{fig:two_structures}(b)). We show that the latter structure, in which features point to specifications in the graphical representation, offers a more natural explanation for unseen data.
%
We further improve the graphical representation to
explain why the novel scenarios do not appear in the training data, which also enables a principled treatment to go beyond such data. 
Interestingly, this structural view reveals a key source of difficulty in extrapolation: features directly involved in multiple specifications are inherently harder to extrapolate, since each specification imposes additional constraints on the admissible data. In contrast, in scenarios where such shared features are absent, extrapolation becomes substantially more straightforward.

The contributions of this work are summarized as follows. 
1) We propose a structural view of the extrapolated data generation problem and conduct a systematic investigation to characterize the conditions under which valid extrapolated data can be generated and those that enable the whole distribution of the extrapolated data to be approximately constructed.
2) We develop practical algorithms to optimize for the extrapolated data or
sample from the constructed conditional distribution given the novel combinations of specifications. 
3) We further demonstrate the effectiveness of the proposed framework on both synthetic extrapolation tasks and extrapolated image generation.
\section{Initial Thoughts}\label{sec:motivation}
\vspace{-2mm}
\begin{wrapfigure}{r}{0.4\textwidth}
\centering
\vspace{-3mm}
\subcaptionbox{Specification generates features.}{
\resizebox{0.18\columnwidth}{!}{
\begin{tikzpicture}[scale=.65, line width=0.5pt, inner sep=0.1mm, shorten >=.1pt, shorten <=.1pt]
	\draw (0, 0) node(X1) [circle, draw]  {{\footnotesize\,${X}_1$\,}};
	\draw (-2*0.7, -1*0.7) node(Z1)[circle, draw]  {{\footnotesize\,${Z}_1$\,}};
    \draw (0, -2*0.7) node(X2) [circle, draw]  {{\footnotesize\,${X}_2$\,}};
	\draw (-2*0.7, -3*0.7) node(Z2)[circle, draw]  {{\footnotesize\,${Z}_2$\,}};
    \draw (0, -4*0.7) node(X3) [circle, draw]  {{\footnotesize\,${X}_3$\,}};
    
    \node at (2*0.7, -2*0.7) {}; 

	\draw[-arcsq] (Z1) -- (X1); 
	\draw[-arcsq] (Z1) -- (X2); 
	\draw[-arcsq] (Z2) -- (X2); 
	\draw[-arcsq] (Z2) -- (X3);
\end{tikzpicture}
}}
\hfill
\subcaptionbox{Features produce specifications.}{
\resizebox{0.18\columnwidth}{!}{
\begin{tikzpicture}[scale=.65, line width=0.5pt, inner sep=0.2mm, shorten >=.1pt, shorten <=.1pt]
	\draw (0, 0) node(X1) [circle, draw]  {{\footnotesize\,${X}_1$\,}};
	\draw (2*0.7, -1*0.7) node(Z1)[circle, draw]  {{\footnotesize\,${Z}_1$\,}};
    \draw (0, -2*0.7) node(X2) [circle, draw]  {{\footnotesize\,${X}_2$\,}};
	\draw (2*0.7, -3*0.7) node(Z2)[circle, draw]  {{\footnotesize\,${Z}_2$\,}};
    \draw (0, -4*0.7) node(X3) [circle, draw]  {{\footnotesize\,${X}_3$\,}};
    
    \node at (4*0.7, -2*0.7) {};

	\draw[-arcsq] (X1) -- (Z1); 
	\draw[-arcsq] (X2) -- (Z1); 
	\draw[-arcsq] (X2) -- (Z2); 
	\draw[-arcsq] (X3) -- (Z2);
\end{tikzpicture}
}}
\caption{The two generating processes as initial thoughts.}
\label{fig:two_structures}
\vspace{-4mm}
\end{wrapfigure}
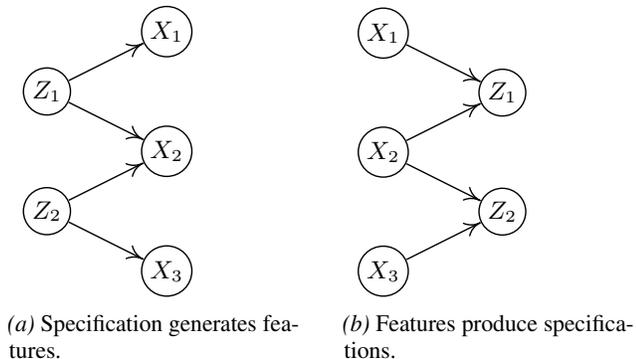


Let us use a simple example to illustrate the setting and our idea. Suppose we are required to generate a sensible image of a peacock eating ice cream, assuming that there is no such image in the training set.  For illustrative purposes, let us simply suppose it involves only two specfications, ``peacock" ($Z_1 = 1$) and ``eating ice cream" ($Z_2 = 1$).\footnote{In practice, this example may involve more specifications, such as ``eating" and ``ice cream"; in  this case, they are merged into a macro specification.}
As an initial idea, one may adopt the graphical model in Fig.~\ref{fig:two_structures}(a) to describe the dependence relations, where features $\mathbf{X}=(X_1,X_2,X_3)$ are generated by specifications $\mathbf{Z}=(Z_1,Z_2)$. Here, we may think of $X_2$ as indicating the presence of a “beak”. Under this model, given  a novel combination of $\mathbf{Z}$ values, say, $(Z_1 = 1, Z_2 = 1)$, which never happened in the training data,  unless strong assumptions are enforced on the functions (for instance, a parametrtic model is assumed for $p(\mathbf{x}\,|\,z_1, z_2))$, 
the  distribution of feature $X_2$, which is adjacent to both $Z_1$ and $Z_2$, is undefined, and hence we cannot find the distribution or values of $\mathbf{X}$  to satisfy the novel combination of $\mathbf{Z}$. In other words, the model does not know how to generate the feature “beak” for the unseen combination “peacock eating ice cream”, since this combination is not present in the training data.  With this graphical model, without additional assumptions on the functions mapping $\mathbf{Z}$ to $\mathbf{X}$, extrapolated data generation for novel specifications is not possible if $Z_i$ with novel value combinations are influenced by the same features.

In contrast, Fig.~\ref{fig:two_structures}(b) seems to provide a more reasonable explanation of the relationship between $\mathbf{X}$ and $\mathbf{Z}$, where features point to specifications. This structure implies the conditional independence $Z_1 \independent Z_2 \,|\, X_2$. For example, upon observing a “beak” ($X_2=1$), inferring whether the object is “eating ice cream” ($Z_2=1$) does not depend on whether it is a “peacock” ($Z_1=1$) or a chicken. Such a conditional independence relationship is not captured by the graphical model in Fig.~\ref{fig:two_structures}(a), which instead generally implies conditional dependence between $Z_1$ and $Z_2$ given $X_2$.
%
%
However, if one adopts the graphical representation in Fig.~\ref{fig:two_structures}(b) for the data-generating process, why is there no such scenario in the training data?  How can we go beyond the training data in a principled way?

Now we are ready to introduce our framework, which models the different distributions across the training data and the data to be generated, figures out their connections, and allows us to go from training data to the construction of new data following novel combinations of the specifications.


\section{SEDGE: Structural Extrapolated Data Generation}
\vspace{-2mm}
Let us now demonstrate how the structural properties can be leveraged for data generation under novel specifications. We first formalize the data generating process in Sec. \ref{sec:formulation}, and then study the case of two specifications in Sec. \ref{sec:two_specifications}, which serves to illustrate our key ideas of extrapolated data generation, including how to generate novel samples and recover the novel data distribution. Lastly, we present theoretical results, together with the involved assumptions, for the general setting with multiple specifications in Sec. \ref{sec:many_specifications}.

\subsection{Formulation}\label{sec:formulation}
\vspace{-1.5mm}
Let $\mathbf{X} = (X_1,\dots,X_d)$ denote the feature variables to be generated, and let $\mathbf{Z} = (Z_1,\dots,Z_n)$ denote the specification variables. Let $S$ be a selection variable indicating whether a sample is contained in the given dataset \cite{Heckman1979SampleSB}. A sample is contained in the given dataset if and only if $S=1$; otherwise, $S=0$. That is, the given training data correspond to the distribution $p(\mathbf{X} \,|\, S=1)$.

We assume that the variables $\mathbf{X}$, $\mathbf{Z}$, and $S$ form a Bayesian network with respect to a directed acyclic graph (DAG) $\mathcal{G}$~\citep{pearl1985bayesian,koller09probabilistic}. That is, the Markov property is satisfied: each variable is conditionally independent of its non-descendants in $\mathcal{G}$ given its parents in $\mathcal{G}$. Consistent with the discussion in Sec. \ref{sec:motivation}, we impose the following structural assumptions: (i) there are no edges from specifications $\mathbf{Z}$ to features $\mathbf{X}$, (ii) there are no edges among the specifications $\mathbf{Z}$, and (iii) selection variable $S$ has no children, and its parents are restricted to the specifications $\mathbf{Z}$. These assumptions formalize our idea of feature generating specification and ensure that the specification variables $\mathbf{Z}$ are conditionally independent of one another given the features $\mathbf{X}$.

Given data points observed as samples from distribution $p(\mathbf{X}\,|\, S=1)$, our goal is to generate samples from $p(\mathbf{X}\,|\, \mathbf{Z}=\mathbf{z})$, where the combination of the specifications values $\mathbf{Z}=\mathbf{z}$ might be novel, i.e., unseen in the given/training data. Note that if the combination of their values already appears in the training data, then this conditional distribution is already implied by the training data, since we have $p(\mathbf{X} \,|\, \mathbf{Z}=\mathbf{z}) = p(\mathbf{X} \,|\, \mathbf{Z}=\mathbf{z}, S=1)$, implied by the conditional independence between $S$ and $\mathbf{X}$ given $\mathbf{Z}$ according to our structural assumptions discussed above.

Our central question is therefore: can we generate new data that satisfy novel combinations of specifications, and, if yes, what conditions make it possible, and under what conditions can the corresponding conditional distributions of $\mathbf{X}$ be constructed?

In this section, we assume that both the features $\mathbf{X}$ and the specifications $\mathbf{Z}$ are given. Our framework, however, is general. For instance, the features $\mathbf{X}$ can represent variables from tabular data or latent representations learned from images, while the specifications $\mathbf{Z}$ can represent textual concepts extracted from text. We also assume that the specifications $\mathbf{Z}$ are binary, which is natural in many settings where specifications correspond to different concepts. The framework can be straightforwardly extended to the continuous case.  In practice, one may need to learn $\mathbf{X}$ and $\mathbf{Z}$  from data; their identifiability under our structural assumptions is discussed in Appendix \ref{Sec:identi_X_Z} and their estimation procedure will be introduced in Sec. \ref{Sec:estimation}.

\paragraph{Notations.}
We use $\mathbf{1}$ to denote an all-ones vector of appropriate dimension. To simplify notation, and without loss of generality, we treat $\mathbf{Z}=\mathbf{1}$ as the novel specification combination to which we aim to extrapolate in this section; the analysis extends straightforwardly to other specification values. We use uppercase letters to denote random variables and lowercase letters to denote their realizations. When the random variables are unambiguous from context, we omit variables in density functions, e.g., writing $p(\mathbf{x})$ instead of $p(\mathbf{X} =  \mathbf{x})$. We use $p^\mathcal{D}(\cdot)$ to denote the distribution implied by the given/training data, which satisfies $S=1$, e.g., $p^\mathcal{D}(\mathbf{X} \,|\, \mathbf{Z}=\mathbf{1})\coloneqq p(\mathbf{X}\,|\, \mathbf{Z}=\mathbf{1}, S=1)$. We occasionally abuse notation by treating a vector as the corresponding set of its entries when the meaning is clear from context. We also write $[n]\coloneqq \{1,\dots,n\}$ for integer $n$ and use $\propto$ to denote equality up to a positive multiplicative constant.


\input{}

\subsection{Extrapolation with Two Specifications}\label{sec:two_specifications}
\vspace{-1.5mm}
For illustrative purposes, we start with the case of two specification variables, $\mathbf{Z}=(Z_1,Z_2)$, to illustrate our key idea of leveraging structural properties to generate data that satisfy novel combinations of specifications. Later in Section~\ref{sec:many_specifications}, we extend this setting to general scenarios with more than two specification variables which can also be continuous. 

To be more precise about novel specifications, suppose we observe only $(Z_1=1, Z_2=0)$ and $(Z_1=0, Z_2=1)$, and possibly $(Z_1=0, Z_2=0)$, in the training data, and aim to generate sensible data $\mathbf{X}$ for $(Z_1 = 1, Z_2 = 1)$, i.e., $p(Z_1=1, Z_2=1 \,|\, S=1) = 0$.
The novel combination has positive probability (and thus is feasible) in the underlying population, i.e., $p(Z_1=1, Z_2=1) > 0$. If each specification value $Z_1=1$ and $Z_2=1$ has never appeared in any form in the given data, i.e., $p(Z_i=1 \,|\, S=1)= 0$, it is impossible to generate the corresponding samples, since $p(\mathbf{x}\,|\, Z_i=1)$ is completely unknown. It is thus necessary to assume that each specification value has been observed in some combination within the dataset. These conditions are formalized in the following assumption.
\begin{assumption}[Basic conditions for specifications and selection]\label{assumption:given_specifications_toy}
We have
\begin{flalign}
P(S=1) > 0,~~~
P(Z_1=1,Z_2=1) > 0,~~~
\textrm{and~}P(Z_i=1 \,|\, S=1)> 0,~~i=1,2.\label{eq:assumption_given_specifications_toy_1}
\end{flalign}
\end{assumption}

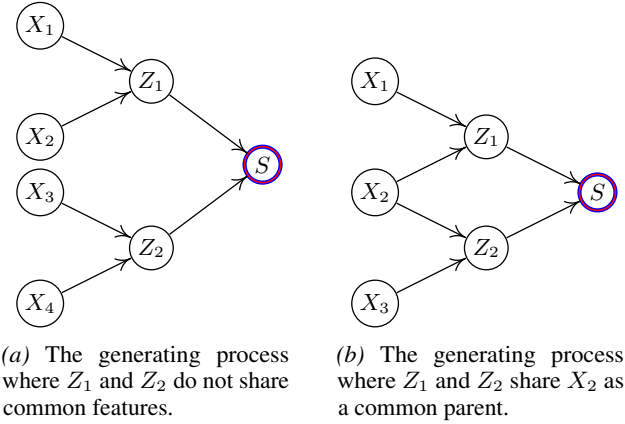
\begin{wrapfigure}{r}{0.4\textwidth}
\centering
\vspace{-3mm}
\subcaptionbox{The generating process where $Z_1$ and $Z_2$ do not share common features.}{
\resizebox{0.18\columnwidth}{!}{
\begin{tikzpicture}[scale=.65, line width=0.5pt, inner sep=0.1mm, shorten >=.1pt, shorten <=.1pt]
	\draw (0, 0) node(2) [circle, draw]  {{\footnotesize\,${X}_1$\,}};
	\draw (2*0.7, -1*0.7) node(1)[circle, draw]  {{\footnotesize\,${Z}_1$\,}};
    \draw (0, -2*0.7) node(3) [circle, draw]  {{\footnotesize\,${X}_2$\,}};
	\draw (2*0.7, -3*0.7-1) node(4)[circle, draw]  {{\footnotesize\,${Z}_2$\,}};
    \draw (0, -2*0.7-1) node(5) [circle, draw]  {{\footnotesize\,${X}_3$\,}};
    \draw (4*0.7, -2*0.7 - 0.5) node(6)[draw=blue,double=red, circle, inner sep=1pt]{{\footnotesize\,${S}$\,}};
    \draw (0, -4*0.7 - 1) node(7) [circle, draw]  {{\footnotesize\,${X}_4$\,}};
	\draw[-arcsq] (2) -- (1); 
	\draw[-arcsq] (3) -- (1); 
	\draw[-arcsq] (7) -- (4); 
	\draw[-arcsq] (5) -- (4);
    \draw[-arcsq] (4) -- (6);
    \draw[-arcsq] (1) -- (6);
\end{tikzpicture}
}}
\hfill
\subcaptionbox{The generating process where $Z_1$ and $Z_2$ share $X_2$ as a common parent.}{
\resizebox{0.18\columnwidth}{!}{
\begin{tikzpicture}[scale=.65, line width=0.5pt, inner sep=0.1mm, shorten >=.1pt, shorten <=.1pt]
	\draw (0, 0) node(2) [circle, draw]  {{\footnotesize\,${X}_1$\,}};
	\draw (2*0.7, -1*0.7) node(1)[circle, draw]  {{\footnotesize\,${Z}_1$\,}};
    \draw (0, -2*0.7) node(3) [circle, draw]  {{\footnotesize\,${X}_2$\,}};
	\draw (2*0.7, -3*0.7) node(4)[circle, draw]  {{\footnotesize\,${Z}_2$\,}};
    \draw (0, -4*0.7) node(5) [circle, draw]  {{\footnotesize\,${X}_3$\,}};
    \draw (4*0.7, -2*0.7) node(6)[draw=blue,double=red, circle, inner sep=1pt]{{\footnotesize\,${S}$\,}};
	\draw[-arcsq] (2) -- (1); 
	\draw[-arcsq] (3) -- (1); 
	\draw[-arcsq] (3) -- (4); 
	\draw[-arcsq] (5) -- (4);
	\draw[-arcsq] (1) -- (6);
    \draw[-arcsq] (4) -- (6);
\end{tikzpicture}
}}
\caption{The generating processes where features generate the specification and the observed data satisfy the constraint that the selection variable $S = 1$.}
\vspace{-0.8em}
\label{fig:setting_toy_two_specifications}
\vspace{-4mm}
\end{wrapfigure}


\subsubsection{Identifiable Case: Specifications do not Share Common Features} 
Let us first consider the generating process in Fig.  \ref{fig:setting_toy_two_specifications}(a). in which the specifications $Z_1$ and $Z_2$ do not share common features. It is worth noting the framework also applies when the features generating the same specification variable, such as $X_1$ and $X_2$, which generate $Z_1$, have dependence relations in between. 
In this case, one can successfully recover the distribution with novel specifications, $p(\mathbf{X} \,|\, Z_1=1,Z_2=1)$, from the given data, as seen below.


\begin{restatable}{proposition}{PropositionRecoverDistributionNoCommonFeaturesToy}\label{proposition:recover_distribution_no_common_features_toy}
Suppose that Assumption \ref{assumption:given_specifications_toy} holds and the data generating process follows Fig. \ref{fig:setting_toy_two_specifications}(a). Then, the novel data distribution $p(\mathbf{X}  \,|\, Z_1 = 1, Z_2 = 1)$ is identifiable from the given data.
\end{restatable}

A proof is given in Appendix A.2. The idea is to make use of specific conditional independence relations implied by the structural properties of Fig. \ref{fig:setting_toy_two_specifications}(a), where the specifications do not share common features, i.e., $Z_1\independent Z_2$. 

Note that after recovering the novel data distribution $p(\mathbf{X} \,|\, Z_1=1, Z_2=1)$, one can generate data points with the novel specifications $(Z_1 = 1, Z_2 = 1)$, thus achieving extrapolated data generation.

\subsubsection{Non-Identifiable Case: Specifications Share Common Features} 
We now consider the scenario in Fig. \ref{fig:setting_toy_two_specifications}(b), where the specifications $Z_1$ and $Z_2$ share a common parent $X_2$. This case is more complex than the previous scenario in Fig. \ref{fig:setting_toy_two_specifications}(a), where specifications do not share common features, because here $X_2$ is a direct parent of both $Z_1$ and $Z_2$.

\paragraph{Non-identifiability of novel data distribution.}
We show in the following proposition that the distribution $p(\mathbf{X} \,|\, Z_1 = 1, Z_2 = 1)$ is, in general, non-identifiable for the generating process in Fig. \ref{fig:setting_toy_two_specifications}(b). 
A full proof is given in Appendix \ref{app:proof_proposition_non_identifiability_distribution_toy}.

\begin{restatable}{proposition}{PropositionNonIdentifiabilityDistributionToy}\label{proposition:non_identifiability_distribution_toy}
Suppose that Assumption \ref{assumption:given_specifications_toy} holds and the data generating process follows Fig. \ref{fig:setting_toy_two_specifications}(b). Then, the novel data distribution $p(\mathbf{X} \,|\, Z_1 = 1, Z_2 = 1)$ is not identifiable from the given data.
\end{restatable}

Intuitively, this is because $p(\mathbf{X}\,|\, Z_1=1,Z_2=1)$ involves distributions conditioning on $S=0$, which are unknown.

\paragraph{Existence of novel data sample.}
Although the distribution corresponding to novel specifications, $p(\mathbf{X}\,|\, Z_1 = 1, Z_2 = 1)$, is not identifiable, one may still ask whether it is possible to generate samples that satisfy these specifications. Surprisingly, this is indeed possible under mild conditions. We begin by stating the following assumption.


\begin{assumption}[Common features in separate specifications]\label{assumption:exist_sample_toy}
There exists value $\tilde{x}_2$ of $X_2$ such that
\[
p^\mathcal{D}(X_2 = \tilde{x}_2 \,|\, Z_i=1) > 0, \quad i=1,2.
\]
\end{assumption}
This assumption requires that there exists at least one value of the common feature $X_2$ that has a positive density under each specification individually in the given data. Intuitively, it ensures that the marginal distributions of the specifications overlap in the space of the shared feature, making it possible to generate samples satisfying novel combinations of $Z_1$ and $Z_2$. Under this assumption, we obtain the following result, which shows that one can construct data points satisfying the novel specifications from the given data. 
A full proof deferred to Appendix \ref{app:proof_proposition_extrapolation_sample_toy}.

\begin{restatable}{proposition}{PropositionExtrapolationSampleToy}\label{proposition:extrapolation_sample_toy}
Suppose Assumptions \ref{assumption:given_specifications_toy} and \ref{assumption:exist_sample_toy} hold, and that the data generating process follows Fig. \ref{fig:setting_toy_two_specifications}(b). Then, there exists a value $\tilde{\mathbf{x}}$ of $\mathbf{X}$ that can be constructed from the given data such that 
\[
p(\mathbf{X}=\tilde{\mathbf{x}} \,|\, Z_1=1, Z_2=1) > 0.
\]
\end{restatable}

Intuively, it is because it can be derived that
\begin{small}\begin{equation*}
    p(\mathbf{X}\,|\, Z_1 = 1, Z_2 = 1) = \frac{p^\mathcal{D}(X_2\,|\, Z_1 = 1)p^\mathcal{D}(X_2\,|\, Z_2 = 1)}{p(X_2)}  \cdot p^\mathcal{D}(X_1\,|\, X_2, Z_1 = 1) p^\mathcal{D}(X_3\,|\,  X_2,  Z_2 = 1),
\end{equation*} \end{small}
and all the terms involved above are positive for $\mathbf{X}=\tilde{\mathbf{x}}$. 
Specifically, given estimates of $p^\mathcal{D}(X_3 \,|\, X_2, Z_2 = 1)$, $p^\mathcal{D}(X_1  \,|\, X_2, Z_1 = 1)$, and $p^\mathcal{D}(X_2 \,|\, Z_i = 1),i=1,2$, all of which can be learned from the given/training data, one can explicitly construct a data point $\tilde{\mathbf{x}}$ that satisfies the novel specifications.

Moreover, if the relevant density functions are further assumed to be continuous, then one can construct a set of data points with nonzero Lebesgue measure satisfying the inequality in Proposition \ref{proposition:extrapolation_sample_toy} (or Assumption \ref{assumption:exist_sample_toy}), since it will hold throughout an open neighborhood of the point $\tilde{\mathbf{x}}$~\citep{royden2018real}.

\paragraph{Conservative solutions for non-identifiable distributions.}
We also propose a way to circumvent the inherent non-identifiability issue with $p(\mathbf{X}\,|\, Z_1=1, Z_2=1)$ described in Proposition \ref{proposition:non_identifiability_distribution_toy}. Let us assume that 
$p(S=0\,|\, Z_1=1)$ and $p(S=0)$ are very small, which means that in the original data generating process, the chance of having novel scenarios, which are not contained in the distribution of the given data, is low.
\begin{assumption}\label{assumption:conservative_distribution_toy}
We have
$P(S=0 \,|\, Z_i=1)\approx 0$, $i=1,2$, and $P(S=0)\approx 0$.
\end{assumption}
With the assumption above, the novel data distribution $p(\mathbf{X}\,|\, Z_1=1, Z_2=1)$ becomes approximately identifiable, in the sense that it is approximately equal to a distribution constructed from certain distribution components implied by the training data, formally stated below.
\begin{restatable}{proposition}{PropositionRecoverConservativeDistributionToy}\label{proposition:recover_conservative_distribution_toy}
Suppose  Assumptions \ref{assumption:given_specifications_toy} and \ref{assumption:conservative_distribution_toy} hold. Then, the novel data distribution $p(\mathbf{X} \,|\, Z_1 = 1, Z_2 = 1)$ is approximately identifiable from the given data.
\end{restatable}

The proof, provided in Appendix \ref{app:proof_proposition_recover_conservative_distribution_toy}, is constructive and provides a way to find the distribution $p(\mathbf{X}\,|\, Z_1 = 1, Z_2 = 1)$ under Assumption \ref{assumption:conservative_distribution_toy}. We call this solution a \emph{conservative solution} since it does not enforce assumptions on the distribution of the (potential) data that were not in the given dataset. 
For instance, the calculation of $p(X_2)$ involves $p(X_2 \,|\, S= 0)$, which is unknown as $S=0$ corresponds to (potential) data that were not included in the given dataset. Fortunately, it can be ignored by assuming that such data are actually rare in conservative solutions, i.e., $P(S=0)\approx 0$.

\subsection{SEDGE with Multiple Specifications}\label{sec:many_specifications}
\vspace{-1.5mm}

We now extend the results from the previous section to show how structural properties can be exploited for extrapolated data generation in the setting with multiple specifications. Due to page limits, we provide a summary of the theoretical extension here, and readers may refer to the details in Appendix~\ref{Sec:multiple}.

We first consider the multi-specification setting in which specifications can be partitioned into disjoint subsets that do not share features. Under the extended basic coverage conditions (Assumption 4 in Appendix), which are analogous to Assumption 1 in the two-specification setting, we know that every specification subset appears with positive probability under selection. The novel distribution is therefore identifiable by recovering each component separately and recombining them; see Theorem 1 in Appendix \ref{Sec:multiple}.

When specifications share common features, $p(\mathbf{X}\mid\mathbf{Z}=\mathbf{1})$ is generally not identifiable. Nevertheless, analogous to the two-specification setting, extrapolated data points can still be constructed. Leveraging conditional independence between specification subsets given the shared features, together with a positivity condition on the shared features, one can combine observed samples to form feature vectors satisfying the novel specification combination. The existence of valid novel data points is therefore guaranteed; see Theorem 2 in Appendix \ref{Sec:multiple}.

Finally, analogous to the conservative recovery result in the two-specification case, non-identifiability can be addressed similarly to Proposition 4. Assuming that the probability of unobserved scenarios is low, the novel distribution $p(\mathbf{X}\mid\mathbf{Z}=\mathbf{1})$ becomes approximately identifiable and can be estimated from the given data, as seen from Theorem 3 in Appendix \ref{Sec:multiple}.

\vspace{-1.5mm}
\paragraph{Overview of the theory.}
To enable extrapolated data generation for novel combinations of specifications, our theory identifies three key conditions.
First, specification variables $Z_i$ are conditionally independent from each other given the features $\mathbf{X}$, allowing joint satisfaction to decompose into individual constraints.
Second, novelty arises only from unseen specification combinations, but not their separate values: every value of the specification $Z_i$ must already appear in the training data.
Third, the mapping from features $\mathbf{X}$ to specifications $\mathbf{Z}$ should be sparse, so that as few features as possible are shared across specifications. This reduces potential conflicts among specifications and facilitates sampling from the constructed distribution; in the extreme case of no shared features (Fig.~\ref{fig:setting_toy_two_specifications}(a)), extrapolation is straightforward.
All the results above can be directly generalized to continuous case.
In the next section, we will rely on these conditions to design practical algorithms for extrapolated data generation.

\vspace{-1.5mm}
\paragraph{On the assumptions.} Please note that an extensive discussion of the plausibility of the involved assumptions is given in Appendix \ref{app:discussion}. In addition, In Appendix~\ref{Sec:F1}, we illustrate how the results of the proposed approaches on synthetic data are robust to violations of each assumption, and we also report ablation study results.

\section{Model Estimation and  Data Generation}
\label{sec: method}
\vspace{-2mm}
We now describe how to estimate the conditional distribution $p(\mathbf{Z}\,|\,\mathbf{X})$ from data as constraint to be required to be satisfied for extrapolated data generation and how to efficiently generate samples under novel combinations of specifications under different data availability settings. 


\subsection{Model Estimation} \label{Sec:estimation}
\vspace{-1.5mm}

\paragraph{Specifications and features are given.}
When both $\mathbf{X}$ and $\mathbf{Z}$ are given (i.e., tabular data), we estimate each conditional distribution term $p^{\mathcal{D}}(Z_i|\mathbf{V}_i)$ separately. We denote $\mathbf{V}_i \subseteq \mathbf{X}$ as the parent set of specification $Z_i$. We assume each specification $Z_i$ follows a Gaussian distribution specified by $\mathbf{V}_i$, such that $Z_i \sim \mathcal{N}(h^{\mu}_i(\mathbf{V}_i), h^{\sigma}_i(\mathbf{V}_i))$, which yields the closed-form log-likelihood: $\log p^{\mathcal{D}}(Z_i = z_i\,|\,\mathbf{V}_i = \mathbf{v}_i) = -\frac{1}{2}\|\frac{Z_i - h^{\mu}_i(\mathbf{V}_i)}{h^{\sigma}_i(\mathbf{V}_i)}\|^2_2 - \frac{1}{2} \log h^{\sigma}_i(\mathbf{v}_i) -\frac{1}{2}\log 2\pi$. Both $h_i^{\mu}$ and $h_i^{\sigma}$ are parameterized by flexible Multi-Layer Perceptrons (MLPs), with $\sigma_i$ constrained to be non-negative. When the parent sets $\mathbf{V}_i$ are known, the networks $h_i^{\mu}$ and $h_i^{\sigma}$ are restricted to only take $\mathbf{V}_i$ as input, allowing the model to exploit the potential sparsity in the  structure. The objective function is
\begin{align}
\mathcal{L}_{g} =
\mathbb{E}_{\mathbf{x}, \mathbf{z}\sim p^{\mathcal{D}}(\mathbf{X},\mathbf{Z})}
\bigg[-\sum_{i\in[n]}\log p^{\mathcal{D}}(Z_i = z_i\,|\, \mathbf{V}_i = \mathbf{v}_i)\bigg],
\label{eq: loss likelihood}
\end{align}
where the subscript $g$ of $\mathcal{L}_{g}$ indicates that both features and specifications are given.

If the parent sets are unknown, we introduce learnable masks $\mathbf{M}_i$ to make the links between specification $Z_i$ and features as sparse as possible. In this case, we first adapt $\mathcal{L}_g$ to encourage sparsity: 
\begin{align}
\mathcal{L}_{m} =
\mathbb{E}_{\mathbf{x,z}}
\bigg[-\sum_{i\in[n]}\log p^{\mathcal{D}}_{Z_i|\mathbf{X}_i}(Z_i = z_i\,|\,\mathbf{V}_i = \mathbf{x} \odot \mathbf{M}_i)\bigg], 
\label{eq: loss likelihood with mask}
\end{align}
where the subscript for the summation is simplified (same as follows). The full objective becomes $\mathcal{L}_{g^{\prime}} = \mathcal{L}_{m} + \beta \mathcal{L}_s$, where $\beta$ controls the strength of the $L_1$ sparsity constraint $ \mathcal{L}_s = \sum_{i\in[n]}\|\mathbf{M}_i\|_1$. 
\paragraph{Specifications and features are not given.}
For modalities such as images or text, where $\mathbf{X}$ and $\mathbf{Z}$ are latent, we estimate them from the respective observations $\mathbf{Y_X}$ and $\mathbf{Y_Z}$. We employ variational autoencoders and define a generic VAE loss for an observation $\mathbf{a}$ and its estimated latent $\mathbf{b}$ as
\begin{align}
\mathcal{L}_{\mathrm{VAE}}(\mathbf{a})
=
- \mathbb{E}_{q_{\phi}(\mathbf{b}\mid \mathbf{a})}\!\left[\log p_{\theta}(\mathbf{a}\mid \mathbf{b})\right] + \mathrm{KL}\!\left(q_{\phi}(\mathbf{b}\mid \mathbf{a}) \,\|\, p(\mathbf{b})\right),
\end{align}
where $q_{\phi}(\mathbf{b}\mid \mathbf{a})$ is the encoder, $p_{\theta}(\mathbf{a}\mid \mathbf{b})$ is the decoder, and $p(\mathbf{b})$ is a prior over latents. We then define the combined objective as
\begin{align}
\mathcal{L}_{r}
=
\mathbb{E}_{\mathbf{Y_X},\mathbf{Y_Z}}
\Big[
\mathcal{L}_{\mathrm{VAE}}(\mathbf{Y_X})
+
\mathcal{L}_{\mathrm{VAE}}(\mathbf{Y_Z})
\Big].
\label{eq:vae_recon_sum}
\end{align}

Reconstruction alone does not enforce the conditional independence and sparsity that facilitate the generation of extrapolated data as we conclude in the end of previous section. These properties are imposed by combining the likelihood (as it assumes conditional independence) and sparsity terms: $\mathcal{L}_{ng} = \mathcal{L}_{r} + \alpha \mathcal{L}_{m} + \beta \mathcal{L}_{s}$, where $\alpha$ adjusts the strength for the negative log-likelihood term and $\beta$ controls the weight for the sparsity constraint. 

One may wonder whether the solutions for $\mathbf{X}$ and $\mathbf{Z}$ are asymptotically unique and, under our conditions of extrapolation, whether the underlying true representations, if they exist, are identifiable. We defer a formal discussion of identifiability and related conditions to Appendix~\ref{app: identify feat and specs}.
\begin{figure*}
    \centering
    \includegraphics[width=0.98\linewidth]{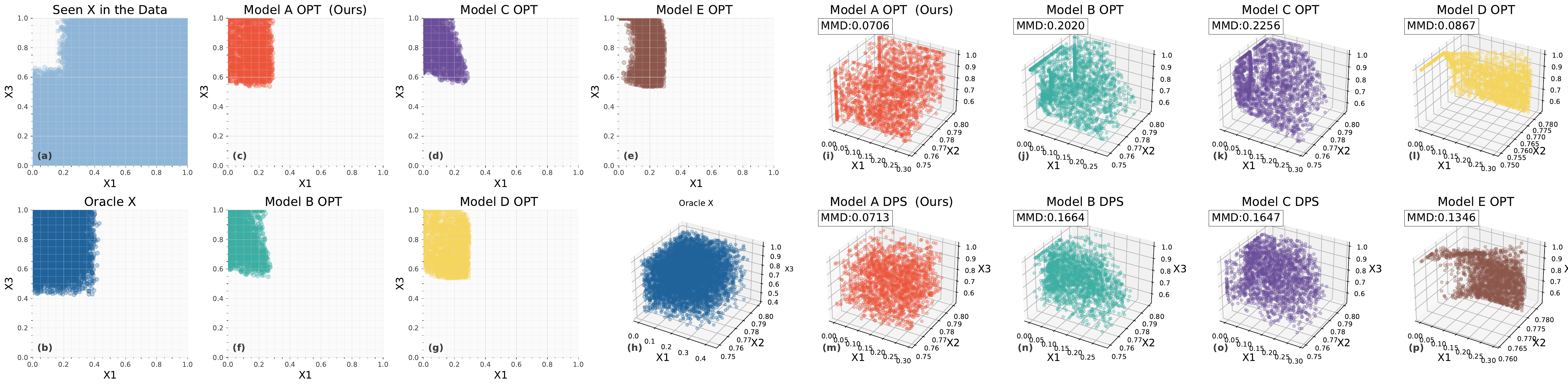}
    \caption{Synthetic experiment results with given $\mathbf{X}$ and $\mathbf{Z}$. Panels (a–b) illustrate the data split in a two-dimensional view over the $X_1$ and $X_3$ axes. The novel specifications induce a previously unseen joint distribution over $(X_1, X_3)$, ensuring that successful performance requires true extrapolation rather than interpolation. Panels (c–h) present the optimization-based generation (OPT) results for all five models. Panel (h) shows the three-dimensional view of the Oracle $\mathbf{X}$. Panels (i–p) visualize the three-dimensional generated $\mathbf{X}$ using OPT and diffusion posterior sampling (DPS) for models A, B, and C, and OPT for models D and E, with the corresponding MMD values to the Oracle $\mathbf{X}$ reported.}
    \label{fig:synthetic} 
\end{figure*}

\begin{figure*}[h]
  \centering
  \setlength{\tabcolsep}{1pt}      
  \renewcommand{\arraystretch}{0.7}  
 
  \newcommand{\imgw}{0.099\linewidth}

\begin{tabular}{@{}*{9}{c}@{}}
{SANA} & {Aligner} & SD3.5-L & QwenImage & Z-Image & GPT5.2 & w/o Spar & w/o CI & \textbf{SEDGE} \\
\addlinespace[1pt]
\hline

\addlinespace[1pt]
\multicolumn{9}{c}{%
\scriptsize\texttt{a peacock is eating ice cream}
} \\
\addlinespace[1pt]

\includegraphics[width=\imgw]{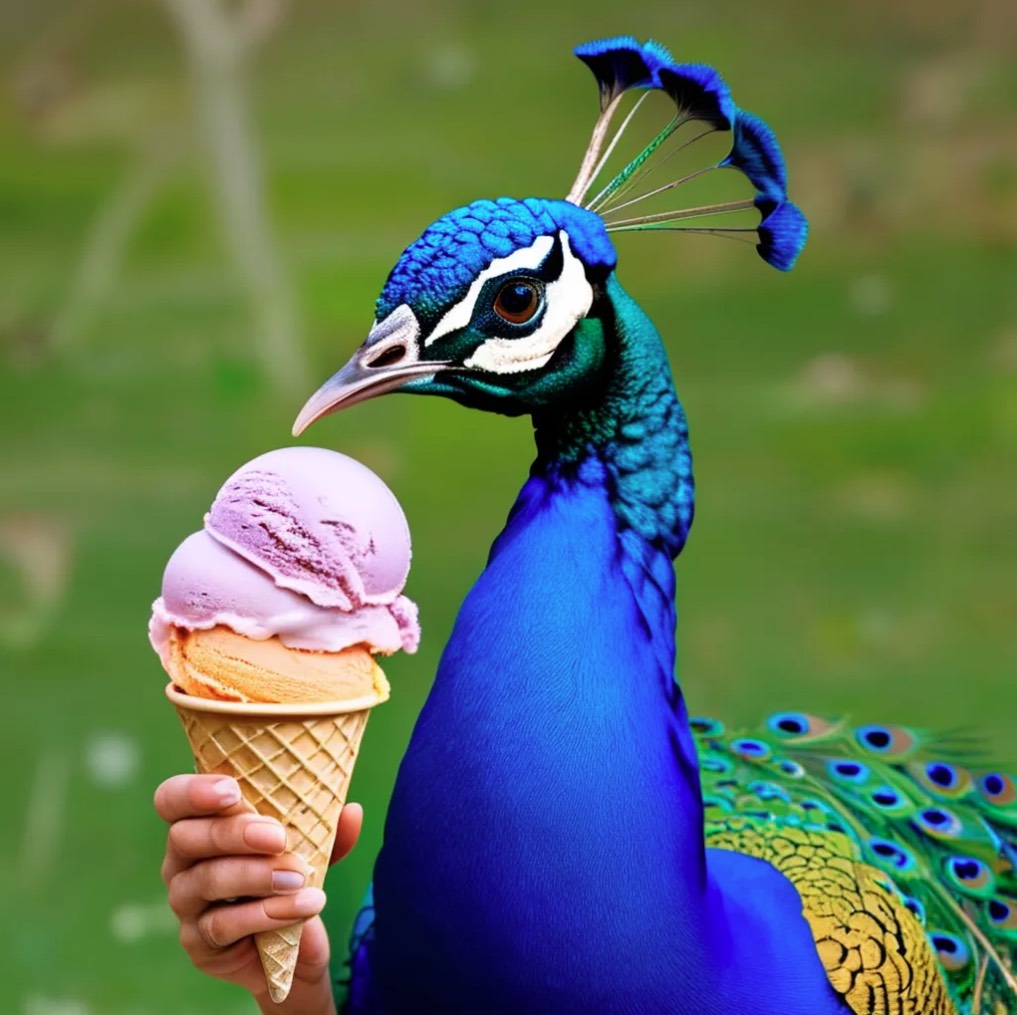} &
\includegraphics[width=\imgw]{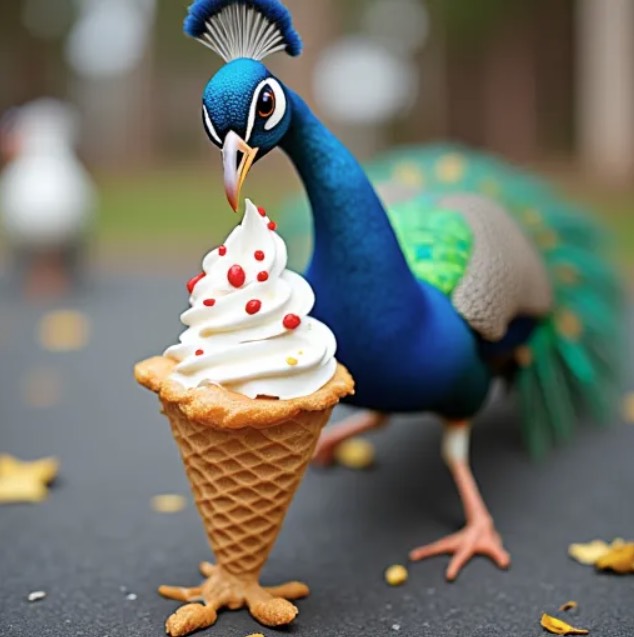} &
\includegraphics[width=\imgw]{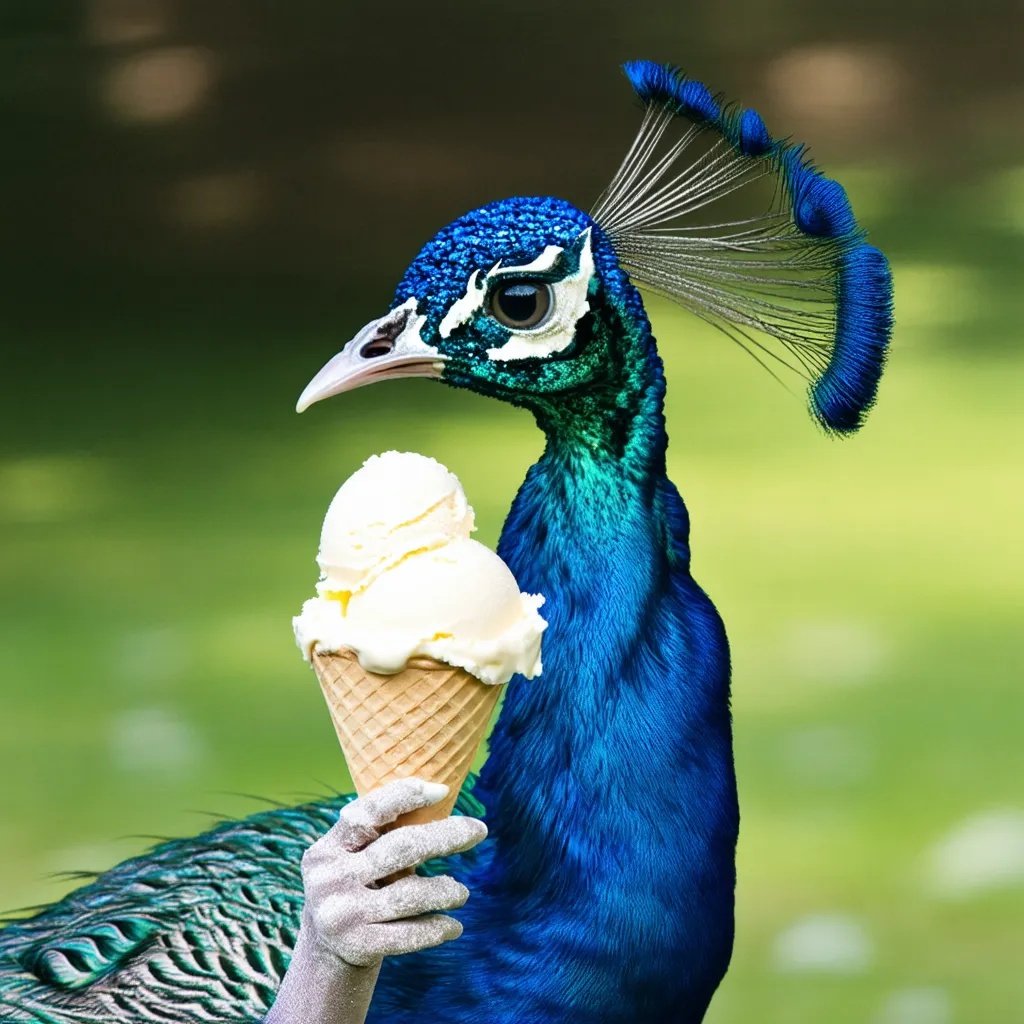} &
\includegraphics[width=\imgw]{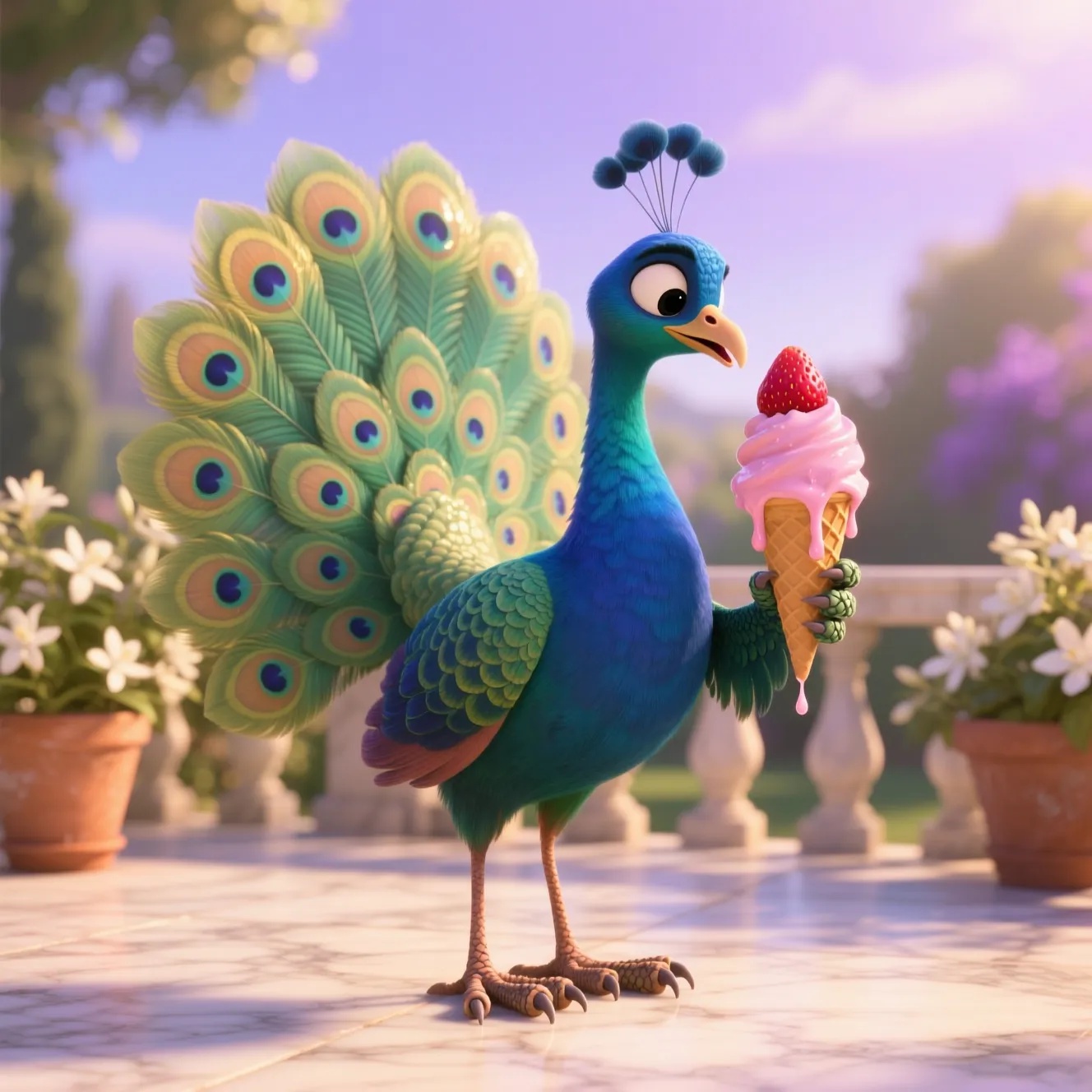} &
\includegraphics[width=\imgw]{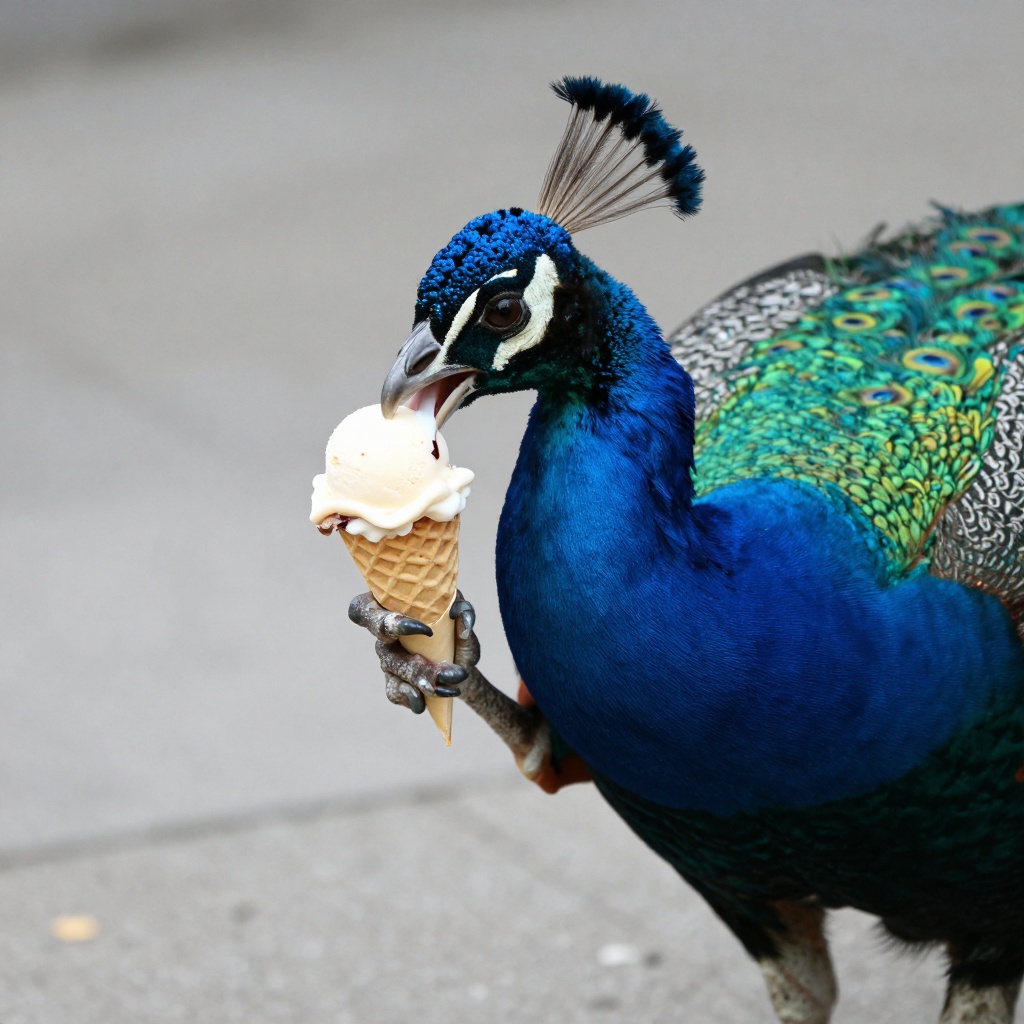} &
\includegraphics[width=\imgw]{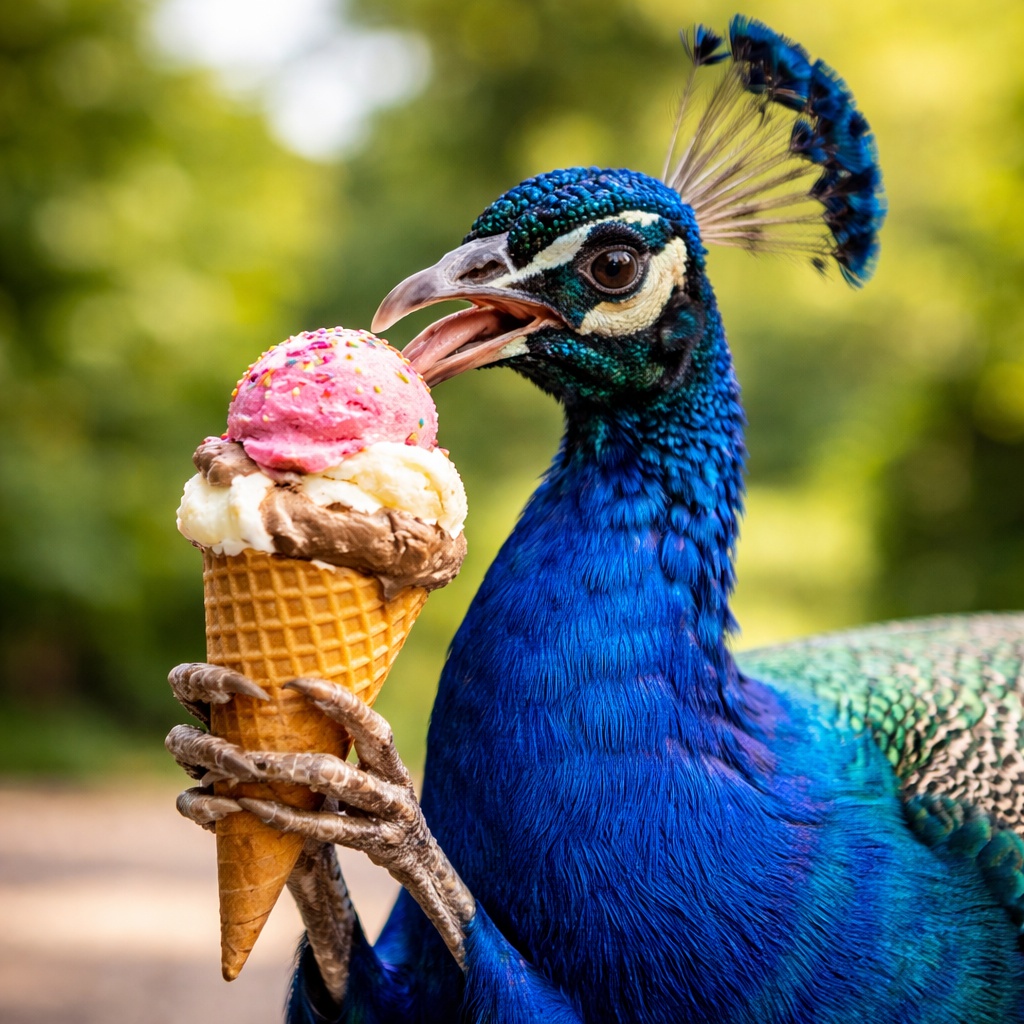} &
\includegraphics[width=\imgw]{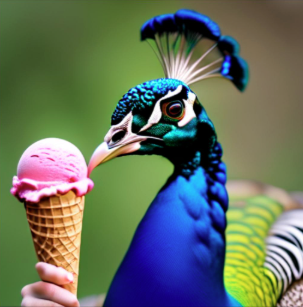} &
\includegraphics[width=\imgw]{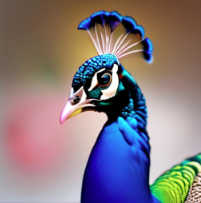} &
\includegraphics[width=\imgw]{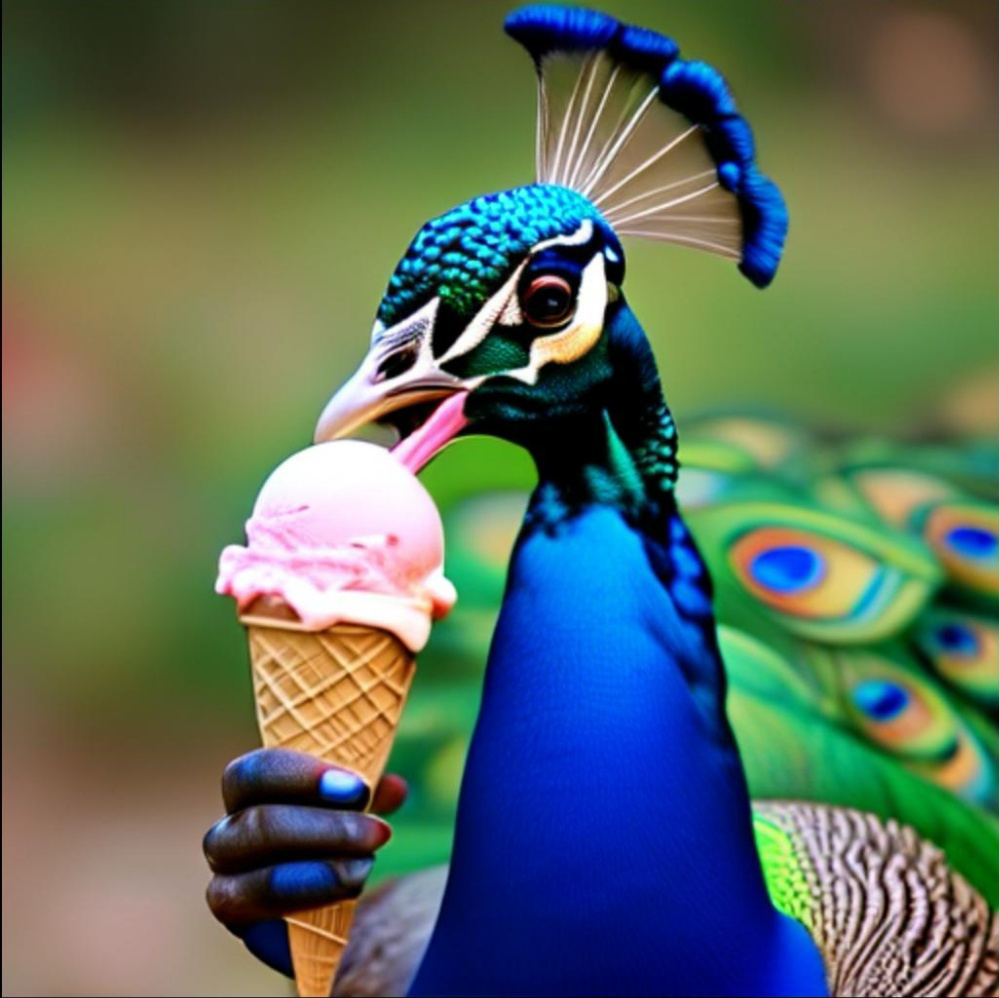} 
\\

\addlinespace[1pt]
\multicolumn{9}{c}{%
\scriptsize\texttt{a suitcase with bird wings}
} \\
\addlinespace[1pt]

\includegraphics[width=\imgw]{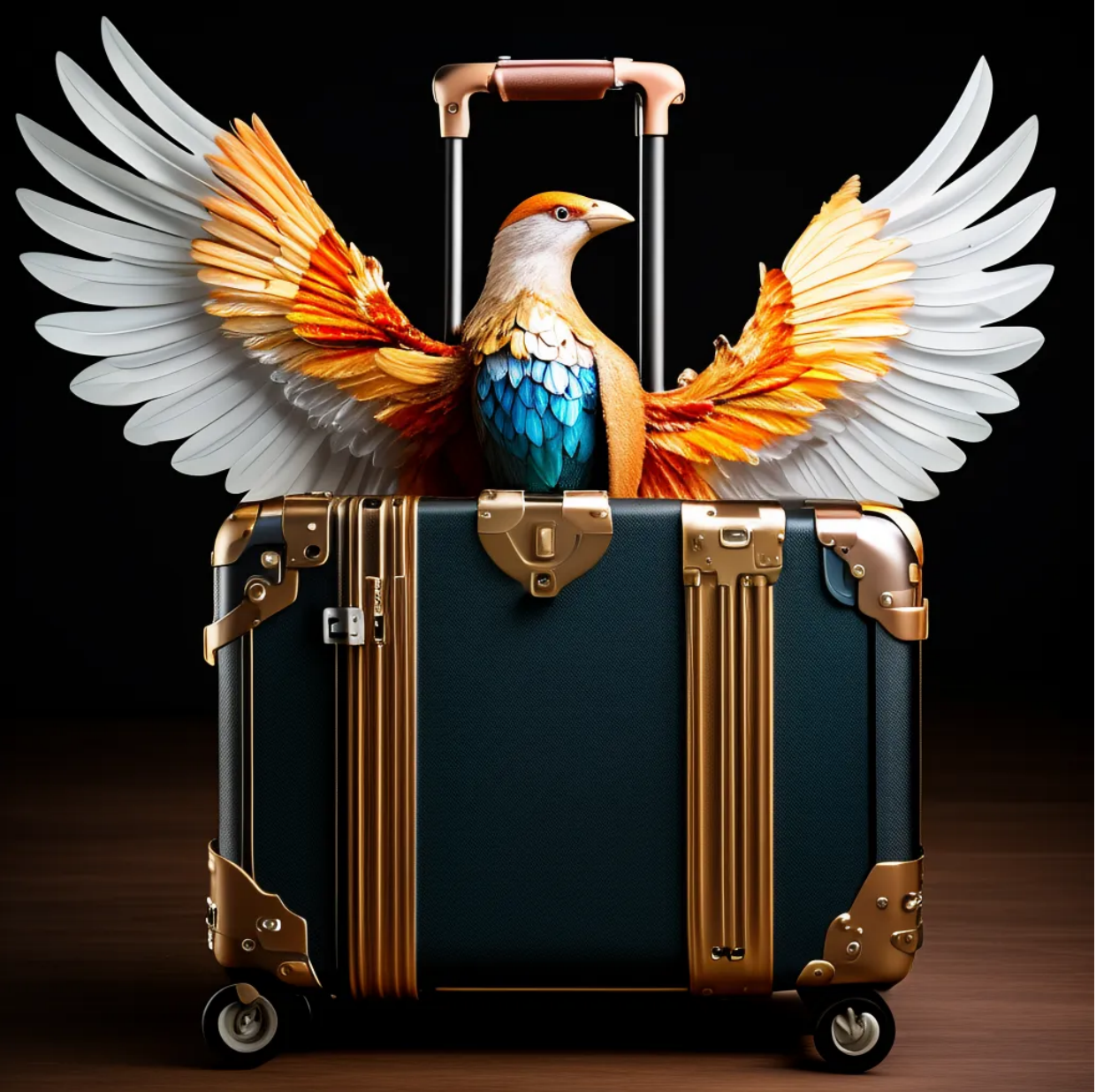} &
\includegraphics[width=\imgw]{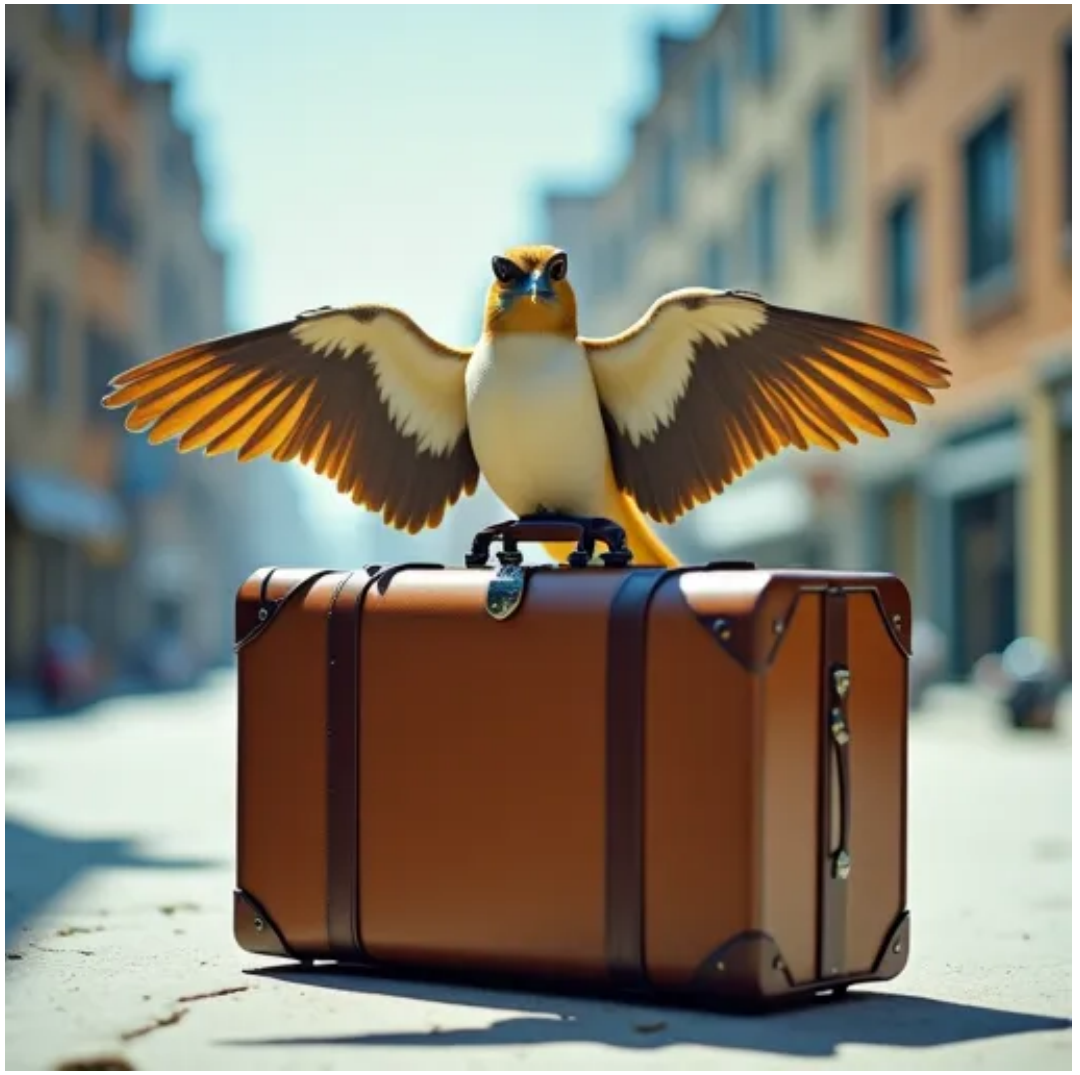} &
\includegraphics[width=\imgw]{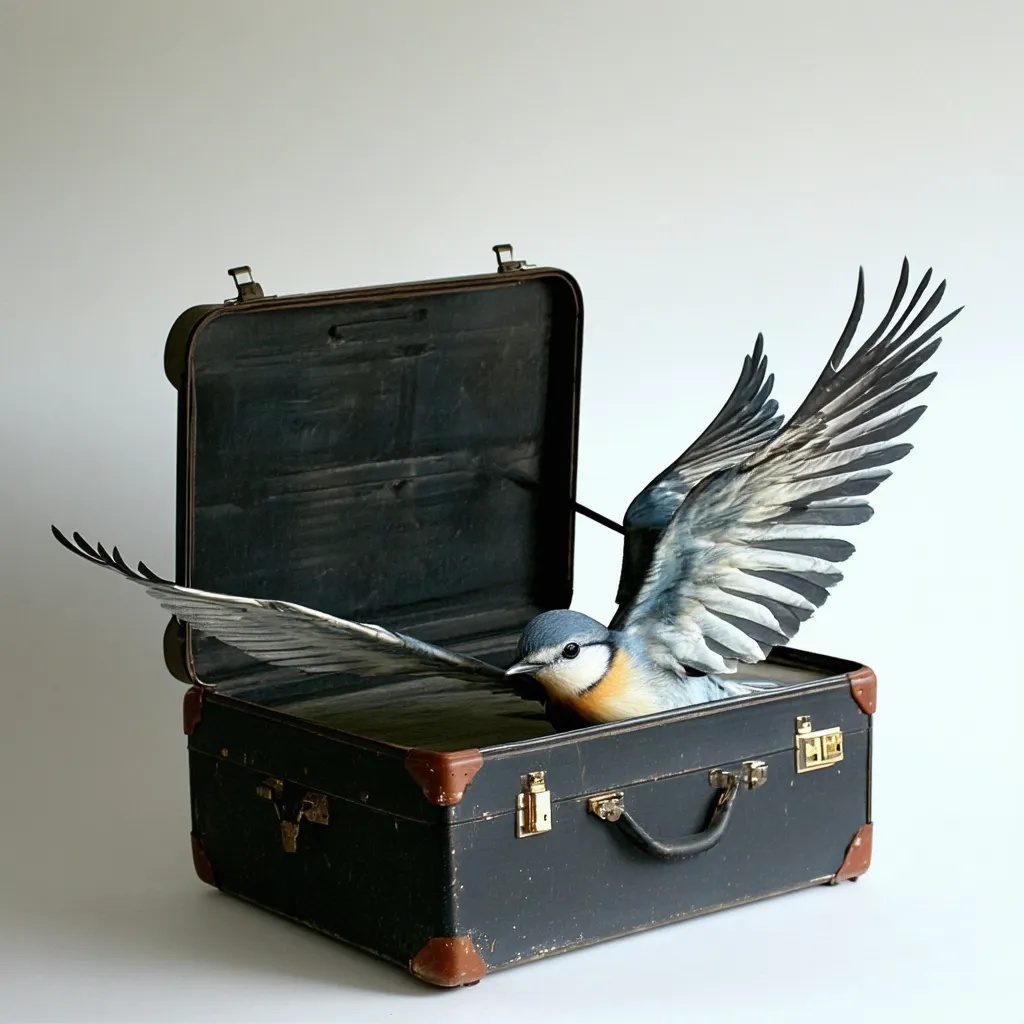} &
\includegraphics[width=\imgw]{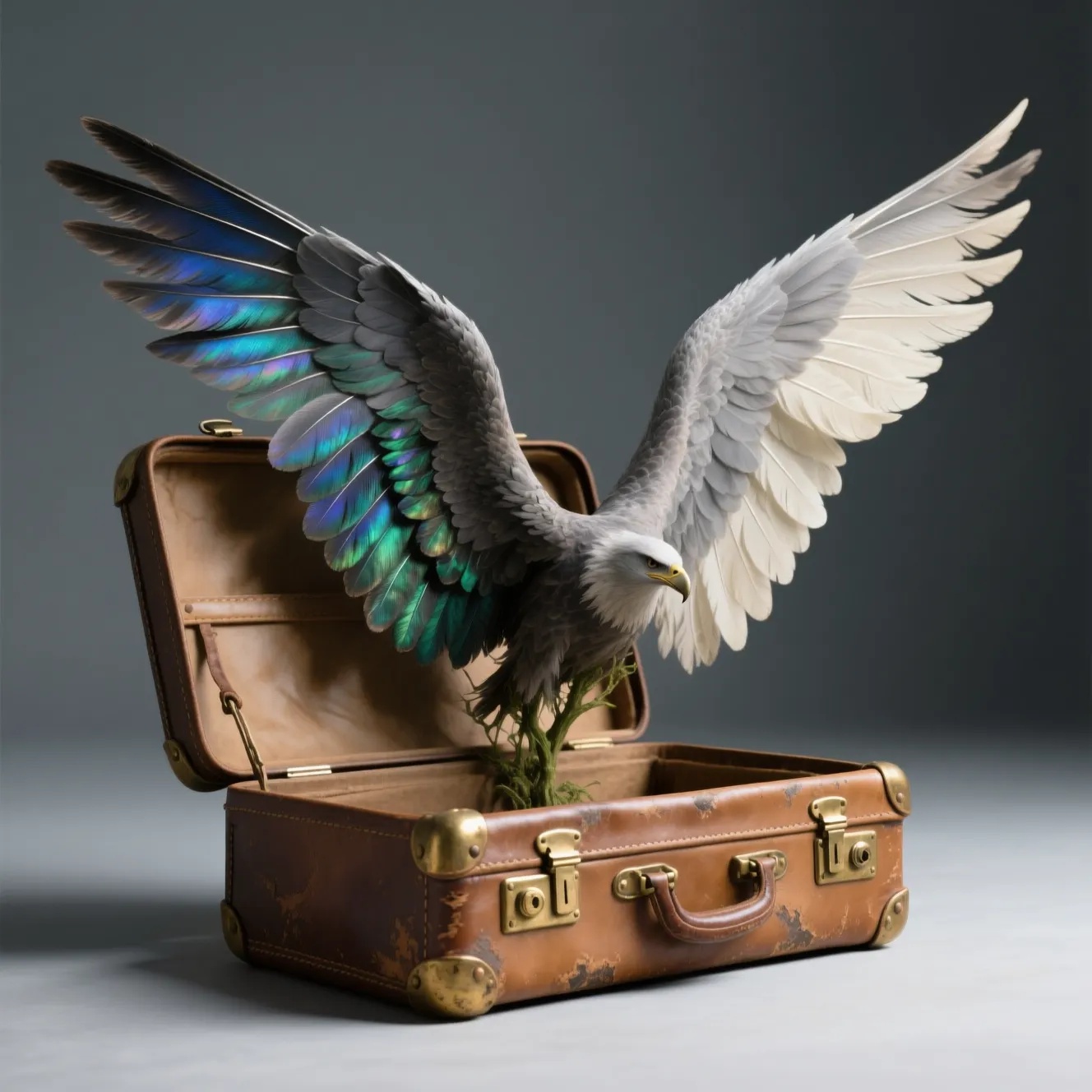} &
\includegraphics[width=\imgw]{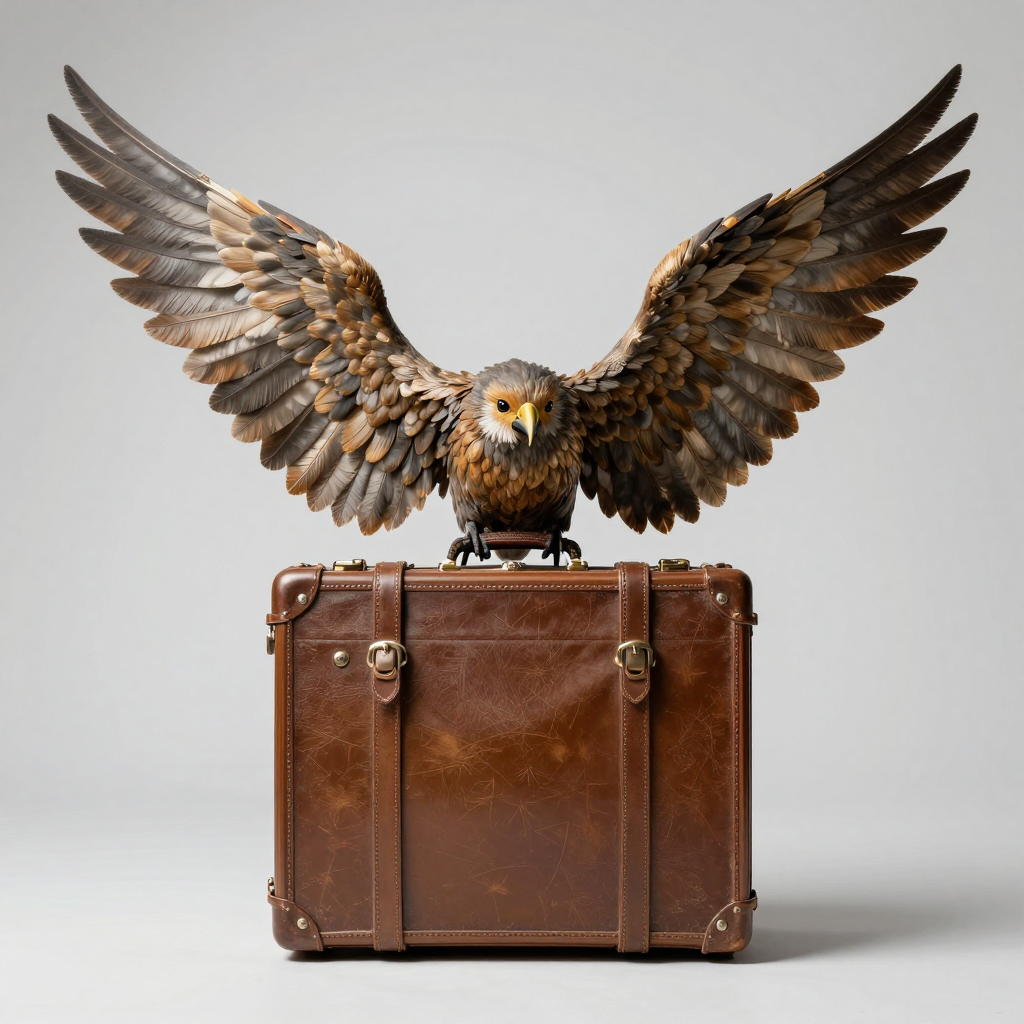} &
\includegraphics[width=\imgw]{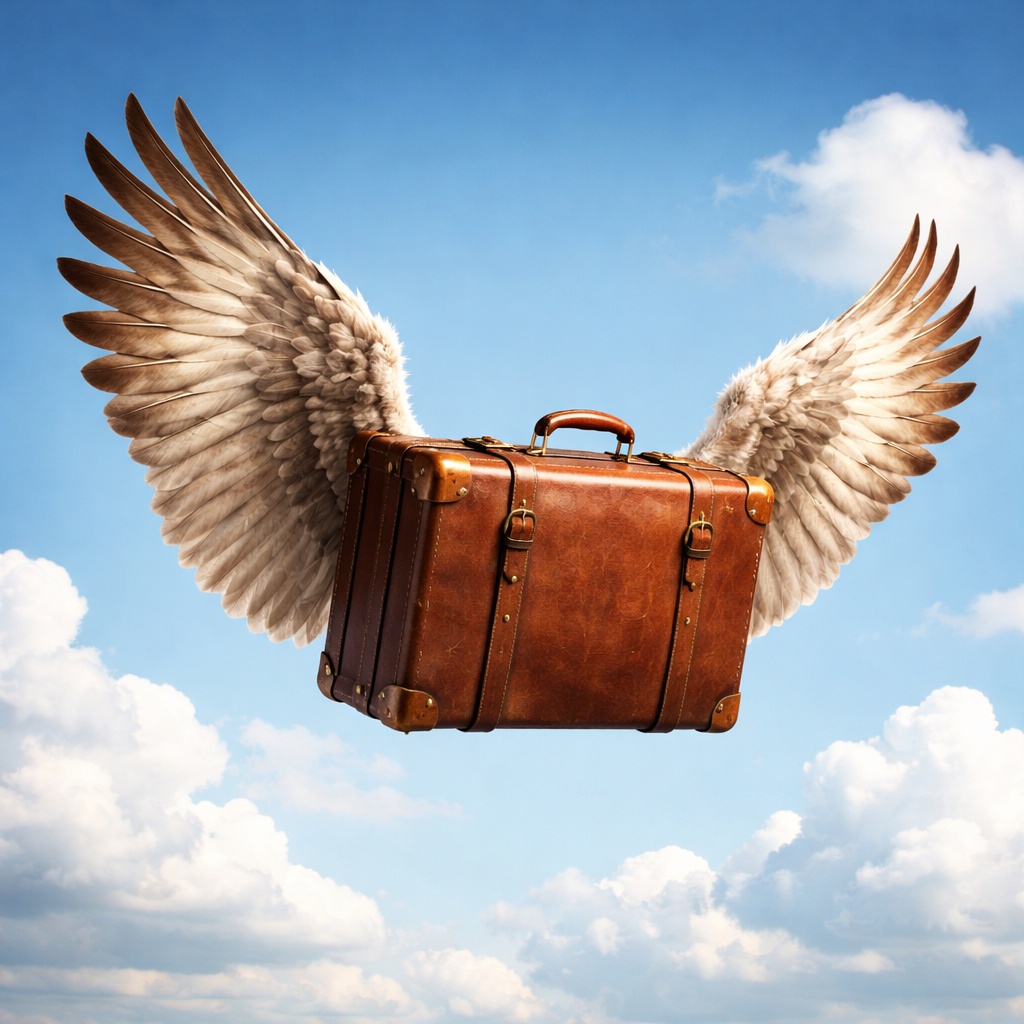} &
\includegraphics[width=\imgw]{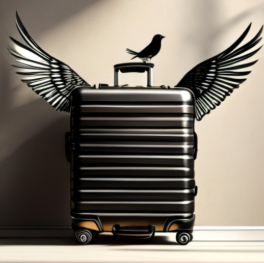} &
\includegraphics[width=\imgw]{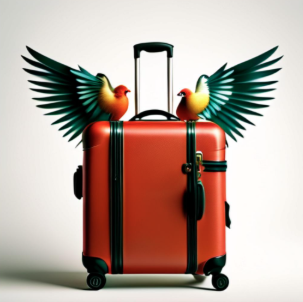} &
\includegraphics[width=\imgw]{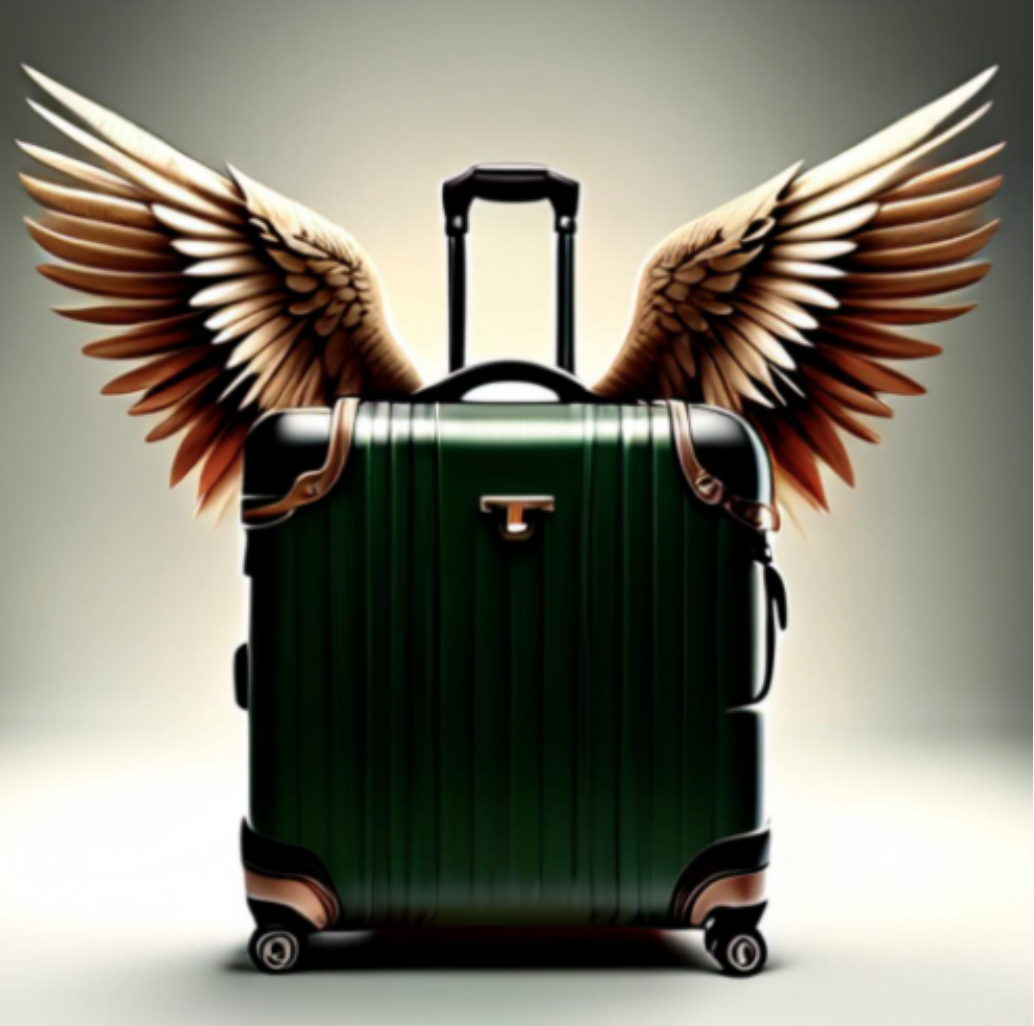} 
\\

\addlinespace[1pt]

\multicolumn{9}{c}{%
\scriptsize\texttt{a camera with insect antennae}
} \\
\addlinespace[1pt]

\includegraphics[width=\imgw]{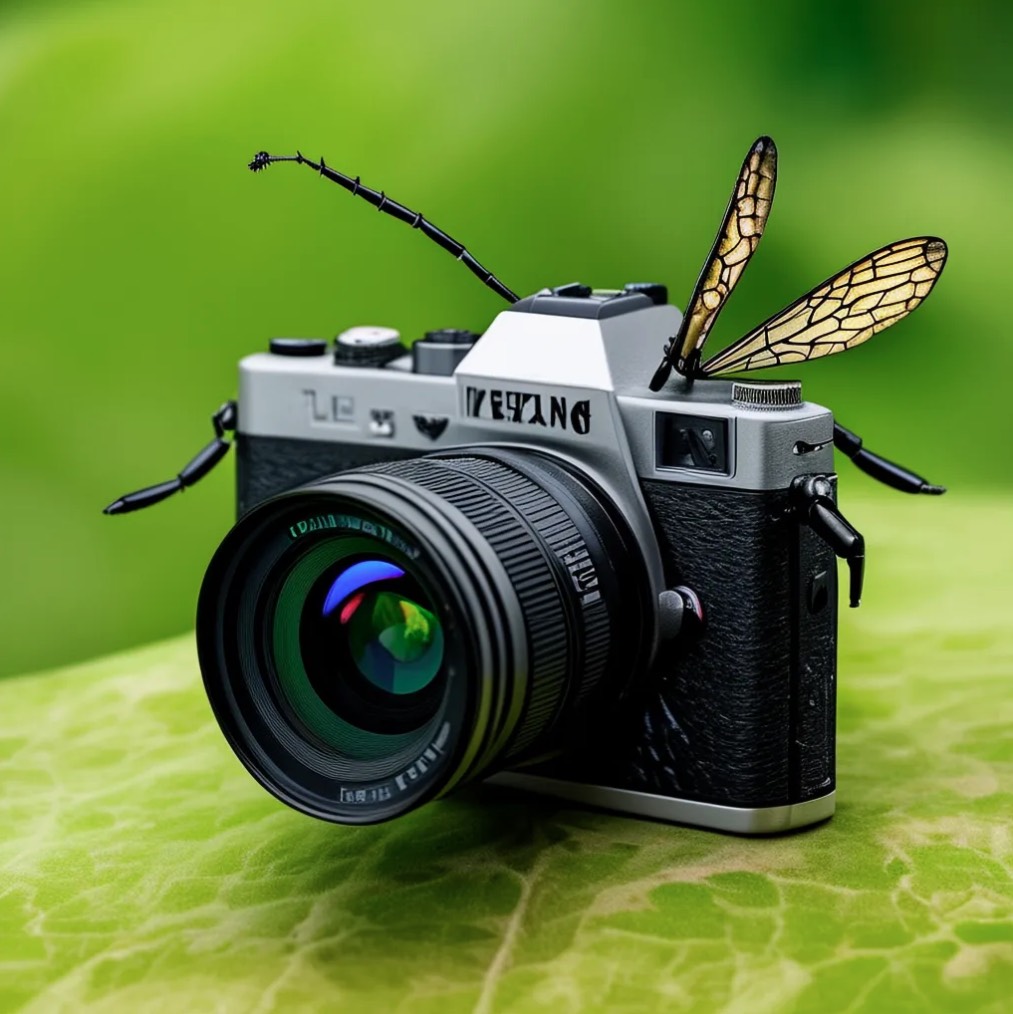} &
\includegraphics[width=\imgw]{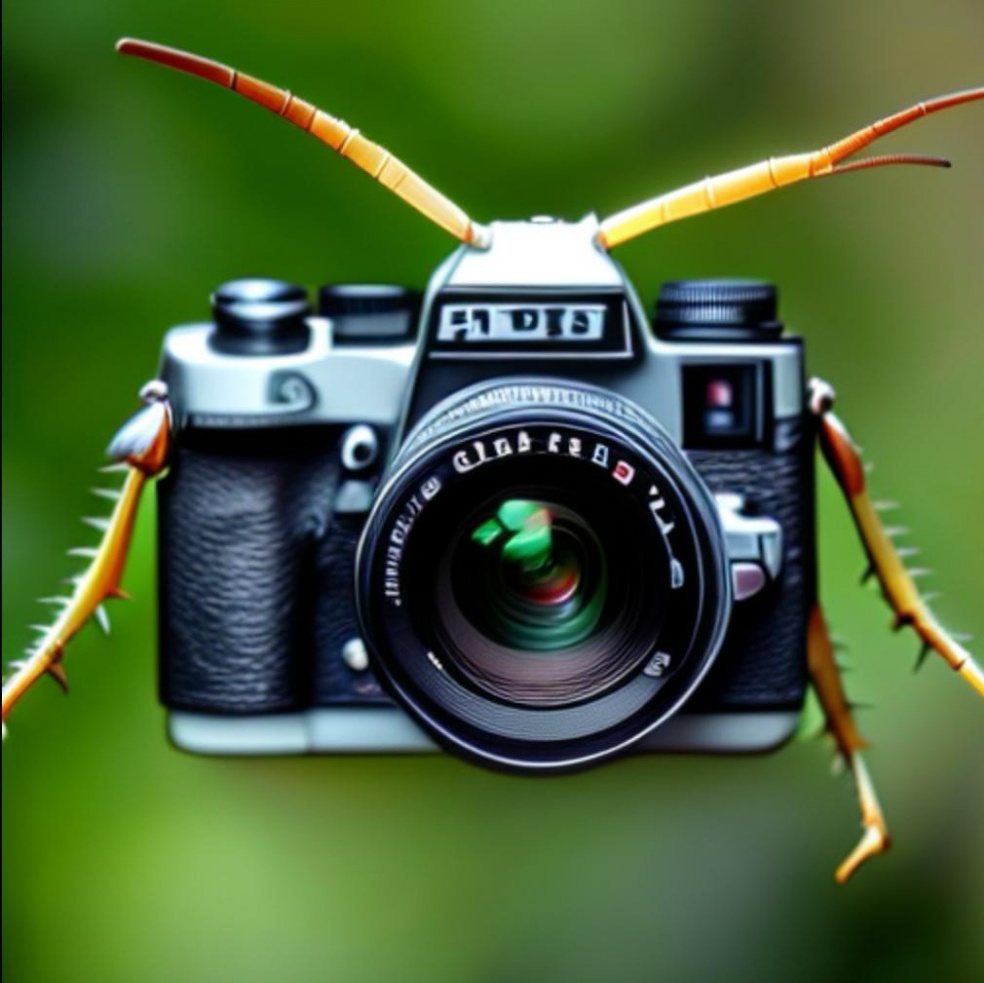} &
\includegraphics[width=\imgw]{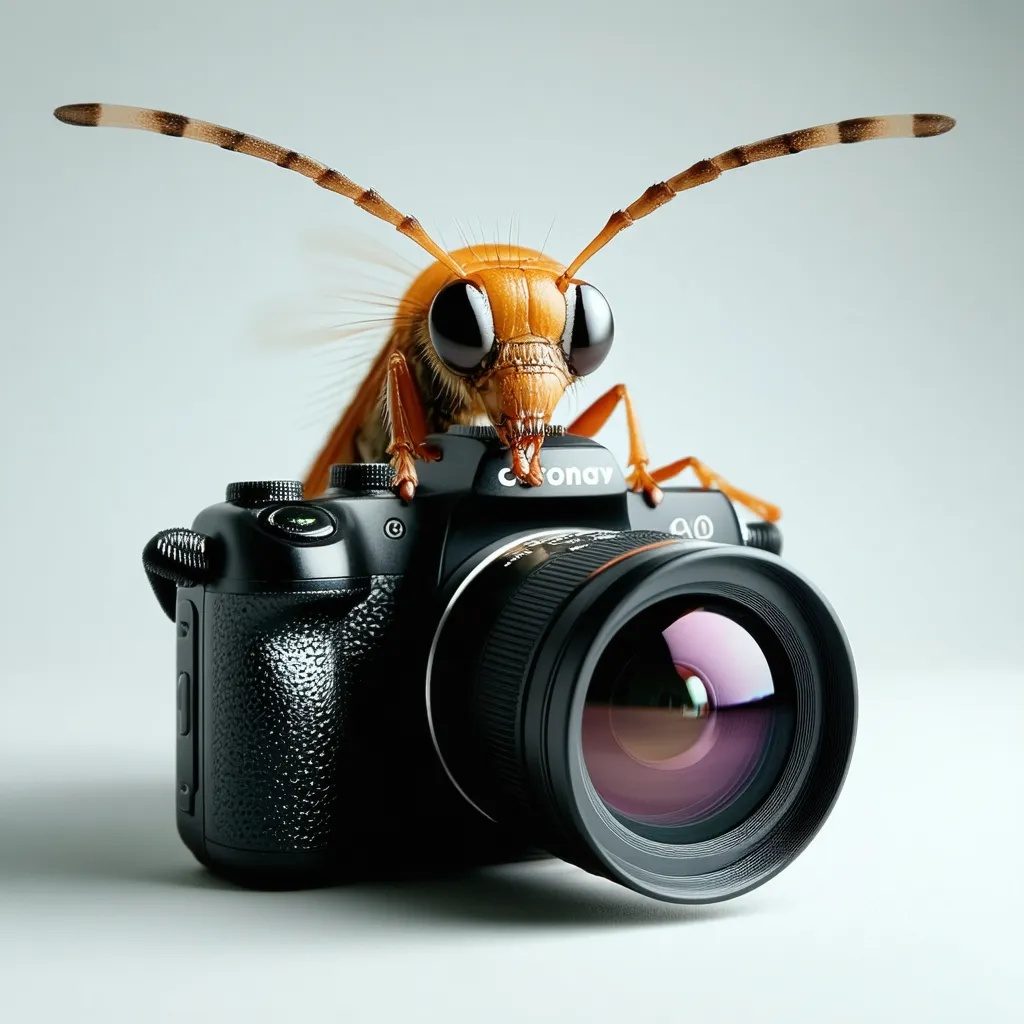} &
\includegraphics[width=\imgw]{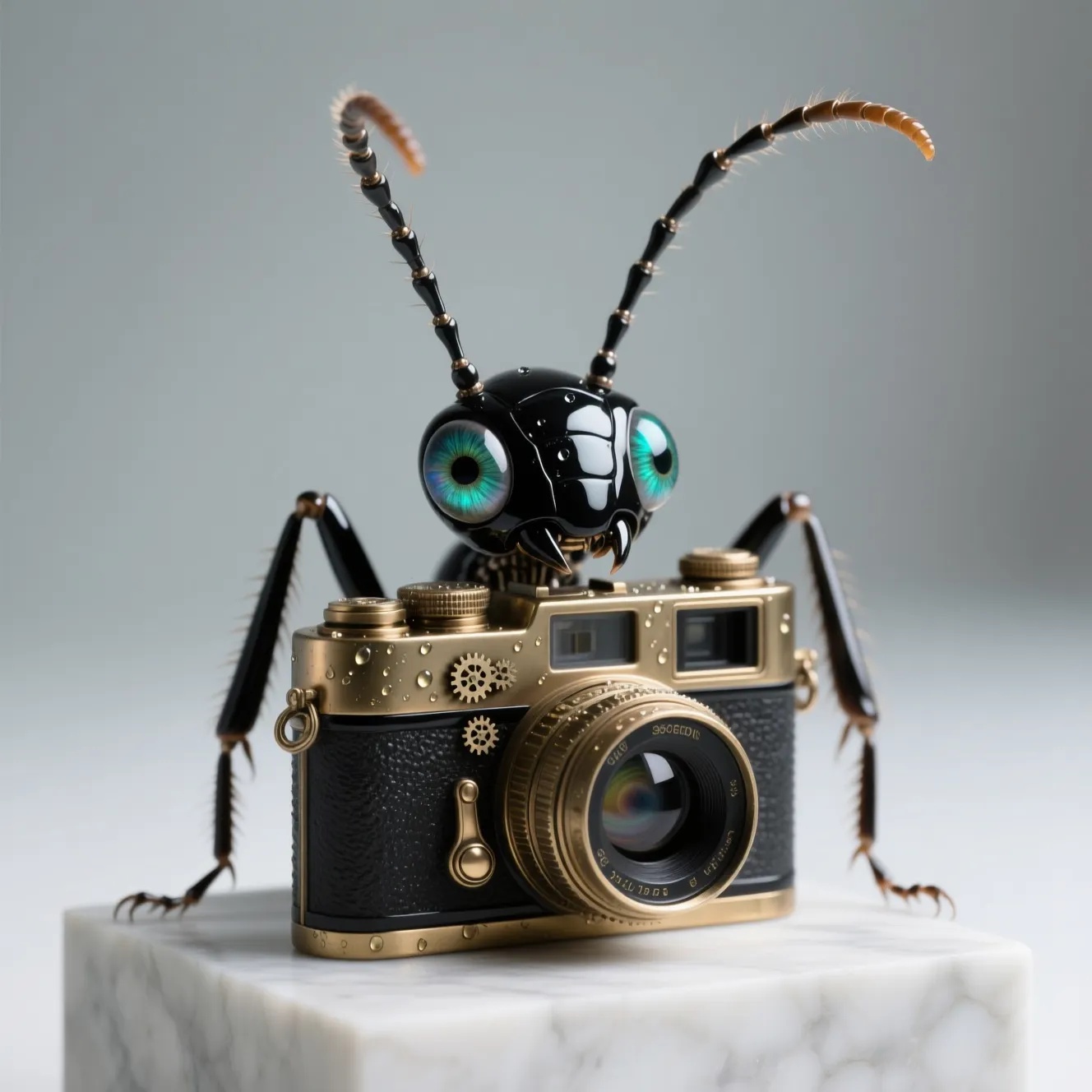} &
\includegraphics[width=\imgw]{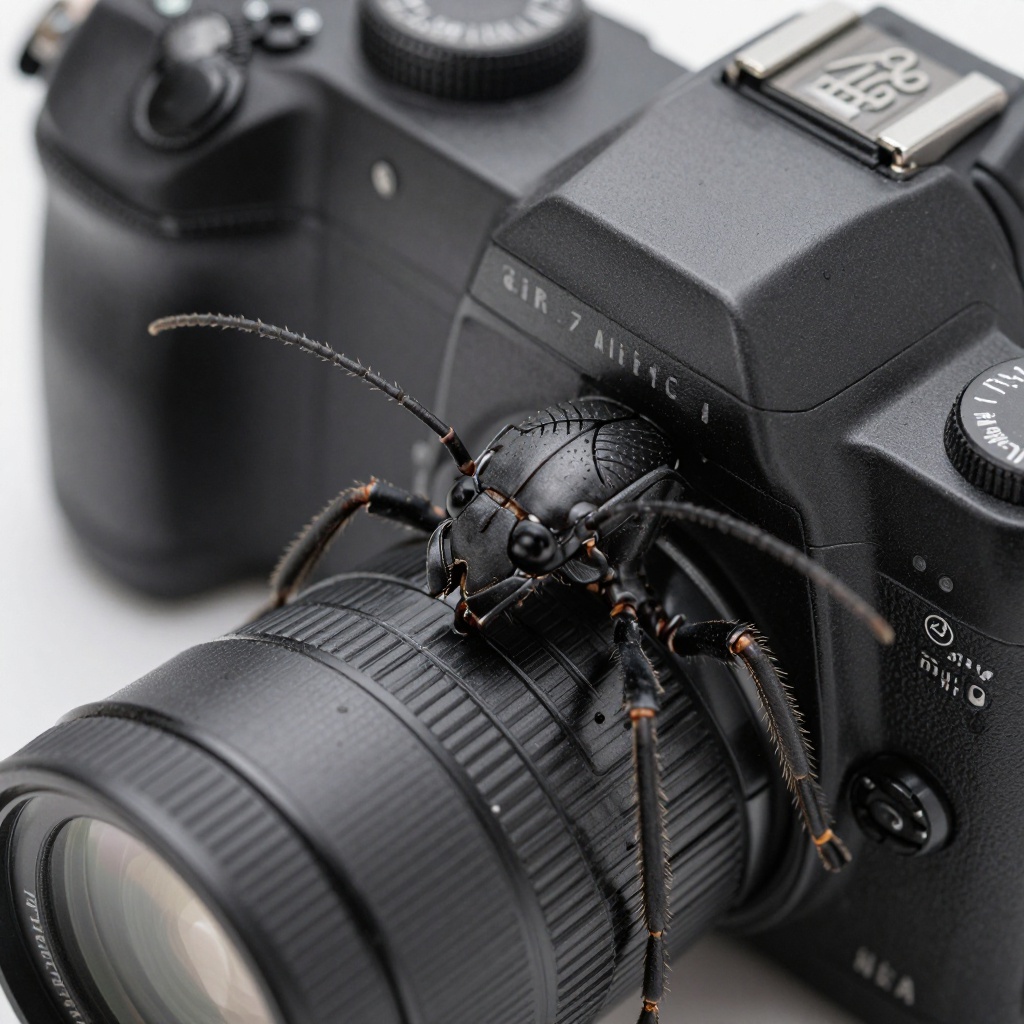} &
\includegraphics[width=\imgw]{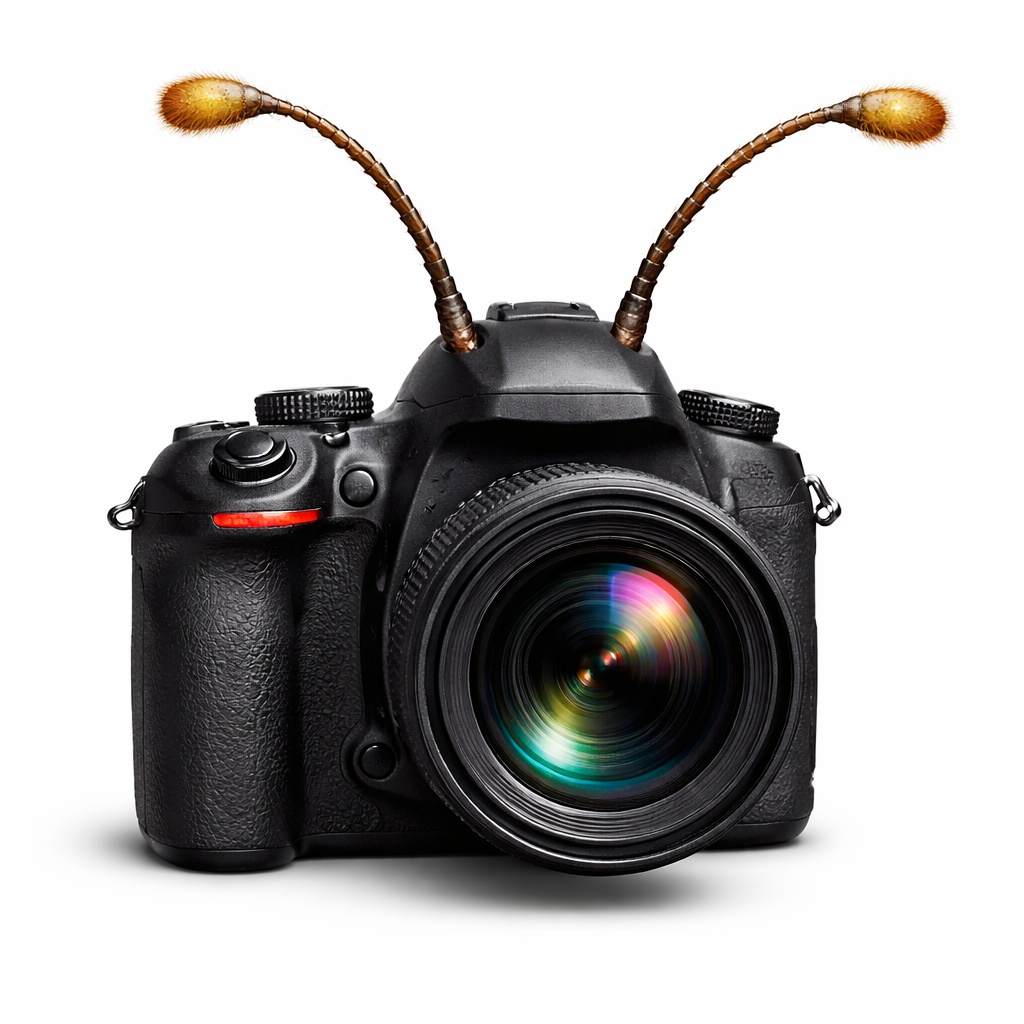} &
\includegraphics[width=\imgw]{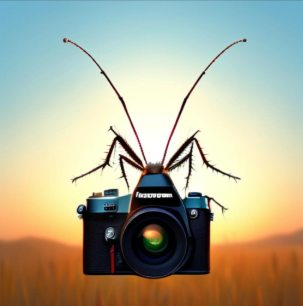} &
\includegraphics[width=\imgw]{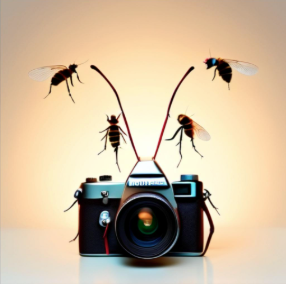} &
\includegraphics[width=\imgw]{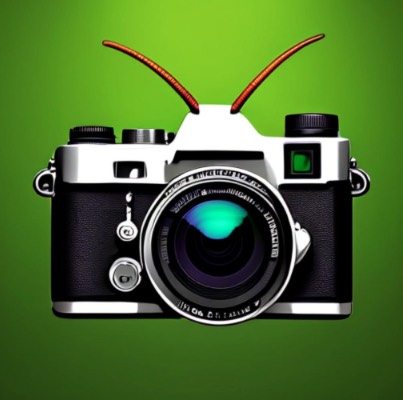} 
\\

\end{tabular}

  \caption{Comparison of image generation results under different compositional prompts. 
  Baseline models include SANA~\citep{xie2024sana}, Aligner~\citep{conceptaligner}, Stable Diffusion 3.5-Large~\citep{esser2024scaling}, QwenImage~\citep{wu2025qwen}, Z-Image~\citep{cai2025z}, and GPT5.2. 
  w/o Spar and w/o CI denote SEDGE without sparsity loss and conditional independence structure, respectively.
  Each row displays generated image for a prompt, and each column presents results from either SEDGE or baseline models. Despite using only 1.6B parameters, SEDGE achieves performance comparable to GPT5.2 and outperforms the 20B QwenImage model, demonstrating its effectiveness for extrapolation. Without the proposed sparsity and conditional independence structure, the model struggles to identify distinct concepts, often producing entangled representations (e.g., a full bird instead of a bird wing) and failing to adhere to the prompt.
  }
  \label{fig:qual-suitcase}
  \vskip -0.1in
\end{figure*}

\subsection{Structure-Informed Extrapolated Data Generation}
\vspace{-1.5mm}
Having estimated  conditional distribution $p(\mathbf{Z}\,|\,\mathbf{X})$ (called likelihood) under different data availability settings, we now describe how to generate samples to satisfy  novel combinations of specifications.
\paragraph{Optimization-based generation.}
When the numbers of $\mathbf{X}$ and $\mathbf{Z}$ are relatively small, we can directly generate extrapolated samples via optimization. Usually in these cases the specifications and features are given. We initialize from the empirical prior, i.e., $ \mathbf{X}_{init} \sim p^{\mathcal{D}}_{\mathbf{X}}$, and perform gradient ascent to maximize the conditional probability $p(\mathbf{Z}\,|\,\mathbf{X})$ to achieve extrapolation on novel specification combinations. The sparsity and conditional independence structure implied by the model significantly stabilizes optimization, as each feature is influenced by only a limited subset of specifications, reducing gradient conflicts. 

\paragraph{Diffusion-based generation.} 
In more general settings with high-dimensional feature and specification spaces, optimization-based generation can be sensitive to initialization, as the region of attraction for high-likelihood solutions may be small. To address this issue, we adopt the diffusion posterior sampling framework of \cite{chung2022dps} to guide the step-wise sampling process to satisfy the novel combination of specification characterized by the {\it structured} $p(\mathbf{Z}\,|\,\mathbf{X})$. To align with the algorithm, we train a diffusion model on $\mathbf{X}$ (if the features are not given, we train the diffusion model on the estimated features $\hat{\mathbf{X}} = f(\mathbf{Y_X})$). We then generate data points from Gaussian noise and gradually apply the constraint of the likelihood in the reversed denoising process, enabling extrapolated data generation under novel specification combinations.  Complete algorithmsare given in Appendix~\ref{app: sedge algorithm}.

\section{Experiments}
 On the synthetic experiments below, we used NVIDIA RTX A6000 GPUs to parallelize independent runs. 
 Each experiment completes within a few minutes on a standard CPU/GPU setup. For the text-to-image experiments, we used AMD Instinct MI210 GPUs for the prior and likelihood models, with 64 GB of memory per GPU. The prior and likelihood model training runs each took on the order of half a day.
\subsection{Synthetic Experiment}
We verify the extrapolation conditions proposed in the previous sections. To make the synthetic setting informative, we design the data generation process and the baselines as follows. 
\textbf{Data:} We follow the structure in Fig.~\ref{fig:two_structures}(b) to sample $\mathbf{X}$ and $\mathbf{Z}$, where $X_1, X_3 \sim \mathcal{U}(0,1)$, $X_2 \sim \mathcal{U}(0.75, 0.8)$, $Z_1 = 0.8 X_1 + 0.6 Z_2 + \epsilon_1$, and $Z_2 = 0.6 X_3 + 0.8 Z_2 + \epsilon_2$, with $\epsilon_1, \epsilon_2 \sim \mathcal{U}(0, 0.2)$. We then split the dataset into training and test sets, corresponding to $S=1$ and $S=0$, respectively. Since the selection variable $S$ depends on $\mathbf{Z}$, the Oracle version of the data to be generated is defined by the condition $(Z_1 < 0.8) \cap (Z_2 > 1.0)$, and the remaining samples constitute Seen X (training set). This construction ensures that the evaluation genuinely tests extrapolation ability.
\textbf{Our model and baselines:} We compare five models, denoted as models A, B, C, D, and E. The first three models adopt the ``features pointing to specifications'' structure shown in Fig. \ref{fig:setting_toy_two_specifications}, while the last two adopt the ``specifications pointing to features'' structure shown in Fig.~\ref{fig:two_structures}(a). Model A (ours) exploits the sparsity of the links and conditional independence among $Z_i$ by estimating the forward likelihood using the exact form of (\ref{eq: loss likelihood}). Model B ignores the sparse structure and models the likelihood by conditioning each $Z_i$ on all of $\mathbf{X}$. Model C further drops the conditional independence assumption and jointly models $p(\mathbf{Z}\mid\mathbf{X})$. Model D reverses the causal direction and learns $p(\mathbf{X}\mid\mathbf{Z})$, as assumed in many conditional generation models. Model E also reverses the direction but assumes each feature $X_i$ is conditionally independent given $\mathbf{Z}$, modeling $p(X_i\mid\mathbf{Z})$, which reflects text-conditioned image generation models \cite{conceptaligner}. For models A, B, and C, we apply both optimization-based inference and diffusion posterior sampling for extrapolated generation. For models C and E, we directly generate samples given novel specifications, as no inverse problem is involved.
\textbf{Metric:} We evaluate the Maximum Mean Discrepancy (MMD)~\cite{NIPS2006_e9fb2eda} with a Gaussian kernel between the generated samples and the test-set features $\mathbf{X}$. The MMD is zero if and only if the two distributions are identical.

Fig.~\ref{fig:synthetic}(a--b) illustrates the data split. 
As seen from the data distributions, the Oracle data (to be generated) are clearly out of support of the Seen data. 
Across all models, our proposed model achieves the smallest MMD, indicating the closest match between the generated data and the oracal data. The improved performance over models B and C verifies that our extrapolation conditions, namely, conditional independence among specifications and sparse structure, are effective. For models D and E, where specifications point to features, extrapolation ability is further limited, as the generated manifolds are more constrained. We refer readers to Appendix~\ref{app: synthetic given 2d} for a more complete two-dimensional visualization of the generation results.


For scenarios in which the specifications and features are not given but have to be learned from the observations, the only change is in their estimation, and the extrapolated data generation process stays the same once the mapping from the observations to the features or the specifications is learned. We refer readers to Appendix~\ref{app: identify feat and specs} for a discussion of the recoverability of $\mathbf{X}$ and $\mathbf{Z}$, even though this is not the main focus of this work. Please refer to Appendix \ref{Sec:F1} for more settings and results on synthetic data, especially on the estimation results of the features and specifications, robustness of the methods to assumption violations, and ablation studies.

\subsection{Extrapolated Text-to-Image Generation}
\vspace{-1.5mm}


To further validate our theory, we conduct extrapolation experiments on text-to-image generation to evaluate the model's ability to synthesize unseen compositional concepts, e.g. ``a peacock eating ice cream." As an initial implementation of SEDGE, we leverage concepts from Aligner~\citep{conceptaligner} to learn a graphical model with links from features (image concepts) to specifications (text concepts) and a likelihood predictor. 
We use diffusion-based generative backbone (here we used SANA~\citep{xie2024sana}) since images are of high dimensions. 
For further implementation details please refer to Appendix~\ref{app:experiment_results} .


We compare our model against several state-of-the-art baselines, including the original SANA,
Aligner, Stable Diffusion 3.5-Large~\citep{esser2024scaling}, 
QwenImage~\citep{wu2025qwen}, Z-Image~\citep{cai2025z}, and GPT5.2 
(Fig.~\ref{fig:qual-suitcase}). 
These baselines often suffer from concept leakage or the generation of spurious components, 
such as a peacock merely touching ice cream instead of eating it, or suitcases accompanied by redundant wings or birds. 
In contrast, SEDGE produces semantically coherent images that better preserve the required compositional attributes. 
We further conduct an ablation study by removing the sparsity loss and conditional independence structure, which results in more entangled compositions (e.g., generating a full bird instead of bird wings) and poorer adherence to the prompt (e.g., the ice cream is missing).
Notably, our 1.6B-parameter model achieves more faithful generation than the 20B QwenImage model and performs comparably to GPT5.2, demonstrating its extrapolation ability and efficiency under limited training data.  For additional results, including quantitative evaluations, please see Appendix~\ref{app:experiment_results}.
\vspace{-0.1cm}

\section{Conclusion and Discussions}
\vspace{-2mm}
For generative AI to reliably produce novel content, the ability to extrapolate beyond the training data is essential. Unlike traditional approaches that often rely on assumptions about the functional forms of the underlying relationships, we investigate this problem from a structural perspective, by employing a graphical model to represent how features and property specifications are related, as well as how the training data are limited. This framework enables us to characterize which components in the decomposition of the joint distribution of the training data can be leveraged to inform extrapolated data generation, and how this can be achieved in practice. 
We further provide conditions under which novel data can be generated reliably, as well as conditions under which the distribution of the novel data can be approximately constructed. As future work, we plan to explore how these ideas can address important problems such as drug discovery and novel protein structure prediction.

{\bf Limitations and Impact.} This paper assumes particular structural properties to make extrapolated data generation feasible and reliable. However, there are situations in which these assumptions do not hold, in which case the proposed framework may fail. In addition, the finite-sample properties of the method are not well understood, although this is a common limitation in deep learning. This work aims to advance the field of machine learning, particularly in achieving reliability and trustworthiness in extrapolated data generation. We hope that this line of research will improve generative AI's ability to support reliable creativity and benefit scientific discovery, drug discovery, and materials design. At the same time, we should be mindful that it may also make it easier to generate fake or misleading novel content.

\bibliography{extrapolation}
\bibliographystyle{abbrv}

\newpage
\appendix
\onecolumn
{\section*{Appendices for \emph{``SEDGE: Structural Extrapolated Data Generation''}}
\addcontentsline{toc}{section}{Appendices}

\startcontents[appendix]

\section*{Appendices Contents}
\printcontents[appendix]{l}{1}{}

\section{Related Work}

\textbf{Extrapolation} 
Existing approaches for extrapolation largely fall into two categories, distinguished by where the assumption are imposed. The first category make assumptions on the data distribution, typically assuming the target distribution shares the same support as the source. It includes most work on the topic of out-of-distribution generalization, which focuses on invariance or reweighting rather than extrapolation on novel specifications \cite{yang2024generalized}. There is also work in regression and statistical inference studying prediction or uncertainty guarantees outside the observed support \cite{dong2022first, shen2025engression} without addressing novel combinations of specifications. 

The second category introduce functional assumptions on the generation process to regularize the behavior ouside training support. Special restrictions are introduced to data generating process including linear, additive, slot-wise, compositional generators, or other specific parametric relations \cite{liu2022compositional_diffusion, lachapelle2023additive, wiedemer2023compositional,  saengkyongam2023linear_intervene}. Other work derives extrapolation guarantees from regularities on distribution shifts, for instance, smoothness \cite{kong2024towards}. These assumptions are hard to be verified from observed data and do not provide a general framework for extrapolated data generation on novel specifications. 
In contrast, we provide a general structural view to the problem that does not assume a particular functional form of the generating mechanism or smoothness of shifts.

\textbf{Text-to-Image Generation}
In recent years, diffusion-based approaches have become the dominant paradigm for text-to-image generation due to their strong performance~\citep{rombach2022high,esser2024scaling,cai2025z,wu2025qwen, flux2024}. One line of research focuses on improving the efficiency of training and sampling~\citep{chen2023pixart,chen2024pixart,xie2024sana,xie2025sana,chen2025sana,chen2025sana-video}. Another line of work aims to enhance controllability by incorporating additional conditioning signals, such as structural or semantic guidance~\citep{zhang2023adding,ye2023ip,zhao2023uni,xie2023omnicontrol}. A further line of research studies disentanglement in text-to-image generation, seeking to separate and manipulate semantic factors within the generative process~\citep{wu2023uncovering,liu2023unsupervised,conceptaligner,dalva2024noiseclr}. In this paper, we consider text-to-image generation as a real problem to address, but in the context of extrapolation.



\section{Main Proofs for Extrapolated Generation}
\subsection{Proof of Proposition \ref{proposition:recover_distribution_no_common_features_toy}}\label{app:proof_proposition_recover_distribution_no_common_features_toy}
\PropositionRecoverDistributionNoCommonFeaturesToy*

\begin{proof}
By Bayes Rule, we have
\begin{flalign}
    &p(\mathbf{X} \,|\, Z_1=1, Z_2=1) \nonumber \\
    &= p(X_1,X_2\,|\, Z_1=1, Z_2=1) \nonumber \\
    &\qquad \cdot p(X_3,X_4 \,|\,  X_1,X_2,Z_1=1, Z_2=1) \nonumber \\ 
    &= p(X_1,X_2\,|\, Z_1=1) p(X_3,X_4 \,|\,  Z_2=1) \label{eq:proof1_c1} \\
    &= p^\mathcal{D}(X_1,X_2\,|\, Z_1=1) p^\mathcal{D}(X_3,X_4, \,|\,  Z_2=1),\label{eq:proof1_c2}
\end{flalign}
where (\ref{eq:proof1_c1}) follows from $\{X_1,X_2\}\independent Z_2 \,|\, Z_1$ and $\{X_3,X_4\}\independent \{X_1,X_2,Z_1\} \,|\, Z_2$. Eq. (\ref{eq:proof1_c2}) follows from $\{X_1,X_2\}\independent S \,|\, Z_1$, $\{X_3,X_4\}\independent S \,|\, Z_2$, and Assumption \ref{eq:assumption_given_specifications_toy_1}.
\end{proof}

\subsection{Proof of Proposition \ref{proposition:non_identifiability_distribution_toy}}\label{app:proof_proposition_non_identifiability_distribution_toy}
\PropositionNonIdentifiabilityDistributionToy*
\begin{proof}
We begin with deriving the data distribution following the novel specification, $p(\mathbf{X}\,|\,Z_1=1, Z_2=1)$.

\begin{align}
    & p(\mathbf{X}\,|\,Z_1 = 1, Z_2 = 1) \\ 
    = &  p(X_1, X_2\,|\,Z_1 = 1, Z_2 = 1)p(X_3\,|\, X_1, X_2 ,Z_1 = 1, Z_2 = 1) \\ 
    = & p(X_2\,|\,Z_1 = 1, Z_2 = 1)p(X_1\,|\,X_2, Z_1 = 1, Z_2 = 1)p(X_3\,|\, X_1, X_2 ,Z_1 = 1, Z_2 = 1) \\ \label{Eq:ind}
    = & p(X_2\,|\,Z_1 = 1, Z_2 = 1)p(X_1\,|\,X_2, Z_1 = 1, )p(X_3\,|\, X_2,  Z_2 = 1) \\ 
    = & p(X_2\,|\,Z_1 = 1, Z_2 = 1) p(X_1\,|\,X_2, Z_1 = 1, S=1)p(X_3\,|\, X_2,  Z_2 = 1, S= 1) \\ \label{Eq:ind2}
    = & p(X_2\,|\,Z_1 = 1, Z_2 = 1) p^\mathcal{D}(X_1\,|\,X_2, Z_1 = 1)p^\mathcal{D}(X_3\,|\, X_2,  Z_2 = 1),
\end{align}
Eq. (\ref{Eq:ind}) holds because $X_1 \independent Z_2 \,|\, \{X_2, Z_1\}$ and $X_3 \independent \{Z_1, X_1\}  \,|\, \{X_2, Z_2\}$. Eq. (\ref{Eq:ind2}) is the case because $X_1 \independent S \,|\, \{X_2, Z_1\}$ and $X_3 \independent S \,|\, \{X_2, Z_2\}$. Conditioning on $S=1$ means that they are distributions implies by the given data.  Note that although $X_2 \independent S \,|\, \{Z_1, Z_2\}$ implies
$$p(X_2\,|\,Z_1 = 1, Z_2 = 1) = p(X_2\,|\,Z_1 = 1, Z_2 = 1, S=1),  $$
the quantity on the right-hand side is not defined since for the conditioning variables, $p(Z_1 = 1, Z_2 = 1, S=1) = p(Z_1 = 1, Z_2 = 1 \,|\, S=1) p(S = 1) = 0$, as the given data do not satisfy $Z_1 = 1, Z_2 = 1$.

Deriving $p(\mathbf{X} \,|\, Z_1 = 1, Z_2 =1)$ then reduces to the estimation of $p(X_2 \,|\, Z_1 = 1, Z_2 = 1)$.   
\begin{align}
    & p(X_2\,|\,Z_1 = 1, Z_2 = 1) \\ 
    = & p(Z_1 = 1, Z_2 = 1 \,|\, X_2) \cdot \frac{p(X_2)}{P(Z_1 = 1, Z_2 = 1)} \\ 
    \propto & p(Z_1 = 1, Z_2 = 1 \,|\, X_2) p(X_2) \\ \label{Eq4}
    = & p(Z_1 = 1\,|\, X_2) p(Z_2 = 1\,|\, X_2)p(X_2) \\
    \propto & p(X_2\,|\,Z_1 = 1)p(X_2\,|\,Z_2 = 1)\cdot\frac{1}{p(X_2)}.
    \label{eq:derived_X2_given_Z1_Z2}
\end{align}
The problem is that none of the three involved quantities, $p(X_2\,|\,Z_1 = 1)$, $p(X_2\,|\,Z_1 = 1)$, and $p(X_2)$ is known since we only have access to the distribution corresponding to the given data, for which $S = 1$. In fact,
\begin{equation} \label{Eq3}
p(X_2\,|\,Z_1 = 1) = p(X_2\,|\,Z_1 = 1, S=1)P(S=1\,|\,Z_1=1) + p(X_2\,|\,Z_1 = 1, S=0)P(S=0\,|\,Z_1=1).
\end{equation}
We do not have information of the data with $S=0$, which are not in the given dataset. As a consequence, neither of $p(X_2\,|\,Z_1 = 1, S=0)$ and $P(S=0\,|\,Z_1=1)$ is known.  Hence, $p(X_2\,|\,Z_1 = 1)$ is 
not identifiable from the given data, and it also happens to $p(X_2\,|\,Z_1 = 1)$ and ${p(X_2)}$. Therefore, $p(X_2\,|\,Z_1 = 1,  Z_2 = 1)$ is {\it not identifiable}, and hence according to (\ref{Eq:ind2}), $p(\mathbf{X}\,|\,Z_1 = 1, Z_2 = 1)$ is not identifiable given the distribution of the observed data.

\end{proof}

\subsection{Proof of Proposition \ref{proposition:extrapolation_sample_toy}}\label{app:proof_proposition_extrapolation_sample_toy}
\PropositionExtrapolationSampleToy*
\begin{proof}
First, note that $p(S=1 \,|\,Z_1 =1) > 0$ as for the given data, $S=1$, and bear in mind that
\[
p(X_2 = \tilde{x}_2\,|\, Z_1 =1, S=1) > 0,
\]
according to Assumption \ref{assumption:exist_sample_toy}. Then according to (\ref{Eq3}), we have
\[
p(X_2 =\tilde{x}_2 \,|\, Z_1 = 1)>0.
\]
Similarly, we have
\[
p(X_2 = \tilde{x}_2 \,|\, Z_2 = 1) > 0.
\]
Then according to (\ref{Eq4}), $p(X_2 = \tilde{x}_2 \,|\, Z_1=1, Z_2=1) > 0$. We further know that there exist $\tilde{x}_1$, as values of $X_1$, such  that
\[
p^\mathcal{D}(X_1 = \tilde{x}_1\,|\,X_2 = \tilde{x}_2, Z_1 = 1) > 0;
\]
this is because otherwise, $p^\mathcal{D}(X_1 = \tilde{x}_1\,|\,X_2 = \tilde{x}_2, Z_1=1) = 0$ implies $p^\mathcal{D}(X_1, X_2 = \tilde{x}_2 \,|\, Z_1=1) = 0$ and hence $p^\mathcal{D}(X_2 = \tilde{x}_2 \,|\, Z_1=1) = 0$, contradicting Assumption \ref{assumption:exist_sample_toy}. Similarly, there exist  values of $X_3$, denoted by $\tilde{x}_3$, such that
\[
p^\mathcal{D}(X_3 = \tilde{x}_3\,|\,X_2 = \tilde{x}_2, Z_2 = 1) > 0.
\]
Finally, according to  (\ref{Eq:ind2}), we obtain
\[
p(X_2 = \tilde{x}_2, X_1 = \tilde{x}_1, X_3 = \tilde{x}_3 \,|\,Z_1 = 1, Z_2=1 ) > 0.
\]
\end{proof}

\subsection{Proof of Proposition \ref{proposition:recover_conservative_distribution_toy}}\label{app:proof_proposition_recover_conservative_distribution_toy}
\PropositionRecoverConservativeDistributionToy*
\begin{proof}
By Assumption \ref{assumption:conservative_distribution_toy}, one can show
\begin{align*}
    p(X_2 \,|\, Z_i = 1) & \approx p^\mathcal{D}(X_2 \,|\, Z_i = 1), \quad i=1,2,\\
    p(X_2) & \approx p^\mathcal{D}(X_2).
\end{align*}
Plugging these into (\ref{Eq:ind2}) and (\ref{eq:derived_X2_given_Z1_Z2}) yields
\begin{align} \nonumber
&p(\mathbf{X}\,|\, Z_1 = 1, Z_2 = 1) \\
\approx &\frac{p^\mathcal{D}(X_2\,|\, Z_1 = 1)p^\mathcal{D}(X_2\,|\, Z_2 = 1)}{p^\mathcal{D}(X_2)}\frac{P(Z_1 = 1)P(Z_2 = 1)}{P(Z_1 = 1, Z_2 = 1)}\nonumber\\ 
  & \quad \cdot p^\mathcal{D}(Z_2=1 \,|\, X_2) p^\mathcal{D}(X_3 \,|\, X_2,Z_2 = 1)\\
  \propto & p^\mathcal{D}(X_1, X_2 \,|\, Z_1 = 1) p^\mathcal{D}(Z_2=1 \,|\, X_2) \cdot p^\mathcal{D}(X_3 \,|\, X_2,Z_2 = 1).\label{eq:sketch_proof4_c1_app}
\end{align}
All the three terms are implied by the given data.
\end{proof}

\subsection{SEDGE with Multiple Specifications} \label{Sec:multiple}
We now extend the results from the previous subsection to show how structural properties can be exploited to enable extrapolated data generation in the general setting with multiple specifications. Analogous to the two-specification setting in Assumption \ref{assumption:given_specifications_toy}, we aim to extrapolate to the novel combination of specifications $\mathbf{Z}=1$ that is unseen in the given data due to selection, i.e., $p(\mathbf{Z}=1 \,|\, S=1)= 0$. We introduce the following basic conditions.
\begin{assumption}[Basic conditions for specifications and selection]\label{assumption:basic_conditions}
We have
\begin{flalign}
P(S=1) &> 0,\label{eq:assumption_basic_conditions_1}\\
P(\mathbf{Z}=\mathbf{1})&>0,\label{eq:assumption_basic_conditions_2}\\
P(Z_i=1 \,|\, S=1)&> 0,\quad i\in [k].\label{eq:assumption_basic_conditions_3}
\end{flalign}
\end{assumption}
Specifically, (\ref{eq:assumption_basic_conditions_1}) ensures that the given dataset is nonempty. Eq. (\ref{eq:assumption_basic_conditions_2}) guarantees that the target specification combination is feasible in the underlying population, while (\ref{eq:assumption_basic_conditions_3}) ensures that each individual specification value appears in the given/training data in some context, providing the necessary marginal coverage to potentially extrapolate to the novel combinations.

\subsubsection{Identifiable Case: Specifications Do Not Share Common Features} 
Recall that Fig. \ref{fig:setting_toy_two_specifications}(a) and Proposition \ref{proposition:recover_conservative_distribution_toy} consider the setting in which specifications do not share common features, by making use of the independence between $Z_1$ and $Z_2$, i.e., $Z_1 \independent Z_2$. Building on this intuition, we now extend it to multiple specifications. We present the following result, which shows that when specifications do not share common features, the distribution corresponding to novel specifications, $p(\mathbf{X}\,|\, \mathbf{Z}=\mathbf{1})$, is identifiable from the observed data.
\begin{restatable}{theorem}{TheoremRecoverDistributionNoCommonFeatures}\label{theorem:recover_distribution_no_common_features}
Suppose that there exists a partition of the specifications $\mathbf{Z}$ into disjoint specification subsets $\mathbf{Z}_1,\dots,\mathbf{Z}_k$ such that no feature is shared between any two different subsets and that features across different subsets are not adjacent. Suppose further that Assumption \ref{assumption:basic_conditions} holds and
\[
P(\mathbf{Z}_i=1 \,|\, S=1)> 0, \quad i\in [k].
\]
Then, the novel data distribution $p(\mathbf{X}\,|\, \mathbf{Z}=\mathbf{1})$ is identifiable from the given data.
\end{restatable}
The proof is given in \ref{app:proof_theorem_recover_distribution_no_common_features}. It is worth noting that, unlike the setting in Fig. \ref{fig:setting_toy_two_specifications}(a) that requires all specifications not to share any feature, the result above further allows the specifications to be partitioned into subset of specifications. That is, while no features are shared across different subsets, arbitrary feature sharing and dependence relations are permitted within each subset.

\subsubsection{Non-Identifiable Case: Specifications Share Common Features}
Previously, Fig. \ref{fig:setting_toy_two_specifications}(b) and Fig. \ref{proposition:extrapolation_sample_toy} study the setting in which two specifications share common features, by leveraging the conditional independence between $Z_1$ and $Z_2$ given their common child $X_2$, i.e., $Z_1 \independent Z_2 \,|\, X_2$. Analogous to this case, we now consider a more general setting with multiple specifications that share common features. As shown in Proposition \ref{proposition:non_identifiability_distribution_toy}, the distribution satisfying novel specifications is generally not identifiable. In the following, we extend Proposition \ref{proposition:extrapolation_sample_toy} from the two-specification setting to show that generating novel data samples remains possible with multiple specifications, and generalize Proposition \ref{proposition:recover_conservative_distribution_toy} to show how the non-identifiability issue with multiple specifications can be addressed by learning conservative solutions.

\paragraph{Existence of novel data sample.}
The following result shows that, under a generalization of Assumption \ref{assumption:exist_sample_toy} which requires that the common features have a positive density under each set of specifications individually in the given data, one can construct data points satisfying the novel specifications $\mathbf{Z}=\mathbf{1}$ from the given data.
\begin{restatable}{theorem}{TheoremExtrapolationSample}\label{theorem:extrapolation_sample}
Suppose that there exists a partition of the specifications $\mathbf{Z}$ into disjoint specification subsets $\mathbf{Z}_1,\dots,\mathbf{Z}_k$ such that they are conditionally independent given a subset of features $\mathbf{X}_{c}\subset \mathbf{X}$. Moreover, suppose that, given $\mathbf{X}_c$, each feature $X_j \notin \mathbf{X}_c$ is conditionally independent of all other features that are not associated with the same specification subset as $X_j$.\footnote{Here, a feature is said to be \emph{associated} with a specification set if it is a parent of at least one specification variable in that set.} Suppose further that Assumption \ref{assumption:basic_conditions} holds, 
\[
P(\mathbf{Z}_i=1 \,|\, S=1)> 0, \quad i\in [k],
\]
and there exists value $\tilde{\mathbf{x}}_c$ of $\mathbf{X}_c$ such that
\[
p^\mathcal{D}(\mathbf{X}_c = \tilde{\mathbf{x}}_c \,|\, \mathbf{Z}_i = 1) > 0, \quad i \in [k].
\]
 Then, there exists a value $\tilde{\mathbf{X} = \mathbf{x}}$ of $\mathbf{X}$ that can be constructed from the given data such that  
\[
p(\mathbf{X}=\tilde{\mathbf{x}} \,|\, \mathbf{Z}=\mathbf{1}) > 0.
\]
\end{restatable}
The proof is provided in \ref{app:proof_theorem_extrapolation_sample}. The theorem is rather general and accommodates complex structures, allowing one to exploit conditional independencies implied by the structural properties to generate data points satisfying novel specifications in a broad range of settings.

\paragraph{Conservative solutions for non-identifiable distributions.}
Similar to Proposition \ref{proposition:recover_conservative_distribution_toy} with two specifications, the non-identifiability issue of $p(\mathbf{X}\,|\, \mathbf{Z}=\mathbf{1})$ with multiple specifications can be resolved by assuming that the chance of having novel scenarios in the data generating process, which are not contained in the given data, is low. This is formally stated in the following result, with a proof given in \ref{app:proof_theorem_recover_conservative_distribution}.
\begin{restatable}{theorem}{TheoremRecoverConservativeDistribution}\label{theorem:recover_conservative_distribution}
Suppose that there exists a partition of the specifications $\mathbf{Z}$ into disjoint specification subsets $\mathbf{Z}_1,\dots,\mathbf{Z}_k$ such that they are conditionally independent given a subset of features $\mathbf{X}_{c}\subset \mathbf{X}$. Moreover, suppose that, given $\mathbf{X}_c$, each feature $X_j \notin \mathbf{X}_c$ is conditionally independent of all other features that are not associated with the same specification subset as $X_j$. Suppose further that Assumption \ref{assumption:basic_conditions} holds and
\begin{flalign*}
P(\mathbf{Z}_i=1 \,|\, S=1)&> 0, \quad i\in [k], \\
P(S=0 \,|\, \mathbf{Z}_i=1)&\approx 0, \quad i\in[k], \\
P(S=0)&\approx 0.
\end{flalign*}
Then, the novel data distribution $p(\mathbf{X} \,|\, \mathbf{Z}=\mathbf{1})$ is approximately identifiable from the given data.
\end{restatable}

\subsection{Proof of Theorem \ref{theorem:recover_distribution_no_common_features}}\label{app:proof_theorem_recover_distribution_no_common_features}
\TheoremRecoverDistributionNoCommonFeatures*
\begin{proof}
With a slight abuse of notation, let $\mathbf{V}_{\mathbf{Z}_i} \subseteq \mathbf{X}$ be the parent set of specification subset $\mathbf{Z}_i$. By assumption, no feature is shared between the specification subsets $\mathbf{Z}_1,\dots,\mathbf{Z}_k$, and features across different subsets are not
adjacent. This implies that (1) $\mathbf{V}_{\mathbf{Z}_1}, \dots, \mathbf{V}_{\mathbf{Z}_k}$ are conditionally independent of each other given $\mathbf{Z}_1,\dots,\mathbf{Z}_k$ and (2) $\mathbf{V}_{\mathbf{Z}_i}\independent \mathbf{Z}_j \,|\, \mathbf{Z}_i$ for $i\neq j$.

By Bayes Rule, we have
\begin{flalign*}
    &p(\mathbf{X}\,|\, \mathbf{Z}=\mathbf{1}) \nonumber \\
    &= \prod_{i=1}^k p(\mathbf{V}_{\mathbf{Z}_i}\,|\, \mathbf{Z}=\mathbf{1}) \nonumber \\ 
    &= \prod_{i=1}^k p(\mathbf{V}_{\mathbf{Z}_i}\,|\, \mathbf{Z}_i=\mathbf{1}) \nonumber \\ 
    &= \prod_{i=1}^k p(\mathbf{V}_{\mathbf{Z}_i}\,|\, \mathbf{Z}_i=\mathbf{1},S=1) \nonumber \\ 
    &= \prod_{i=1}^k p^\mathcal{D}(\mathbf{V}_{\mathbf{Z}_i}\,|\, \mathbf{Z}_i=\mathbf{1}). \nonumber 
\end{flalign*}
Here, the second line follows from the fact that $\mathbf{V}_{\mathbf{Z}_1}, \dots, \mathbf{V}_{\mathbf{Z}_k}$ are conditionally independent of each other given $\mathbf{Z}_1,\dots,\mathbf{Z}_k$. The third line follows from  $\mathbf{V}_{\mathbf{Z}_i}\independent \mathbf{Z}_j \,|\, \mathbf{Z}_i$ for $i\neq j$ and the assumption of $P(\mathbf{Z}_i=1 \,|\, S=1)> 0$. The fourth line follows from $\mathbf{V}_{\mathbf{Z}_i}\independent S \,|\, \mathbf{Z}_i$.
\end{proof}

\subsection{Proof of Theorem \ref{theorem:extrapolation_sample}}\label{app:proof_theorem_extrapolation_sample}
\TheoremExtrapolationSample*
\begin{proof}
With a slight abuse of notation, let $\mathbf{V}_{\mathbf{Z}_i} \subseteq \mathbf{X}$ be the parent set of specification subset $\mathbf{Z}_i$. By assumption, the disjoint specification subsets $\mathbf{Z}_1,\dots,\mathbf{Z}_k$ are conditionally independent given $\mathbf{X}_c$. Also, each feature $X_j \notin \mathbf{X}_c$ is conditionally independent of all other features that are not associated with the same specification subset as $X_j$.

By Bayes rule, We have
\begin{align}
    & p(\mathbf{X}\,|\,\mathbf{Z}=\mathbf{1}) \nonumber\\ 
    =& p(\mathbf{X}_c\,|\,\mathbf{Z}=\mathbf{1})p(\mathbf{X}\setminus\mathbf{X}_c\,|\,\mathbf{X}_c,\mathbf{Z}=\mathbf{1}) \nonumber\\ 
    =& p(\mathbf{X}_c\,|\,\mathbf{Z}=\mathbf{1})\prod_{i=1}^k p(\mathbf{V}_{\mathbf{Z}_i}\setminus\mathbf{X}_c\,|\,\mathbf{X}_c,\mathbf{Z}=\mathbf{1}) \nonumber \\ 
    =& p(\mathbf{X}_c\,|\,\mathbf{Z}=\mathbf{1})\prod_{i=1}^k p(\mathbf{V}_{\mathbf{Z}_i}\setminus\mathbf{X}_c\,|\,\mathbf{X}_c,\mathbf{Z}_i=\mathbf{1}) \nonumber\\ 
    =& p(\mathbf{X}_c\,|\,\mathbf{Z}=\mathbf{1})\prod_{i=1}^k p(\mathbf{V}_{\mathbf{Z}_i}\setminus\mathbf{X}_c\,|\,\mathbf{X}_c,\mathbf{Z}_i=\mathbf{1},S=1) \nonumber\\ 
    =& p(\mathbf{X}_c\,|\,\mathbf{Z}=\mathbf{1})\prod_{i=1}^k p^\mathcal{D}(\mathbf{V}_{\mathbf{Z}_i}\setminus\mathbf{X}_c\,|\,\mathbf{X}_c,\mathbf{Z}_i=\mathbf{1})\label{eq:proof_temp_2}
\end{align}
Here, the third line follows from the fact that, given $\mathbf{X}_c$, each feature $X_j \notin \mathbf{X}_c$ is conditionally independent of all other features that are not associated with the same specification subset as $X_j$. The fourth line follows from the fact that each $X_j\in\mathbf{V}_{\mathbf{Z}_i}\setminus\mathbf{X}_c$ is conditionally independent with $\mathbf{Z}_l,l\neq i$ given $\mathbf{X}_c$. The fifth line follows from the assumption of $P(\mathbf{Z}_i=1 \,|\, S=1)> 0$ and the fact that each $X_j\in\mathbf{V}_{\mathbf{Z}_i}\setminus\mathbf{X}_c$ is conditionally independent with $S$ given $\mathbf{X}_c$.


Deriving $p(\mathbf{X} \,|\, \mathbf{Z}=\mathbf{1})$ then reduces to the estimation of $p(\mathbf{X}_c \,|\, \mathbf{Z}=\mathbf{1})$. We have
\begin{align}
    & p(\mathbf{X}_c \,|\, \mathbf{Z}=\mathbf{1})\nonumber \\ 
    = & p(\mathbf{Z}=\mathbf{1} \,|\, \mathbf{X}_c) \cdot \frac{p(\mathbf{X}_c)}{P(\mathbf{Z}=\mathbf{1})} \nonumber\\ 
    \propto & p(\mathbf{Z}=\mathbf{1} \,|\, \mathbf{X}_c) p(\mathbf{X}_c)\nonumber \\
    = & \prod_{i=1}^k p(\mathbf{Z}_i = \mathbf{1}\,|\, \mathbf{X}_c) p(\mathbf{X}_c) \\
    \propto & \prod_{i=1}^k p(\mathbf{X}_c\,|\,\mathbf{Z}_i = \mathbf{1})\cdot\frac{1}{p(\mathbf{X}_c)}.\label{eq:proof_general_density}
\end{align}
By assumption, there exists value $\tilde{\mathbf{x}}_c$ of $\mathbf{X}_c$ such that
\[
p^\mathcal{D}(\mathbf{X}_c = \tilde{\mathbf{x}}_c \,|\, \mathbf{Z}_i = 1) > 0, \quad i \in [k],
\]
This implies
\[
p(\mathbf{X}_c = \tilde{\mathbf{x}}_c \,|\, \mathbf{Z}_i = 1)>0.
\]
Then according to (\ref{eq:proof_general_density}), $p(\mathbf{X}_c = \tilde{\mathbf{x}}_c \,|\, \mathbf{Z}=\mathbf{1}) > 0$. We further know that there exist $\tilde{\mathbf{v}}_{\mathbf{Z}_i}\setminus\tilde{\mathbf{x}}_c$, as values of $\mathbf{V}_{\mathbf{Z}_i}\setminus\mathbf{X}_c$, such  that
\[
p^\mathcal{D}(\mathbf{V}_{\mathbf{Z}_i}\setminus\mathbf{X}_c=\tilde{\mathbf{v}}_{\mathbf{Z}_i}\setminus\tilde{\mathbf{x}}_c\,|\,\mathbf{X}_c=\mathbf{x}_c,\mathbf{Z}_i=\mathbf{1}) > 0;
\]
Finally, according to  (\ref{eq:proof_temp_2}), we can construct the point such that
\[
p(\mathbf{X} = \tilde{x} \,|\,\mathbf{Z}=\mathbf{1}) > 0.
\]
\end{proof}

\subsection{Proof of Theorem \ref{theorem:recover_conservative_distribution}}\label{app:proof_theorem_recover_conservative_distribution}
\TheoremRecoverConservativeDistribution*
\begin{proof}
By assumption, one can show
\begin{align*}
    p(\mathbf{X}_c \,|\, \mathbf{Z}_i = \mathbf{1}) & \approx p^\mathcal{D}(\mathbf{X}_c \,|\, \mathbf{Z}_i = \mathbf{1}), \quad i\in[k],\\
    p(\mathbf{X}_c) & \approx p^\mathcal{D}(\mathbf{X}_c).
\end{align*}
Plugging these into (\ref{eq:proof_general_density}) and (\ref{eq:proof_temp_2}) yields
\[
p(\mathbf{X}\,|\, \mathbf{Z}=\mathbf{1}) \propto \frac{\prod_{i=1}^k p^\mathcal{D}(\mathbf{X}_c\,|\, \mathbf{Z}_i = 1)}{p^\mathcal{D}(\mathbf{X}_c)}\prod_{i=1}^k p^\mathcal{D}(\mathbf{V}_{\mathbf{Z}_i}\setminus\mathbf{X}_c\,|\,\mathbf{X}_c,\mathbf{Z}_i=\mathbf{1}).
\]
\end{proof}




\section{Discussion of Conditions and Assumptions}
\label{app:discussion}

We clarify the role and intention of the structural assumptions underlying \textsc{SEDGE}.
A point that applies uniformly to all assumptions below is that \textbf{none of them is
strictly required for the validity of our theory and method}.
They are introduced to keep the theoretical analysis clean and tractable, and to clarify
the boundaries of what can be reliably inferred from available data.
Concretely, the assumptions serve two purposes: (1)~they enable clean identifiability
guarantees, and (2)~they act as inductive biases that improve the ease of reliable
extrapolation.
Neither purpose implies that the method fails when an assumption is violated; see
Appendix~\ref{app:robustness} for empirical confirmation that performance degrades
gracefully rather than catastrophically under moderate violations.

\paragraph{No edges from $Z$ to $X$ (direction of causality).}
The direction $X \to Z$ (features generate specifications) is not a restriction of the
framework, but rather the meaningful setting for nontrivial extrapolation.
If $Z \to X$ holds instead, reliable extrapolation for novel specifications is generally
not possible without additional assumptions when specifications share features; if no
shared features exist, $p(X \mid Z)$ factorises and generation is straightforward in any
case.
The $X \to Z$ direction is also more plausible for the data-generating process in most
practical settings: features such as visual attributes are primitive properties of the
world that give rise to observable concept combinations, rather than the reverse.
Empirically, reversing the direction to $Z \to X$ with shared features leads to noticeable
but not catastrophic degradation, while $Z \to X$ without shared features performs
comparably to the original setting (see Appendix~\ref{app:robust_direction}).

\paragraph{No edges among specification variables $Z$ (conditional independence).}
The assumption that $Z_i$ are mutually conditionally independent is best understood as a
\emph{canonical form} rather than a claim about the true data-generating process.
When $Z$ variables are not conditionally independent --- for example, when $Z_1 \to Z_2$
holds --- the model can always be reduced to this form: we rewrite
$Z_2 = f(Z_1, \varepsilon_2)$ for independent noise $\varepsilon_2$, yielding equivalent
independent representations $Z_1$ and $\varepsilon_2$ to which the same results apply.
Regarding latent confounding: if a hidden variable is specified by observable features $X$,
it is not truly hidden in our setting; if it is not, then the relevant specifications
cannot be fully determined from the image, which is outside our controllable generation
framework.
Introducing direct $Z_1 \to Z_2$ dependencies does degrade performance, and conditional
independence is the more critical of our two inductive biases, but neither its removal nor
that of sparsity causes failure (see Appendix~\ref{app:robust_condindep}).

\paragraph{Sparsity of feature--specification links ($X$--$Z$ links).}
Sparsity of $X$--$Z$ links is an inductive bias that makes reliable extrapolation easier,
not a prerequisite for the theory's correctness.
If $Z_i$ share no parent features, reliable extrapolation is straightforward regardless of
sparsity.
If they do share parent features, extrapolation requires positive density of the shared
parent given each $Z_i$ (Assumption~2); even with fully connected $X$--$Z$ links such
samples exist, but sparsity makes them easier to locate.
Ablating the sparsity constraint causes only minor degradation in practice
(see Appendix~\ref{app:robust_sparsity}).

\paragraph{Selection variable $S$ depends only on $Z$.}
This assumption is natural since selection defines which specification combinations are
absent from training.
If $S$ depends only on $X$, then $Z \perp S \mid X$ and $p(Z \mid X)$ can be modeled
directly, reducing to an even simpler case.
If $S$ depends on both $X$ and $Z$, the key assumptions (Assumptions~2 and~3) remain
unchanged and the main results still hold.
This assumption can therefore be relaxed without undermining the framework.

\paragraph{The conservative approximation (Assumption~3).}
Assumption~3 requires that $P(S=0)$ and $P(S=0 \mid Z_i=1)$ are both small.
This is an \emph{identifiability condition}, not an operational requirement: it
characterises the regime in which the target joint distribution under novel specification
combinations can be recovered from observed data, not the regime in which the method
produces useful samples.
When the assumption is violated, the formal identifiability guarantee lapses, but the
method continues to generate admissible samples satisfying target specifications, with
gradual rather than abrupt degradation as the degree of violation increases
(see Appendix~\ref{app:robust_conservative}).
This assumption should therefore be read as delineating the boundary of what can be
\emph{guaranteed} from data, not the boundary of \emph{applicability} of the method.

\section{Identifiability of Features and Specifications} \label{Sec:identi_X_Z}
We focus on the case where the features and the specifications are neither observed nor given. Recall that in Section~\ref{sec:formulation} we introduced the structural assumptions, including the conditional independence of the specifications given the features. Here, we want to clarify that the violation of this assumption does not forbid the application of our proposal, while it potentially hurts how precise the results would be, as the identifiability of the latent variables would be weaker.  

\label{app: identify feat and specs}

\subsection{Canonical Representation of the Data Generating Process}
\label{app: canonical representation}

\paragraph{Canonical representation of the data generation process.}
We first introduce a canonical representation of the data generation process, and then demonstrate its generality. 
\begin{definition}(Canonical Representation of the Data Generation Process)
    The features $\mathbf{X} = \{X_1, \cdots, X_n\}$ are unconditionally independent, and $\mathbf{X}\sim p_{\mathbf{X}}$.
    Each specification $Z_i \in \mathbf{Z}$, for $i=1,\cdots,d$, follows the structural equation $Z_i = h_i(\mathbf{V}_i, e_i)$, where $e_i$ is the noise term, $\mathbf{V}i := \mathrm{PA}_{Z_i}$ denotes the parent set of $Z_i$, and $\mathbf{V}_i \subseteq \mathbf{X}$.
    The observed modality data corresponding to $\mathbf{X}$ and $\mathbf{Z}$ are denoted by $\mathbf{Y_X}$ and $\mathbf{Y_Z}$, respectively. They are generated by two mixing procedures: $\mathbf{Y_X} = g_X(\mathbf{X})$, $\mathbf{Y_Z} = g_Z(\mathbf{Z})$, where $g_X : \mathcal{X}\mapsto \mathcal{Y_X}$ and $g_Z : \mathcal{Z}\mapsto \mathcal{Y_Z}$ are both invertible. The overall observed data domain is defined as $\mathcal{D} := {\mathcal{Y_X}, \mathcal{Y_Z}}$.
    \label{segde_app_def_data_generation}
\end{definition}

This canonical representation is general, in the sense that the following cases can be reduced to it. First, consider cases where there are dependencies among $\mathbf{X}$. These are not excluded by the structural assumptions in Sec. 3.1, as they do not affect the identifiability of the posterior distribution (since component-wise identifiability of $\mathbf{X}$ is not required). Suppose there is an edge $X_1\rightarrow X_3$ in the structure of Fig. 2(b). In this case, $X_3$ is a function of $X_1$ and $\epsilon_3$, assuming that $\epsilon_3$ is the noise term for $X_3$. Then, we can replace $X_3$ with $\epsilon_3$ to render the features unconditionally independent, while introducing an additional edge $X_1 \rightarrow Z_2$ as a consequence of the original edge $X_1 \rightarrow X_3$. More generally, if there are additional edges within $\mathbf{X}$, the same transformation can be applied, converting them into (possibly additional) edges from $\mathbf{X}$ to $\mathbf{Z}$, thereby reducing the system to the canonical representation. Second, consider cases where there are edges among $\mathbf{Z}$, which are explicitly excluded in Sec. 3.1. This is because dependencies among $\mathbf{Z}$ would make the identifiability of the posterior distribution of features under new specification combinations significantly more difficult to analyze. Nevertheless, we illustrate how such cases can also be reduced to the canonical representation. Suppose that there is an edge $Z_1 \rightarrow Z_2$, then $Z_2$ can be seen as a function of $Z_1$ and $e_2$, where $e_2$ is the noise term for $Z_2$. If we substitute $Z_2$ with $e_2$ to satisfy the conditional independence constraint, then there will be an additional link $X_3\rightarrow Z_1$. Similarly, in more general settings, the same transformation can be applied and transform the edges between $\mathbf{Z}$ to possibly extra edges from $\mathbf{X}$ to $\mathbf{Z}$, to reduce to the canonical representation as defined in Definition~\ref{segde_app_def_data_generation}. The canonical representation enables a clearer analysis of identifiability for both the latent variables and the posterior distributions.

\paragraph{Existing related identifiability results.}\cite{conceptaligner} gives conditions for identifiability for both the features $\mathbf{X}$ and the specifications $\mathbf{Z}$. In our work, we share the part where the specifications are required to be identifiable in order to decompose the constraint of a novel combination of specifications into individual ones, and the individual constraints are already defined as seen in the data (recall that there is no novel value for any specification $Z_i$). The component-wise identifiability of $\mathbf{Z}$ is mostly built on the conditional independence of each $Z_i$ given $\mathbf{X}$ (see Condition 4.2 and Theorem 4.4 in their work). However, Condition 4.3 for deriving component-wise identifiability of $\mathbf{X}$ is rather strong in our setting, and our extrapolated data generation does not require each $X_i$ to be identifiable. Instead, to guarantee a sufficient and clean estimate of the conditional distribution $p(Z_i \mid \mathbf{V}_i)$, we need to locate the subspace of $\mathbf{V}_i$. Such a property should be facilitated by the sparsity constraint of the functions from $\mathbf{X}$ to $Z_i$. At the same time, features that connect to multiple specifications should be easier to identify.

Formally, we derive the identifiability of both specifications $\mathbf{Z}$ and subspace identifiability of the features $\mathbf{X}$.

\paragraph{Logic of proving identifiability in this paper.}
The high-level idea for establishing identifiability of the latent variables is as follows. Assume the observed joint distribution $p(\mathbf{Y_X}, \mathbf{Y_Z})$ is generated following the canonical representation of the data generation process in Definition~\ref{segde_app_def_data_generation}, and we learn a model $(\hat{g}_X, \hat{g}_Z, p_{\hat{\mathbf{X}}}, p_{\hat{\mathbf{Z}}\mid\hat{\mathbf{X}}})$ that assumes the same process as in Definition~\ref{segde_app_def_data_generation} and matches the true joint distribution, i.e.,
\begin{align*}
    p_{\mathbf{Y_X,Y_Z}}(\mathbf{y_x},\mathbf{y_z}) = p_{\hat{\mathbf{Y}}_X,\hat{\mathbf{Y}}_Z}(\mathbf{y_x},\mathbf{y_z}), \forall y_x, y_z\in \mathcal{D},
\end{align*}
where the random variables $\mathbf{Y_X,Y_Z}$ are generated from the true process $(g_X, g_Z, p_{{\mathbf{X}}}, p_{{\mathbf{Z}}\mid{\mathbf{X}}})$, and $\hat{\mathbf{Y}}_X,\hat{\mathbf{Y}}_Z$ are generated from the estimated process $(\hat{g}_X, \hat{g}_Z, p_{\hat{\mathbf{X}}}, p_{\hat{\mathbf{Z}}\mid\hat{\mathbf{X}}})$. In the following, leveraging this distributional equivalence along with additional assumptions, we establish relationships between $\hat{\mathbf{Z}}$ and $\mathbf{Z}$, and between $\hat{\mathbf{X}}$ and $\mathbf{X}$. 

\subsection{Component-wise Identifiability of Specifications $\mathbf{Z}$ via Conditional Independence}
The following theorem is adapted from ConceptAligner~\cite{conceptaligner} (Step 1 of Theorem 4.4), where the key idea is to exploit conditional independence to identify the latent variables.
\begin{theorem}(Component-wise Identifiability of Specifications)
If the following conditions C1-3 hold for the canonical data generation process, then the specification variables $\mathbf{Z}$ are component-wise identifiable.
\begin{itemize}
    \item C1 (Third-order differentiable and positive). For every $\mathbf{x} \in \mathbb{R}^n$, the conditional density $p_{\mathbf{Z} \mid \mathbf{X}}(\mathbf{z} \mid \mathbf{x})$ is strictly positive and third-order continuously differentiable with respect to $\mathbf{z} \in \mathbb{R}^d$ and $\mathbf{x} \in \mathbb{R}^n$.
    \item C2 (Conditional independence). The specification variables $\mathbf{Z}={Z_1,\cdots,Z_d}$ are conditionally independent given the features $\mathbf{X}$.   
    \item C3 (Sufficient variability on $\mathbf{Z}$). For each fixed value of $\mathbf{Z}$, there exist $2d+1$ distinct values of $\mathbf{X}$, denoted by $\mathbf{x}^{(0)}, \mathbf{x}^{(1)}, \cdots, \mathbf{x}^{(2d)}$, such that the $2d$ vectors $l(\mathbf{z},\mathbf{x}^{(m)}) = w(\mathbf{z},\mathbf{x}^{(m)}) - w(\mathbf{z},\mathbf{x}^{(0)}) \in \mathbb{R}^{2d}, m=1, \cdots, 2d$, where 
    \begin{align*}
        w(\mathbf{z}, \mathbf{x}^{(m)}) = (\frac{\partial  \log p(\mathbf{Z}_i\mid\mathbf{X})}{\partial Z_i}, \frac{\partial^2 \log p(\mathbf{Z}_i\mid\mathbf{X})}{\partial Z_i^2})_{i\in[d]} \mid_{\mathbf{X}=\mathbf{x}^{(m)}}.
    \end{align*}
\end{itemize}
\label{thm: Z_component_wise_identifiability}
\end{theorem}

\begin{proof}

To establish identifiability, we consider an alternative parameterization $(\hat{g}_X, \hat{g}_Z, p_{\hat{\mathbf{X}}}, p_{\hat{\mathbf{Z}} \mid \hat{\mathbf{X}}})$ that induces the same observed distribution as the true model $(g_X, g_Z, p_{\mathbf{X}}, p_{\mathbf{Z} \mid \mathbf{X}})$, i.e.,
\begin{align*}
p_{\mathbf{Y_X},\mathbf{Y_Z}}(\mathbf{y_x},\mathbf{y_z}) = p_{\hat{\mathbf{Y}}_X,\hat{\mathbf{Y}}Z}(\mathbf{y_x},\mathbf{y_z}), \quad \forall \mathbf{y_x}, \mathbf{y_z} \in \mathcal{D}.
\end{align*}
Unless otherwise stated, we write $p_{\mathbf{A}}(\mathbf{a})$ as $p(\mathbf{A})$, and similarly for joint and conditional densities.
We will show that this implies that $\hat{\mathbf{Z}}$ and $\mathbf{Z}$ coincide up to a component-wise transformation and permutation.

Since $\mathbf{Y_X} = g_X(\mathbf{X}), \mathbf{Y_Z} = g_Z(\mathbf{Z})$ and since $g_X$ and $g_Z$ are invertible, we have the following equality via the change-of-variables formula,
\begin{align*}
    p(\mathbf{Y_X}) = p(\mathbf{X}) |\mathrm{det} J_{{g_X}^{-1}}\mid, p(\mathbf{Y_Z}) = p(\mathbf{Z}) |\mathrm{det} J_{{g_Z}^{-1}}\mid \Rightarrow p(\mathbf{Y_X},\mathbf{Y_Z}) = p(\mathbf{X}, \mathbf{Z}) |\mathrm{det} J_{{g_X}^{-1}}| | \mathrm{det} J_{{g_Z}^{-1}}\mid,
\end{align*}
where $J_{{g_X}^{-1}}$ and $J_{{g_Z}^{-1}}$ are the Jacobian matrices of the inverse of $g_X$ and $g_Z$. Similarly, 
\begin{align*}
    p(\mathbf{\hat{Y}_X}) = p(\mathbf{\hat{X}}) |\mathrm{det} J_{{\hat{g}_X}^{-1}}\mid, p(\mathbf{\hat{Y}_Z}) = p(\mathbf{\hat{Z}}) |\mathrm{det} J_{{\hat{g}_Z}^{-1}}\mid \Rightarrow p(\mathbf{\hat{Y}_X},\mathbf{\hat{Y}_Z}) = p(\mathbf{\hat{X}}, \mathbf{\hat{Z}}) |\mathrm{det} J_{{\hat{g}_X}^{-1}}| | \mathrm{det} J_{{\hat{g}_Z}^{-1}}\mid.
\end{align*}
Due to the observational equivalence of the joint distributions $p(\mathbf{\hat{Y}_X},\mathbf{\hat{Y}_Z}) = p(\mathbf{Y_X},\mathbf{Y_Z})$ which implies observational equivalence of the marginal distributions:
\begin{align*}
    p(\mathbf{\hat{X}}) = p(\mathbf{X}) |\mathrm{det} J_{{\hat{g}_X}^{-1}\circ{g}_X}\mid ,
    p(\mathbf{\hat{Z}}) = p(\mathbf{Z}) |\mathrm{det} J_{{\hat{g}_Z}^{-1}\circ{g}_Z}\mid ,
    p(\mathbf{\hat{X}},\mathbf{\hat{Z}}) = p(\mathbf{X}, \mathbf{Z})|\mathrm{det} J_{{\hat{g}_X}^{-1}\circ{g}_X}| |\mathrm{det} J_{{\hat{g}_Z}^{-1}\circ{g}_Z}\mid. 
\end{align*}
Therefore, 
\begin{align}
    p(\mathbf{\hat{Z}}\mid\mathbf{\hat{X}}) = \frac{p(\mathbf{\hat{X}},\mathbf{\hat{Z}})}{p(\mathbf{\hat{X}})} = \frac{p(\mathbf{X}, \mathbf{Z})|\mathrm{det} J_{{\hat{g}_X}^{-1}\circ{g}_X}| |\mathrm{det} J_{{\hat{g}_Z}^{-1}\circ{g}_Z}|}{p(\mathbf{X}) |\mathrm{det} J_{{\hat{g}_X}^{-1}\circ{g}_X}|} = p(\mathbf{Z}\mid\mathbf{X})|\mathrm{det} J_{{\hat{g}_Z}^{-1}\circ{g}_Z}|.
    \label{app_eq:sedge_conditional_density_equivalence}
\end{align}

Since, in the data-generating process, $\mathbf{Z}={Z_1,\cdots, Z_d}$ are assumed to be conditionally independent given $\mathbf{X}$, and the estimated model is assumed the same way, we have, $p(\mathbf{\hat{Z}}\mid\mathbf{\hat{X}}) = \prod_{i\in[d]} p(\hat{Z}_i\mid\hat{\mathbf{X}})$, and the second cross derivative of the log conditional density is zero, i.e, 
\begin{align}
   \frac{\partial^2 \log p(\mathbf{\hat{Z}}\mid\mathbf{\hat{X}})}{\partial \hat{Z}_k \partial \hat{Z}_l} = \frac{\partial^2 \sum_{i\in[d]} \log p(\mathbf{\hat{Z}}_i\mid\mathbf{\hat{X}})}{\partial \hat{Z}_k \partial \hat{Z}_l} \equiv 0, \forall k\neq l, k,l\in[d].
\end{align}
Substituting (\ref{app_eq:sedge_conditional_density_equivalence}) into the second-order cross derivative of any pair $k\neq l, k,l\in[d]$, we have,
\begin{align}
   0 &= \frac{\partial^2 \log p(\mathbf{\hat{Z}}\mid\mathbf{\hat{X}})}{\partial \hat{Z}_k \partial \hat{Z}_l} \\
   &= \sum_{i\in[d]} \frac{\partial  \log p(Z_i\mid\mathbf{X})}{\partial Z_i} \frac{\partial^2 Z_i}{\partial \hat{Z}_k \partial \hat{Z}_l} + \sum_{i\in[d]} \frac{\partial^2  \log p(Z_i\mid\mathbf{X})}{\partial Z_i^2} \frac{\partial Z_i}{\partial \hat{Z}_l} \frac{\partial Z_i}{\partial \hat{Z}_k} + \\
   & \sum_{i,j\in[d], i\neq j} \frac{\partial^2  \log p(Z_i\mid\mathbf{X})}{\partial Z_i Z_j} (\frac{\partial Z_i}{\partial \hat{Z}_l} \frac{\partial Z_i}{\partial \hat{Z}_k} + \frac{\partial Z_j}{\partial \hat{Z}_l} \frac{\partial Z_i}{\partial \hat{Z}_k}) + \frac{\partial^2 |\mathrm{det} J_{{\hat{g}_Z}^{-1}\circ{g}_Z}\mid}{\partial \hat{Z}_k \partial \hat{Z}_l},
   \label{app_eq: sedge_app_linear_system_thm1}
\end{align}
where $\frac{\partial^2  \log p(\mathbf{Z}_i\mid\mathbf{X})}{\partial Z_i Z_j} = 0$ due to conditional independence.  
Now, for any fixed values of $Z_i, i\in[d]$, suppose we have $2d+1$ different values of $\mathbf{X}$, i.e., $\mathbf{x}^{(0)}, \mathbf{x}^{(1)}, \cdots, \mathbf{x}^{(m)}, \cdots, \mathbf{x}^{(2d)}$, so that we can construct $2d$ linearly independent vectors:
\begin{align*}
   l(\mathbf{z},\mathbf{x}^{(m)}) = w(\mathbf{z},\mathbf{x}^{(m)}) - w(\mathbf{z},\mathbf{x}^{(0)}) \in \mathbb{R}^{2d}, 
\end{align*}
where 
\begin{align*}
    w(\mathbf{z}, \mathbf{x}^{(m)}) = (\frac{\partial  \log p(\mathbf{Z}_i\mid\mathbf{X})}{\partial Z_i}, \frac{\partial^2 \log p(\mathbf{Z}_i\mid\mathbf{X})}{\partial Z_i^2})_{i\in[d]} \mid_{\mathbf{X}=\mathbf{x}^{(m)}}. 
\end{align*}

By condition C3, we obtain $2d$ linearly independent vectors $l(\mathbf{z},\mathbf{x}^{(m)})$. This allows us to solve the linear system in (\ref{app_eq: sedge_app_linear_system_thm1}), by subtracting $w(\mathbf{z}, \mathbf{x}^{(0)})$ from each of the $w(\mathbf{z}, \mathbf{x}^{(m)}); m=1,\cdots,2d$ rows. We conclude that, for all $Z_i$, $\frac{\partial^2 Z_i}{\partial \hat{Z}_k \partial \hat{Z}_l} = 0$ and $\frac{\partial Z_i}{\partial \hat{Z}_k} \frac{\partial Z_i}{\partial \hat{Z}_l} = 0$ for any pair $k \neq l$. Therefore, each $Z_i$ can depend on at most one component of $\hat{\mathbf{Z}}$. Therefore, there exists a permutation of $[d]$, i.e., $\pi$, so that for each $Z_i, i=1, \cdots, d$, it is a function solely of $\hat{Z}_{\pi(i)}$. This establishes that $\mathbf{Z}$ is component-wise identifiable.
\end{proof}
 
\subsection{Subspace Identifiability of Features $\mathbf{X}$ via Sparse Structure}

\begin{theorem}(Subspace Identifiability of Features)
If conditions C1–C2 in Theorem~\ref{thm: Z_component_wise_identifiability} hold, together with the following conditions C4–C5, then the feature variables $\mathbf{X}$ are identifiable in the following sense:
\begin{itemize}
    \item C4 (Sufficient variability). For each $Z_k \in \mathbf{Z}$ and each $X_j \in \mathbf{V}_{\pi(k)}$, and for every fixed value of the parent variables $\{X_j : X_j \in \mathbf{V}_{\pi(k)}\}$, there exist $|\mathbf{V}_{\pi(k)}|$ distinct values of $Z_{\pi(k)}$, denoted by $z^{(1)}, \dots, z^{(|\mathbf{V}_{\pi(k)}|)}$, such that the vectors
    \begin{align*}
        u(\mathbf{x}, z_{\pi(k)}) 
        = \left( \frac{\partial \eta'_{\pi(k)}}{\partial X_j} \right)_{X_j \in \mathbf{V}_{\pi(k)}} 
        \Big|_{Z_{\pi(k)} = z^{(m)}}
    \end{align*}
    are linearly independent.
    
    \item C5 (Sparsity). The model $p(\hat{\mathbf{Z}} \mid \hat{\mathbf{X}})$ is estimated under a sparsity constraint on the dependency structure from $\hat{\mathbf{X}}$ to each component of $\hat{\mathbf{Z}}$.
\end{itemize}

Under these conditions, the feature variables $\mathbf{X}$ satisfy the following identifiability properties:
\begin{itemize}
    \item Case 1. If $X_j$ is not a parent of any specification (i.e., $X_j$ is an isolated feature variable), then $X_j$ can be an arbitrary function of $\hat{\mathbf{X}}$, and is therefore not identifiable.
    
    \item Case 2. If $X_j$ is a parent of exactly one specification $Z_{\pi(k)}$, then $X_j$ can depend only on the estimated parent set $\hat{\mathbf{V}}_k$ of $\hat{Z}_k$. Under C5, the size of this admissible subspace matches that of the true parent set, i.e., $|\hat{\mathbf{V}}_k| = |\mathbf{V}_{\pi(k)}|$.
    
    \item Case 3. If $X_j$ is a parent of multiple specifications, then $X_j$ can depend only on the intersection of the estimated parent sets corresponding to those specifications. This intersection is nonempty, since the estimated counterpart of $X_j$ must belong to each such parent set. In particular, if the intersection contains only one element, then $X_j$ is component-wise identifiable.
\end{itemize}
\label{thm: X_identifiability}
\end{theorem}

\begin{proof}
We aim to characterize, for each true feature $X_j$, which estimated features in $\hat{\mathbf{X}}$ it may depend on. Using the component-wise identifiability of $\mathbf{Z}$, we first relate each estimated specification $\hat{Z}_k$ to its true counterpart $Z_{\pi(k)}$, and then use the local parent structure to restrict the possible dependence of $X_j$ on $\hat{\mathbf{X}}$.

Due to the structural assumptions on both the true and estimated models, we have
\[
p(\mathbf{Z}\mid \mathbf{X}) = \prod_{i\in[d]}p(Z_i\mid\mathbf{V}_i), \quad
p(\hat{\mathbf{Z}}\mid \hat{\mathbf{X}}) = \prod_{i\in[d]}p(\hat{Z}_i\mid\hat{\mathbf{V}}_i).
\]
The derivative of $\log p(\hat{\mathbf{Z}}\mid \hat{\mathbf{X}})$ with respect to any $\hat{Z}_k \in \hat{\mathbf{Z}}$ is therefore
\begin{align*}
    \frac{\partial \log p(\hat{\mathbf{Z}}\mid \hat{\mathbf{X}})}{\partial \hat{Z}_k} 
    = \frac{\partial \log p(\hat{Z}_k\mid \hat{\mathbf{V}}_k)}{\partial \hat{Z}_k}.
\end{align*}

Define
\begin{align*}
    \hat{\eta}^{\prime}_k 
    := \frac{\partial \log p(\hat{\mathbf{Z}}\mid \hat{\mathbf{X}})}{\partial \hat{Z}_k} 
    = \frac{\partial \log p(\hat{Z}_k\mid \hat{\mathbf{V}}_k)}{\partial \hat{Z}_k}.
\end{align*}
By making use of the component-wise identifiability of $\mathbf{Z}$, we have
\begin{align*}
    \hat{\eta}^{\prime}_k 
    &= \frac{\partial \log p(\mathbf{Z}\mid\mathbf{X})}{\partial \hat{Z}_k} 
    + \frac{\partial \log |\mathrm{det} J_{{\hat{g}_Z}^{-1}\circ g_Z}|}{\partial \hat{Z}_k} \\
    &= \sum_{i\in[d]} \eta_{i}^{\prime} \frac{\partial Z_i}{\partial \hat{Z}_k} 
    + \frac{\partial \log |\mathrm{det} J_{{\hat{g}_Z}^{-1}\circ g_Z}|}{\partial \hat{Z}_k}.
\end{align*}
Because $\mathbf{Z}$ is component-wise identifiable, only when $i=\pi(k)$ do we have $\frac{\partial Z_i}{\partial \hat{Z}_k} \neq 0$. Therefore,
\begin{align*}
    \hat{\eta}^{\prime}_k 
    = \eta_{\pi(k)}^{\prime} \frac{\partial Z_{\pi(k)}}{\partial \hat{Z}_k} 
    + \frac{\partial \log |\mathrm{det} J_{{\hat{g}_Z}^{-1}\circ g_Z}|}{\partial \hat{Z}_k}.
\end{align*}

For any $\hat{X}_l \notin \hat{\mathbf{V}}_k$, the derivative of $\hat{\eta}^{\prime}_k$ with respect to $\hat{X}_l$ is zero, that is,
\begin{align*}
    0 
    &\equiv \frac{\partial \hat{\eta}^{\prime}_k}{\partial \hat{X}_l} \\
    &= \eta_{\pi(k)}^{\prime} \frac{\partial^2 Z_{\pi(k)}}{\partial \hat{Z}_k \partial \hat{X}_l} +
    \frac{\partial \eta^{\prime}_{\pi(k)}}{\partial Z_{\pi(k)}} \frac{{\partial Z_{\pi(k)}}}{\partial \hat{X}_l}\frac{\partial Z_{\pi(k)}}{\partial \hat{Z}_k}
    + \sum_{j: X_j \in \mathbf{V}_{\pi(k)}} 
    \frac{\partial \eta_{\pi(k)}^{\prime}}{\partial X_j} 
    \frac{\partial X_j}{\partial \hat{X}_l} 
    \frac{\partial Z_{\pi(k)}}{\partial \hat{Z}_k} + \frac{\partial^2 \log |\mathrm{det} J_{{\hat{g}_Z}^{-1}\circ g_Z}|}{\partial \hat{Z}_k \partial \hat{X}_l},
\end{align*}
where $\frac{\partial^2 Z_{\pi(k)}}{\partial \hat{Z}_k \partial \hat{X}_l} = 0$, $\quad 
\frac{\partial^2 \log |\mathrm{det} J_{{\hat{g}_Z}^{-1}\circ g_Z}|}{\partial \hat{Z}_k \partial \hat{X}_l} = 0$, $\frac{\partial \eta_{\pi(k)}^{\prime}}{\partial X_j} \neq 0$, and $\frac{\partial Z_{\pi(k)}}{\partial \hat{Z}_k} \neq 0$. Besides, we derive:
\begin{align*}
    \frac{\partial Z_{\pi(k)}}{\partial \hat{X}_l} = \frac{\partial Z_{\pi(k)}}{\partial \hat{Z}_k}\frac{\partial \hat{Z}_k}{\partial \hat{X}_l} + \sum_{j:X_j\in \mathbf{V}_{\pi(k)}} \frac{\partial Z_{\pi(k)}}{\partial X_j}\frac{\partial X_j}{\partial \hat{X}_l} = \sum_{j:X_j\in \mathbf{V}_{\pi(k)}} \frac{\partial Z_{\pi(k)}}{\partial X_j}\frac{\partial X_j}{\partial \hat{X}_l}.
\end{align*}

Therefore, the above equation simplifies to
\begin{align*}
    0 
    &= \sum_{j: X_j \in \mathbf{V}_{\pi(k)}} 
    \frac{\partial \eta_{\pi(k)}^{\prime}}{\partial X_j} 
    \frac{\partial X_j}{\partial \hat{X}_l} 
    \frac{\partial Z_{\pi(k)}}{\partial \hat{Z}_k} + \sum_{j:X_j\in \mathbf{V}_{\pi(k)}} \frac{\partial\eta^{\prime}_{\pi(k)}}{\partial Z_{\pi(k)}}\frac{\partial Z_{\pi(k)}}{\partial X_j}\frac{\partial X_j}{\partial \hat{X}_l}\frac{\partial Z_{\pi(k)}}{\partial \hat{Z}_k}\\
    &= 2 \sum_{j: X_j \in \mathbf{V}_{\pi(k)}} 
    \frac{\partial \eta_{\pi(k)}^{\prime}}{\partial X_j} 
    \frac{\partial X_j}{\partial \hat{X}_l} 
    \frac{\partial Z_{\pi(k)}}{\partial \hat{Z}_k}.
\end{align*}

To solve this linear system and recover the relationship between $\mathbf{X}$ and $\hat{\mathbf{X}}$, we require sufficient variability in $Z_{\pi(k)}$. This is formalized in Condition C4. That is, for every $Z_k \in \mathbf{Z}$ and for every fixed value of the parent variables $\{X_j : X_j \in \mathbf{V}_{\pi(k)}\}$, there exist $|\mathbf{V}_{\pi(k)}|$ distinct values of $Z_{\pi(k)}$, i.e., $z^{(1)}, \cdots, z^{(|\mathbf{V}_{\pi(k)}|)}$, such that we obtain $|\mathbf{V}_{\pi(k)}|$ linearly independent vectors $u(\mathbf{x}, z_{\pi(k)})$, where
\begin{align*}
    u(\mathbf{x}, z_{\pi(k)}) 
    = \left( \frac{\partial \eta_{\pi(k)}^{\prime}}{\partial X_j} \right)_{X_j \in \mathbf{V}_{\pi(k)}} 
    \Big|_{Z_{\pi(k)} = z^{(m)}}.
\end{align*}

Under Condition C4, the linear system admits a unique solution, implying that
\[
\frac{\partial X_j}{\partial \hat{X}_l} = 0, \quad \forall \hat{X}_l \notin \hat{\mathbf{V}}_k.
\]
This means that for any $X_j \in \mathbf{V}_{\pi(k)}$, it cannot depend on any $\hat{X}_l \notin \hat{\mathbf{V}}_k$. Equivalently, $X_j$ can only be a function of the intersection of the parent sets of $\hat{Z}_k$ corresponding to its children.

Without the sparsity constraint, we may have $|\hat{\mathbf{V}}_k| \geq |\mathbf{V}_{\pi(k)}|$. Condition C5 enforces the reverse inequality, and therefore
\[
|\hat{\mathbf{V}}_k| = |\mathbf{V}_{\pi(k)}|.
\]

Together, these conclusions imply the three cases stated in the theorem:
\begin{itemize}
    \item Case 1. When $X_j$ is not a parent of any specification (i.e., $X_j$ is an isolated feature variable), $X_j$ can be an arbitrary function of $\hat{\mathbf{X}}$, and is therefore not identifiable. 
    \item Case 2. When $X_j$ is the parent of only one specification, i.e., $Z_{\pi(k)}$, it is a function of the parent set $\hat{\mathbf{V}}_k$ of $\hat{Z}_k$. Condition C5 further restricts the possible mixing by enforcing $|\hat{\mathbf{V}}_k| = |\mathbf{V}_{\pi(k)}|$.
    \item Case 3. When $X_j$ is the parent of multiple specifications, it is a function of the intersection of the parent sets of the corresponding estimated specifications. This intersection is nonempty, since the corresponding representation of $X_j$ must belong to each such parent set. In particular, when the intersection contains only one element, $X_j$ is component-wise identifiable.
\end{itemize}

Finally, we illustrate the result using the example in Fig. 1(b). The theorem can be interpreted as imposing structural constraints on how each true feature $X_j$ depends on the estimated features $\hat{\mathbf{X}}$. In particular, it restricts which components of $\hat{\mathbf{X}}$ can represent each $X_j$, as determined by the parent structure of the corresponding specifications. Since $\mathbf{X}$ and $\hat{\mathbf{X}}$ are related by an invertible transformation, these restrictions on $\hat{\mathbf{X}}$ equivalently constrain how each $X_j$ can be represented.

In this example, $X_1$ and $X_3$ each have exactly one child specification and therefore fall into Case 2. The theorem implies that each of them can depend only on the estimated features in the corresponding parent set $\hat{\mathbf{V}}_k$, meaning that their dependence on $\hat{\mathbf{X}}$ is restricted to the corresponding subsets of estimated features. Consequently, their representations are identifiable only up to the subspaces spanned by those features. Concretely, $X_1$ is identifiable up to the subspace of $\hat{\mathbf{X}}$ spanned by the components corresponding (under the learned mapping) to $\{X_1, X_2\}$, and similarly $X_3$ is identifiable up to the subspace corresponding to $\{X_2, X_3\}$.

In contrast, $X_2$ is a parent of multiple specifications and falls into Case 3. The theorem implies that $X_2$ can depend only on the intersection of the estimated parent sets corresponding to its children, meaning that its dependence on $\hat{\mathbf{X}}$ is restricted to this intersection. In this example, the intersection contains a single element, and therefore the representation of $X_2$ depends on only one component of $\hat{\mathbf{X}}$. Consequently, $X_2$ is component-wise identifiable.
\end{proof}

\section{Appendix of Structured Extrapolated Data Generation Algorithms}
\label{app: sedge algorithm}

\begin{algorithm}[H]
\caption{Structure-Informed Extrapolated Sampling (Optimization-Based)}
\label{alg:sedge-opt}
\textbf{Input:}
target specifications $\mathbf{z} \in \mathbb{R}^{n}$;
forward models $\{h_i: \mathbf{V}_i \to Z_i\}_{i=1}^n$ where $\mathbf{V}_i \subseteq \mathbf{X}$ and $h_i$ consist $h_i^{\mu}$ and $h_i^{\sigma}$;
prior diffusion model $p_{\mathbf{X}}$ trained on $\mathbf{X}$;
normalization statistics $\mu_{\mathbf{X}}, \sigma_{\mathbf{X}}, \mu_{\mathbf{Z}}, \sigma_{\mathbf{Z}}$;
learning rate $\eta > 0$ and optimization step $T_{\mathrm{opt}}$.

\textbf{Output:}
extrapolated sample $\mathbf{x} \in \mathbb{R}^{d}$.

\begin{enumerate}
    \item \textbf{Preprocessing:}
     Normalize target: $\hat{\mathbf{z}} \leftarrow (\mathbf{z} - \mu_{\mathbf{Z}}) / \sigma_{\mathbf{Z}}$.

    \item \textbf{Method 1: Optimization-based posterior sampling}
    \begin{enumerate}
        \item Initialize from training data: $\mathbf{x}^{(0)} \gets \mathbf{X}_{\text{train}}[\text{random\_idx}]$ (normalized).
        \item For $t = 1, \ldots, T_{\text{opt}}$ (e.g., $T_{\text{opt}} = 1000$):
        \begin{enumerate}
            \item Compute forward predictions: each dimension of $\hat{\mathbf{z}}$ is computed by $\hat{\mathbf{z}}_i^{(t)} \leftarrow f_i(\mathbf{V}_i^{(t)})$ with gradients enabled.
                \item Compute likelihood loss normalized space: $\mathcal{L}_{g}^{(t)} = \left\| \hat{\mathbf{z}}^{(t)}  - \mathbf{z}
                \right\|_2^2$ (batch iteration is ignored).
            \item Compute loss in the normalized space: $\mathcal{L}_{g}^{(t)} = \frac{1}{n} \sum_{i=1}^{n} \left\|\hat{\mathbf{z}}_i^{(t)} - \mathbf{z}\right\|_2^2$.
            \item Compute gradient: $\mathbf{g}^{(t)} \leftarrow \nabla_{\mathbf{x}} -\mathcal{L}_{g}^{(t)}$.
            \item Update with Adam optimizer: $\mathbf{x}^{(t)} \leftarrow \text{Adam}(\mathbf{x}^{(t-1)}, \mathbf{g}^{(t)}, \eta)$.
        \end{enumerate}
        \item Denormalize: $\mathbf{x} \leftarrow \sigma_{\mathbf{X}} \cdot \mathbf{x}^{(T_{\text{opt}})} + \mu_{\mathbf{X}}$.
    \end{enumerate}

\end{enumerate}

\end{algorithm}

\begin{algorithm}[H]
\caption{Structure-Informed Extrapolated Sampling (Diffusion Posterior Sampling)~\cite{chung2022dps}}
\label{alg:sedge-dps}
\textbf{Input:}
target specifications $\mathbf{z} \in \mathbb{R}^{n}$;
forward models $\{h_i: \mathbf{V}_i \to Z_i\}_{i=1}^n$ where $\mathbf{V}_i \subseteq \mathbf{X}$;
diffusion model $p_{\mathbf{X}}$ trained on $\mathbf{X}$;
normalization statistics $\mu_{\mathbf{X}}, \sigma_{\mathbf{X}}, \mu_{\mathbf{Z}}, \sigma_{\mathbf{Z}}$;
guidance scale $\lambda > 0$ and  diffusion steps $T$ (DPS).

\textbf{Output:}
extrapolated sample $\mathbf{x} \in \mathbb{R}^{d}$.

\begin{enumerate}
    \item \textbf{Preprocessing:}
    Normalize target: $\hat{\mathbf{z}} \leftarrow (\mathbf{z} - \mu_{\mathbf{Z}}) / \sigma_{\mathbf{Z}}$.
    
    \item \textbf{Method 2: Diffusion posterior sampling (DPS)}
    \begin{enumerate}
        \item Initialize from noise: $\mathbf{x}^{(T)} \sim \mathcal{N}(\mathbf{0}, \mathbf{I})$.
        \item For $t = T, T-1, \ldots, 1$:
        \begin{enumerate}
            \item \textbf{Prior denoising step:}
            \begin{enumerate}
                \item Predict noise: $\hat{\boldsymbol{\epsilon}}^{(t)} \leftarrow \epsilon_\theta(\mathbf{x}^{(t)}, t)$ using prior model.
                \item Compute predicted mean:
               $ \boldsymbol{\mu}^{(t)} \leftarrow \frac{1}{\sqrt{\alpha_t}} \left(
                \mathbf{x}^{(t)} - \frac{\beta_t}{\sqrt{1 - \bar{\alpha}_t}} \hat{\boldsymbol{\epsilon}}^{(t)}
                \right)$
            \end{enumerate}
            \item \textbf{Likelihood guidance:}
            \begin{enumerate}
                \item Compute forward predictions: each dimension of $\hat{\mathbf{z}}$ is computed by $\hat{\mathbf{z}}_i^{(t)} \leftarrow h_i(\mathbf{V}_i^{(t)})$ with gradients enabled.
                \item Compute likelihood loss normalized space: $\mathcal{L}_{g}^{(t)} = \left\| \hat{\mathbf{z}}^{(t)}  - \mathbf{z}
                \right\|_2^2$ (batch iteration is ignored).
                \item Compute guidance gradient: $\mathbf{g}_{\text{like}}^{(t)} \leftarrow \nabla_{\mathbf{x}} \mathcal{L}_{\text{like}}^{(t)}$.
                \item Apply guidance: $\boldsymbol{\mu}^{(t)} \leftarrow \boldsymbol{\mu}^{(t)} - \lambda \cdot \mathbf{g}_{\text{like}}^{(t)}$.
            \end{enumerate}
            \item \textbf{Add noise and update:}
            \[
            \mathbf{x}^{(t-1)} \leftarrow 
            \begin{cases}
            \boldsymbol{\mu}^{(t)} + \sqrt{\beta_t} \cdot \boldsymbol{\epsilon}, & \text{if } t > 1, \quad \boldsymbol{\epsilon} \sim \mathcal{N}(\mathbf{0}, \mathbf{I}) \\
            \boldsymbol{\mu}^{(t)}, & \text{if } t = 1
            \end{cases}
            \]
        \end{enumerate}
        \item Denormalize: $\mathbf{x} \leftarrow \sigma_{\mathbf{X}} \cdot \mathbf{x}^{(0)} + \mu_{\mathbf{X}}$.
    \end{enumerate}
\end{enumerate}
\end{algorithm}

\section{Appendix on Experiment Details and Additional Results}
\subsection{Synthetic Experiments}
\label{Sec:F1}
The split of the data set given the distribution of $\mathbf{Z}$ is given in Fig.~\ref{fig:app data split}, where the marginal distribution of each $Z_i$ is complete, but only some parts of the joint distribution are unseen in the given dataset. And out data generation setting ensures that to satisfy such novel combinations of specifications, the ability to do extrapolated generation is necessary, as shown in Fig~\ref{fig:synthetic} (a-b). 

\begin{figure}
    \centering
    \includegraphics[width=0.8\linewidth]{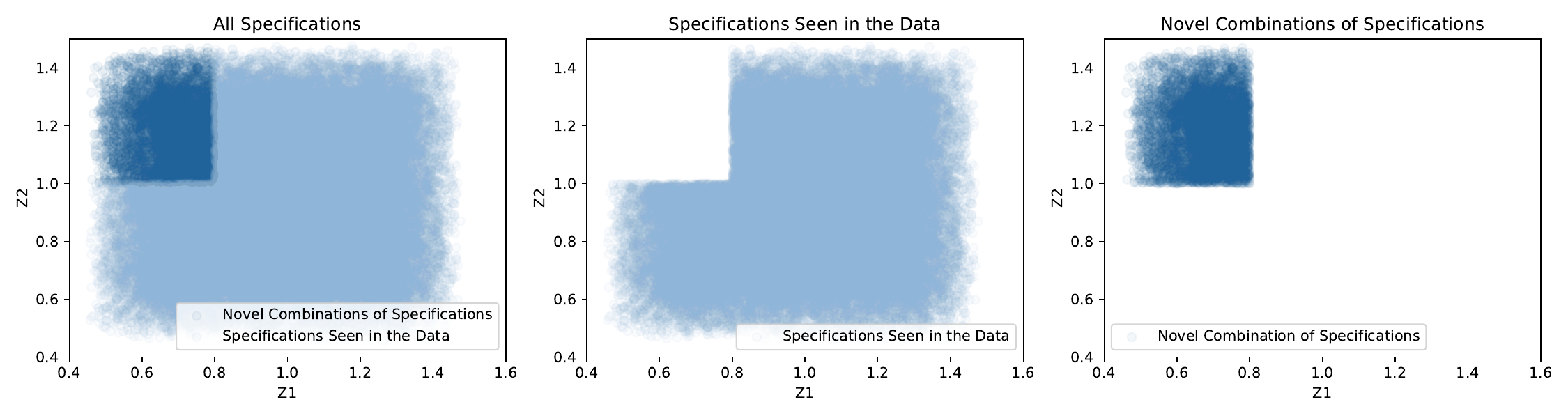}
    \caption{Data Split Based on Specifications}
    \label{fig:app data split}
\end{figure}

\subsubsection{More Results on Given Features and Specifications}
Fig.~\ref{fig: app complete given 2d optimization} and Fig.~\ref{fig: app complete given 2d dps} show more 2-d views of the data points generated by our method (model A) and other baselines. From both diagrams, it is clear to see that the samples generated by our method are the most close to the oracle $\mathbf{X}$ and achieves extrapolation on novel combination of specifications.   
\label{app: synthetic given 2d}
\begin{figure}
    \centering
    \includegraphics[width=0.7\linewidth]{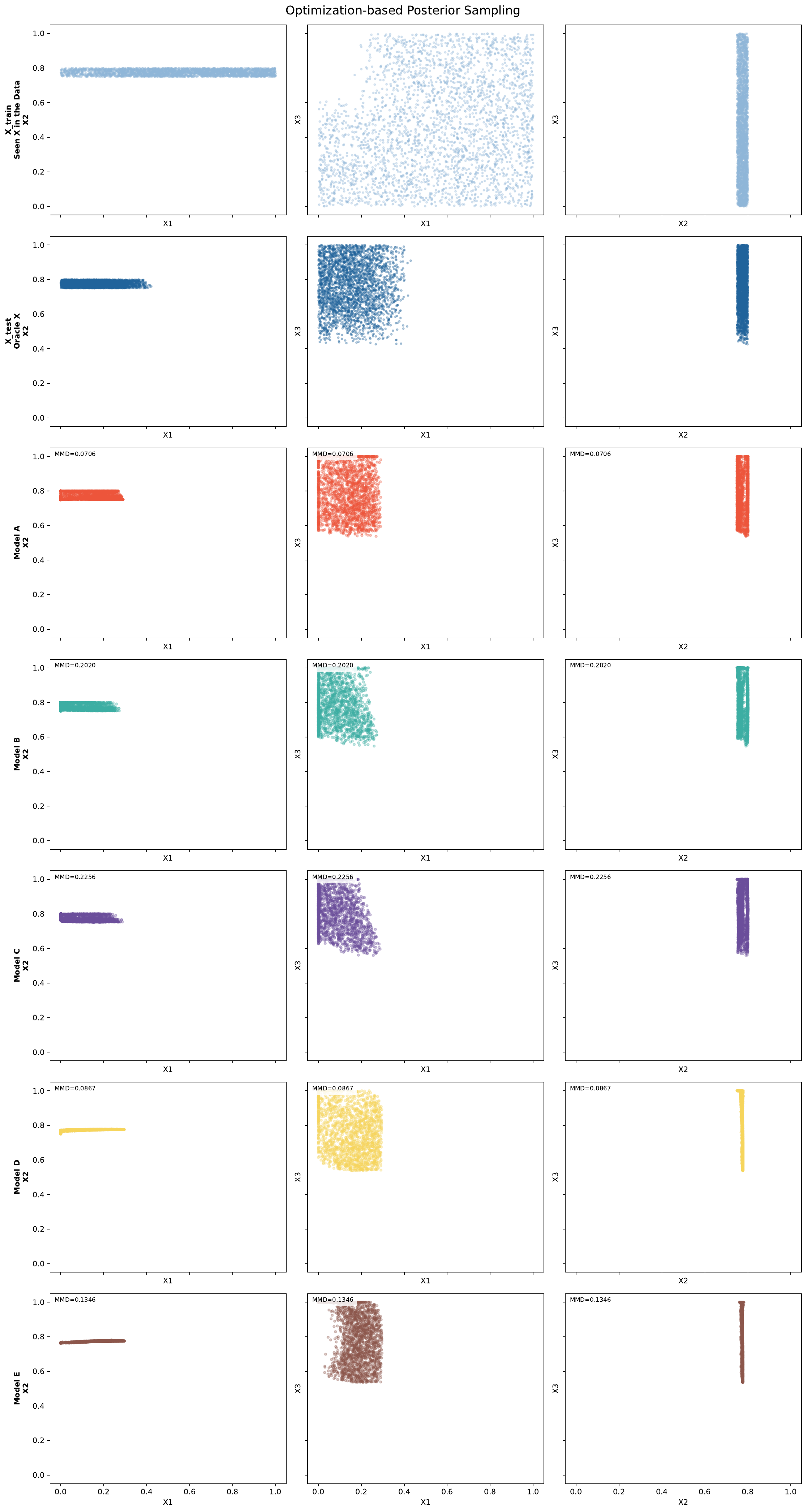}
    \caption{Complete 2-d view of the extrapolated data points compared with baselines (model B, C, D, E) on optimization-based generation method}
    \label{fig: app complete given 2d optimization}
\end{figure}

\begin{figure}
    \centering
    \includegraphics[width=0.7\linewidth]{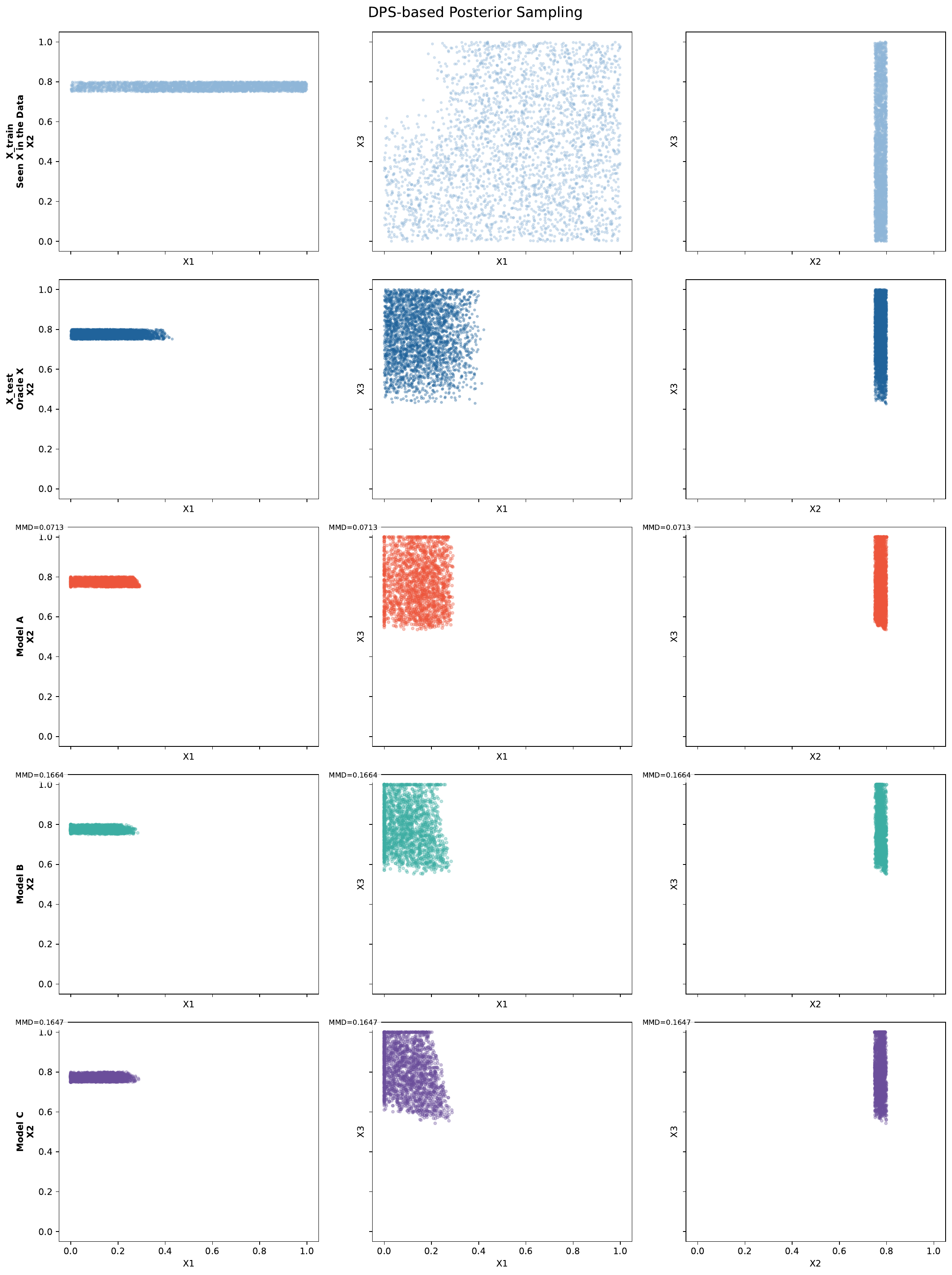}
    \caption{Complete 2-d view of the extrapolated data points compared with baselines (model B, C) on diffusion posterior sampling method}
    \label{fig: app complete given 2d dps}
\end{figure}

\subsubsection{Identification Results of Features and Specifications}

We provide preliminary results on the identification of $\mathbf{X}$ and $\mathbf{Z}$ with sparsity constraints. The experimental setting is straightforward: we follow the same structure as Fig.~\ref{fig:setting_toy_two_specifications}(b) to generate $\mathbf{X}$ and $\mathbf{Z}$, assuming linear gaussian model for the latent $\mathbf{X}$ and $\mathbf{Z}$. Then, we apply linear mixture to the generated $\mathbf{X}$ and $\mathbf{Z}$ to produce the actual observations $\mathbf{Y_X}$ and $\mathbf{Y_Z}$, respectively. The parameters for the latent model and the data split are the same as what we used in the synthetic experiment in the main content. To generate the generated data $\mathbf{Y}_X$ and $\mathbf{Y}_Z$, of which both the dimensions are set to $20$, the linear mixing transformations for is uniformly sampled from $\mathcal{U}(0,1)$. Finally, we follow the estimation procedure in Section~\ref{sec: method} and obtain the estimated $\mathbf{X}$ and $\mathbf{Z}$. We demonstrate the results by enforcing a sparsity strength of $\beta = 5e^{-4}$ and a likelihood constraint with weight $\alpha = 0.05$. See the scatter plots in Figure~\ref{fig: app Z identify} and Figure~\ref{fig: app X identify}. 

\begin{figure}[h]
    \centering
    \includegraphics[width=0.5\linewidth]{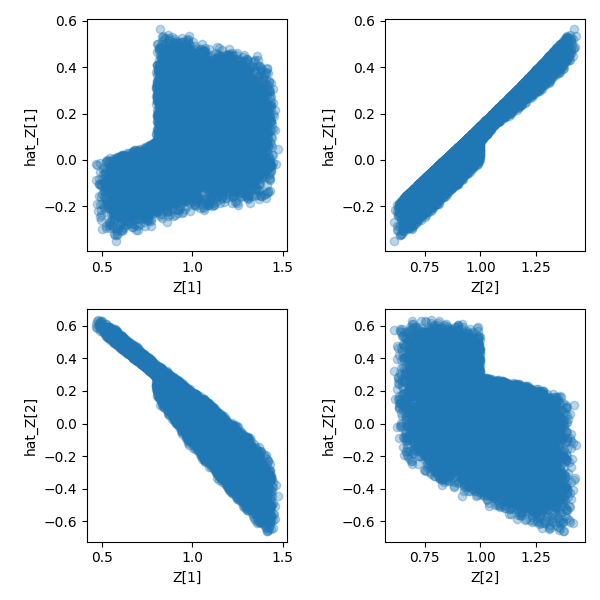}
    \caption{Identification of Specifications ($Z[0]$ and $Z[1]$ corresponds to $Z1$ and $Z2$, respectively.)}
    \label{fig: app Z identify}
\end{figure}

\begin{figure}[h]
    \centering
    \includegraphics[width=0.5\linewidth]{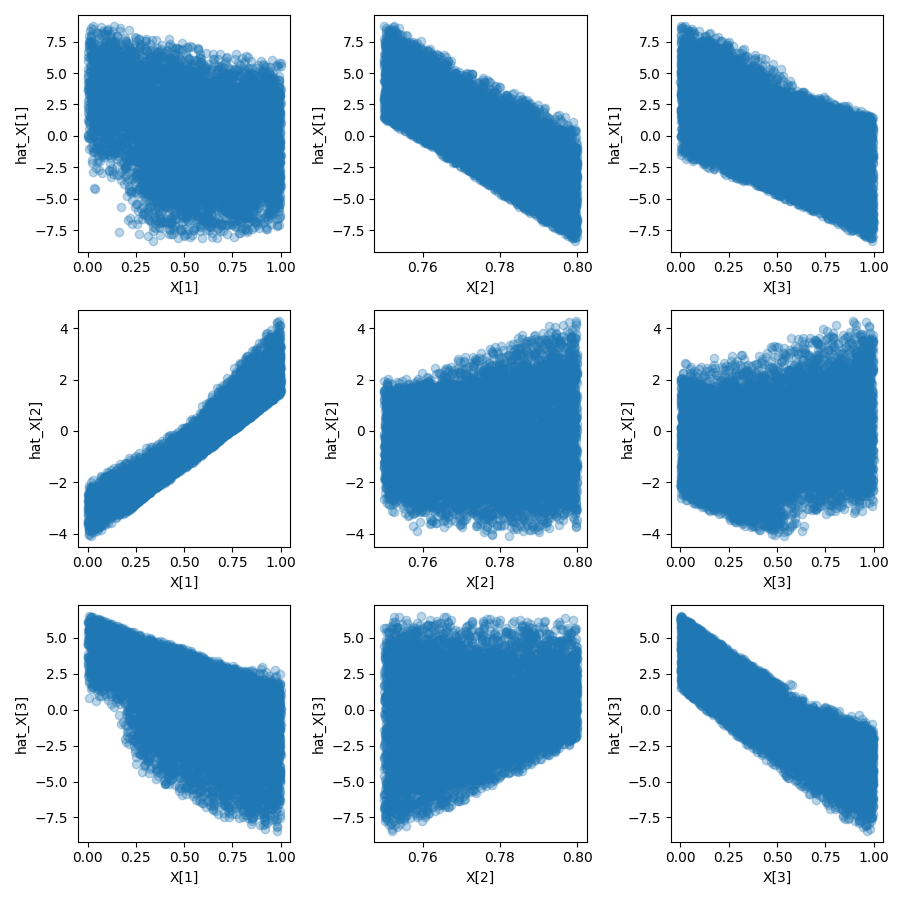}
    \caption{Identification of Features ($X[0]$, $X[1]$, and $X[2]$ correspond to $X1$, $X2$, and $X3$, respectively.)}
    \label{fig: app X identify}
\end{figure}

The identifiability of both features and specifications is more universally verified by averaging the Mean Correlation Coefficient (MCC) between the estimated variables and the ground truth variables on multiple runs. In Table~\ref{tab: app mcc}, we collect the averaged MCC scores for $Z_i;i=1,2$ and $X_i;i=1,2,3$ from 10 runs.  

\begin{table}[h]
\centering
\caption{MCC scores of features and specifications for verification of identifiability}
\begin{tabular}{ccccc}
\hline
$X_1$ & $X_2$ & $X_3$ & $Z_1$ & $Z_2$ \\
\hline
0.4346 & 0.9017 & 0.7097 & 0.9088 & 0.9393 \\
\hline
\end{tabular}
\label{tab: app mcc}  
\end{table}

Additionally, in Table~\ref{tab: app mmd}, we compare the distribution of generated data using the recovered features and specifications with that based on the oracle features in terms of the Maximum Mean Discrepancy (MMD)~\cite{NIPS2006_e9fb2eda} score. The lower the MMD, the closer the two distributions. One can see the distribution of extrapolated data with the recovered features is comparable to that with the oracle features.

\begin{table}[h]
\centering
\caption{MMD scores of extrapolated generation by the two proposed methods}
\begin{tabular}{cc}
\hline
MMD (OPT) & MMD (DPS) \\
\hline
$0.0332 \pm 0.0024$ & $0.0346 \pm 0.0029$ \\
\hline
\end{tabular}
\label{tab: app mmd}  
\end{table}

\subsubsection{Regularization on Optimization-based Sampling}

In practice, we regularize the optimization by aligning the marginal distributions of generated features with those observed in the training data. Although constrained optimization can sometimes produce implausible samples, imposing a joint prior may instead harm extrapolation performance. Marginal constraints therefore offer a practical and effective compromise. 

For Diffusion Posterior Sampling (DPS), however, such regularization is unnecessary. At each step, the sampling procedure incorporates prior guidance that effectively acts as a regularizer, while still allowing flexibility through the likelihood score. As a result, it avoids overly strict adherence to the joint prior distribution.


\subsubsection{Robustness Experiments Under Assumption Violations}
\label{app:robustness}

This section summarizes the simulation studies evaluating the robustness of \textsc{SEDGE}
when the structural assumptions discussed in Appendix~\ref{app:discussion} are violated.
All results are reported in terms of the MMD between the oracle extrapolated data distribution and the actual one, where lower values indicate better performance.
The two generation strategies compared throughout are 
\textbf{OPT} (structure-informed likelihood optimisation) and
\textbf{DPS} (diffusion posterior sampling, guidance $= 0.5$).

\paragraph{Robustness to the direction of causality (\texorpdfstring{$X \to Z$}{X→Z}
            vs.\ \texorpdfstring{$Z \to X$}{Z→X}).}
\label{app:robust_direction}

We constructed four synthetic settings to evaluate the effect of reversing the causal
direction between features and specifications:
\textbf{(A)}~Original setting ($X \to Z$);
\textbf{(B)}~$Z \to X$ with no shared features;
\textbf{(C)}~$Z \to X$ with shared features. The extrapolated generation performance, measured by MMD, is reported in Table~\ref{tab:direction_violation}. 
(In addition, the setting \textbf{(D)}~$Z_1 \to Z_2$ will  be considered separately below.)

\begin{table}[h]
  \centering
  \caption{%
    MMD under causal-direction violations.
    (A)~is the original setting; (B) and (C) reverse the direction to $Z \to X$,
    without and with shared features, respectively.
  }
  \label{tab:direction_violation}
  \small
  \begin{tabular}{lccc}
    \toprule
    \textbf{Method}
      & \textbf{(A) Original ($X \to Z$)}
      & \textbf{(B) $Z \to X$, no shared}
      & \textbf{(C) $Z \to X$, shared} \\
    \midrule
    OPT & $0.0394 \pm 0.0068$ & $0.0389 \pm 0.0090$ & $0.0588 \pm 0.0106$ \\
    DPS & $0.0362 \pm 0.0024$ & $0.0405 \pm 0.0122$ & $0.0856 \pm 0.0247$ \\
    \bottomrule
  \end{tabular}
\end{table}

When there are no shared features, the performance under $Z \to X$ remains close to the original
setting~(A).
When shared features are present~(C), performance degrades more noticeably, consistent
with the theoretical prediction that reliable extrapolation under $Z \to X$ with shared
features is fundamentally harder.
Critically, the method does not fail catastrophically even in this more challenging case,
illustrating graceful degradation.

\paragraph{Robustness to violations of conditional independence among
            \texorpdfstring{$Z$}{Z}.}
\label{app:robust_condindep}

\textit{\underline{Direct violation: introducing \texorpdfstring{$Z_1 \to Z_2$}{Z1→Z2} edges.}}
We tested the performance when a direct link $Z_1 \to Z_2$ is added to the
data-generating process, as well as when a fully dense graph among $Z$ variables is used. The comparison result is reported in Table~\ref{tab:condindep_violation}.

\begin{table}[h]
  \centering
  \caption{%
    MMD under violations of conditional independence among $Z$.
    ``Dense'' refers to a fully connected graph among $Z$ variables.
  }
  \label{tab:condindep_violation}
  \small
  \begin{tabular}{lccc}
    \toprule
    \textbf{Method}
      & \textbf{Original}
      & \textbf{$Z_1 \to Z_2$ added}
      & \textbf{Dense graph} \\
    \midrule
    OPT & $0.0394 \pm 0.0068$ & $0.1033 \pm 0.0150$ & $0.0595 \pm 0.0132$ \\
    DPS & $0.0362 \pm 0.0024$ & $0.0973 \pm 0.0002$ & $0.0738 \pm 0.0243$ \\
    \bottomrule
  \end{tabular}
\end{table}

Violating conditional independence causes noticeable degradation, but both methods still
produce reasonable samples and remain stable even in the fully dense setting.

\paragraph{Ablation stduies: removing conditional independence and sparsity objectives from the model.}
To isolate the contribution of each structural property, we evaluate four model variants, with the result given in Table \ref{tab:ablation_variants}.  One can see that removing sparsity only causes minor degradation, while removing conditional independence
hurts more substantially.
This confirms that the factorised conditional structure is the more critical of the two
inductive biases, though neither causes failure when removed.

\begin{table}[h]
  \centering
  \caption{%
    Ablation study: DPS MMD for four model variants removing conditional independence
    (Cond.\ Indep.) and/or sparsity from the estimated model.
  }
  \label{tab:ablation_variants}
  \small
  \begin{tabular}{lc}
    \toprule
    \textbf{Model variant} & \textbf{DPS MMD} \\
    \midrule
    Full model                         & $0.0468 \pm 0.0133$ \\
    Without sparsity                   & $0.0472 \pm 0.0151$ \\
    Without cond.\ independence        & $0.0690 \pm 0.0237$ \\
    Without sparsity or cond.\ indep.\ & $0.0699 \pm 0.0240$ \\
    \bottomrule
  \end{tabular}
\end{table}

\paragraph{Robustness to violations of the conservative approximation (Assumption~3).}
\label{app:robust_conservative}

We constructed five synthetic settings with increasing degrees of violation of the
conservative assumption (Assumption~3):
\textbf{Test~0}~(baseline): conservative setting where all marginals and joint support are
well covered;
\textbf{Test~1}: expanded joint support beyond the observed region, while keeping all
marginals covered;
\textbf{Test~2}: partial missing support in marginal $p(Z_1)$;
\textbf{Test~3}: partial missing support in marginal $p(Z_2)$;
\textbf{Test~4}: partial missing support in both marginals $p(Z_1)$ and $p(Z_2)$. The comparison results are presented in Table \ref{tab:conservative_violation}. 

\begin{table}[h]
  \centering
  \caption{%
    MMD across five settings of increasing violation of the conservative assumption
    (Assumption~3).
    Tests 0--4 correspond to increasing severity of violation.
  }
  \label{tab:conservative_violation}
  \small
  \begin{tabular}{lccccc}
    \toprule
    \textbf{Method}
      & \textbf{Test 0}
      & \textbf{Test 1}
      & \textbf{Test 2}
      & \textbf{Test 3}
      & \textbf{Test 4} \\
    \midrule
    OPT
      & $0.0554 \pm 0.0269$
      & $0.0490 \pm 0.0230$
      & $0.0510 \pm 0.0169$
      & $0.0551 \pm 0.0253$
      & $0.0425 \pm 0.0164$ \\
    DPS
      & $0.0431 \pm 0.0121$
      & $0.0578 \pm 0.0207$
      & $0.0672 \pm 0.0207$
      & $0.0570 \pm 0.0204$
      & $0.0668 \pm 0.0178$ \\
    \bottomrule
  \end{tabular}
\end{table}

They suggest that the performance degrades gradually rather than abruptly as violations become stronger.
OPT is generally more robust under stronger violations, likely because it explores the
solution space more freely, whereas DPS is partially constrained by the prior learned from
the training joint distribution.
These results confirm that the method retains practical robustness even outside the
identifiable regime guaranteed by Assumption~3.

\paragraph{Robustness to incorrect structure estimation (misspecified
            \texorpdfstring{$X$--$Z$}{X-Z} relations).}
\label{app:robust_structure}

In practice, the parent sets are unknown and must be estimated.
We evaluated robustness to incorrect structure estimation by testing four misspecified
$X$--$Z$ structures:
(i)~removing a true edge ($X_1 \to Z_1$);
(ii)~adding a spurious edge ($X_1 \to Z_2$);
(iii)~both simultaneously; and
(iv)~the correct structure (original). The comparison results are given in Table \ref{tab:misspecified_structure}.

\begin{table}[h]
  \centering
  \caption{%
    MMD under misspecified $X$--$Z$ graph structures.
  }
  \label{tab:misspecified_structure}
  \small
  \begin{tabular}{lcccc}
    \toprule
    \textbf{Method}
      & \textbf{Original}
      & \textbf{Remove $X_1 \!\to\! Z_1$}
      & \textbf{Add $X_1 \!\to\! Z_2$}
      & \textbf{Remove + Add} \\
    \midrule
    OPT & $0.0346 \pm 0.0090$ & $0.0734 \pm 0.0150$ & $0.0385 \pm 0.0096$ & $0.0505 \pm 0.0133$ \\
    DPS & $0.0374 \pm 0.0048$ & $0.0380 \pm 0.0040$ & $0.0370 \pm 0.0049$ & $0.0381 \pm 0.0045$ \\
    \bottomrule
  \end{tabular}
\end{table}

One can see that the method is quite robust to extra (spurious) edges: even with a denser likelihood model,
the estimator can down-weight irrelevant inputs internally.
In contrast, missing true edges is more harmful, as it limits the model's ability to
represent underlying dependencies.
Despite imperfect structure recovery, MMD remains low in most settings.




\paragraph{Sensitivity to concept-extractor pipeline and prior model choices.}
\label{app:robust_latent}

The theoretical results assume access to the true feature and specification variables $X$
and $Z$.
In practice these are estimated using VAEs and normalising flows.
We evaluate robustness to errors in latent representations by comparing different extractors, including the deterministic AE, frozen extractor (pretrained VAE extractor kept fixed during likelihood training), and linear extractor,
and prior choices, including the DDPM prior and  Flow Matching prior. Table \ref{tab:extractor_prior} reports the MMD by the two methods, OPT and DPS.

\begin{table}[h]
  \centering
  \caption{%
    The MMD by OPT and DPS for different extractor and prior choices.
    ``Frozen extractor'' denotes a pretrained VAE extractor kept fixed during
    likelihood training.
  }
  \label{tab:extractor_prior}
  \small
  \begin{tabular}{lcc}
    \toprule
    \textbf{Model configuration} & \textbf{OPT MMD} & \textbf{DPS MMD} \\
    \midrule
    Deterministic AE              & $0.0295 \pm 0.0037$ & $0.0345 \pm 0.0031$ \\
    Frozen extractor (pretrained VAE) & $0.0281 \pm 0.0032$ & $0.0316 \pm 0.0073$ \\
    Linear extractor              & $0.0806 \pm 0.0048$ & $0.0557 \pm 0.0078$ \\
    \midrule
    DDPM prior                    & $0.0307 \pm 0.0040$ & $0.0351 \pm 0.0037$ \\
    Flow Matching prior           & $0.0307 \pm 0.0040$ & $0.0310 \pm 0.0026$ \\
    \bottomrule
  \end{tabular}
\end{table}

The results suggest that the method is robust across extractors and prior choices as long as these components are sufficiently expressive. Performance degrades with the linear extractor, which is insufficiently expressive to capture the latent structure, but all nonlinear variants yield low MMD. The framework is agnostic to the specific concept-extraction pipeline and prior model; these components mainly affect latent quality rather than the structural extrapolation mechanism itself.

\subsection{Extrapolated Text-to-Image Generation}
\label{app:experiment_results}


In this section, we first describe the experimental setup, including implementation details and datasets. 

\subsubsection{Experiment Setup}

As mentioend in the main paper, we use Diffusion-based generation (i.e. diffusion posterior sampling, DPS) here since images (even image concepts) are of high dimensions. Therefore, we need three components: a likelihood model, a causal mask, and a prior model.

To give an initial implementation of SEDGE method in text-to-image generation task, we use a simple model for this qualitative experiment.
We want to do extrapolation for image, so we define the image concepts ($z^I$) as the feature variables and text concepts ($z^T$) as the specification ones. We assume the concepts learned by Aligner~\citep{conceptaligner} are the ground-truth concepts, so the task fall into Scenario 1 in Synthetic experiment, i.e. specifications and features are given. However, we do not have the true causal graph from features to specifications, since the original causal relations learned by Aligner is inverse to our setting. So we design a model to train a likelihood predictor while leanring the causal relations from image concepts to text concepts. We use $N_I$ and $N_T$ as the number of image and text concepts, respectively.

\paragraph{Likelihood Predictor.} 
To model the conditional dependency between image and text concepts, we use a concept-wise conditional flow model to estimate $p^\mathcal{D}(\mathbf{Z} \,|\, \mathbf{X})=p(\mathbf{z}^T\mid \mathbf{z}^I).$ For each target text concept $z^T_i$, the model employs a learnable soft-mask $\mathbf{M}_i\in[0,1]^{N_I}$ to aggregate its parent image concepts $z^I$ (denoted as $\mathbf{V}_i$ in the main paper). This masked context is then encoded via an MLP to parameterize a conditional normalizing flow, which computes the log-likelihood of $z^T_i$, 
$$\log p(z^T_i\mid z^I) = \log p_u(f_i(z^T_i;c_i)) + \log \left\vert \det \frac{\partial f_i}{\partial z^T_i}\right\vert,$$
where $c_i=\operatorname{MLP}_i(z^I\odot \mathbf{M}_i)$, $p_u$ represents a standard Gaussian distribution $\mathcal{N}(0, \mathbf{I})$.
The entire predictor is optimized by minimizing the total negative log-likelihood (NLL) across all $N_T$ text concepts.

\paragraph{Causal Mask.}
As mentioned above, we construct a learnable mask $\mathbf{M}\in\mathbb{R}^{N_T\times N_I}$, where $\mathbf{M}_{ij}=\sigma(w_{ij})$, $\sigma(\cdot)$ means sigmoid function. We apply the sparsity constraint to learn this mask, i.e. $\mathcal{L}_s = \sum_{ij}\Vert \mathbf{M}_{ij}\Vert_1$.

\paragraph{Prior Model.}
We employ an unconditional Diffusion Transformer (DiT) to model the prior distribution of image concepts $p^\mathcal{D}(z^I)$. This model treats the concept set $\mathbf{Z}^I\in\mathbb{R}^{N_I\times D_I}$ as a sequence of tokens ($D_I$ is image concepts dimension). 

\paragraph{Backbone.}
For the generation model from image concepts $z^I$ to true images, we implement our method based on SANA~\citep{xie2024sana}, similar to Aligner~\citep{conceptaligner}. 

\paragraph{Datasets.} We use the dataset similar to Aligner, since we want to keep the leanred latnet concepts meaningful. We exclude the EMU-Edit dataset.

\subsubsection{More Experiment Results}

We show some more results of our SEDGE method on image generation, compared to SANA and Aligner, as shown in Fig.~\ref{app_fig:more_qual}. Also, to show the generation consistency of our model, we also include Fig.~\ref{app_fig:camera-aligner-vs-ours} for analysis.

\begin{figure}[t]
  \centering
  \setlength{\tabcolsep}{3pt}      
  \renewcommand{\arraystretch}{1}  

  \newcommand{\imgw}{0.18\linewidth}

  \begin{tabular}{@{}c c c c c @{}}
    & \shortstack{\scriptsize\texttt{a bird is}\\\scriptsize\texttt{eating ice cream}}
     & \shortstack{\scriptsize\texttt{a crocodile with }\\\scriptsize\texttt{a lion's mane}}
     & \shortstack{\scriptsize\texttt{a dolphin with}\\\scriptsize\texttt{ eagle wings}}
     & \shortstack{\scriptsize\texttt{a shopping cart }\\\scriptsize\texttt{with tank treads}}
 \\
    \addlinespace[2pt]

    \raisebox{1.1\height}{\rotatebox{90}{\textbf{SANA}}} &
    \includegraphics[width=\imgw]{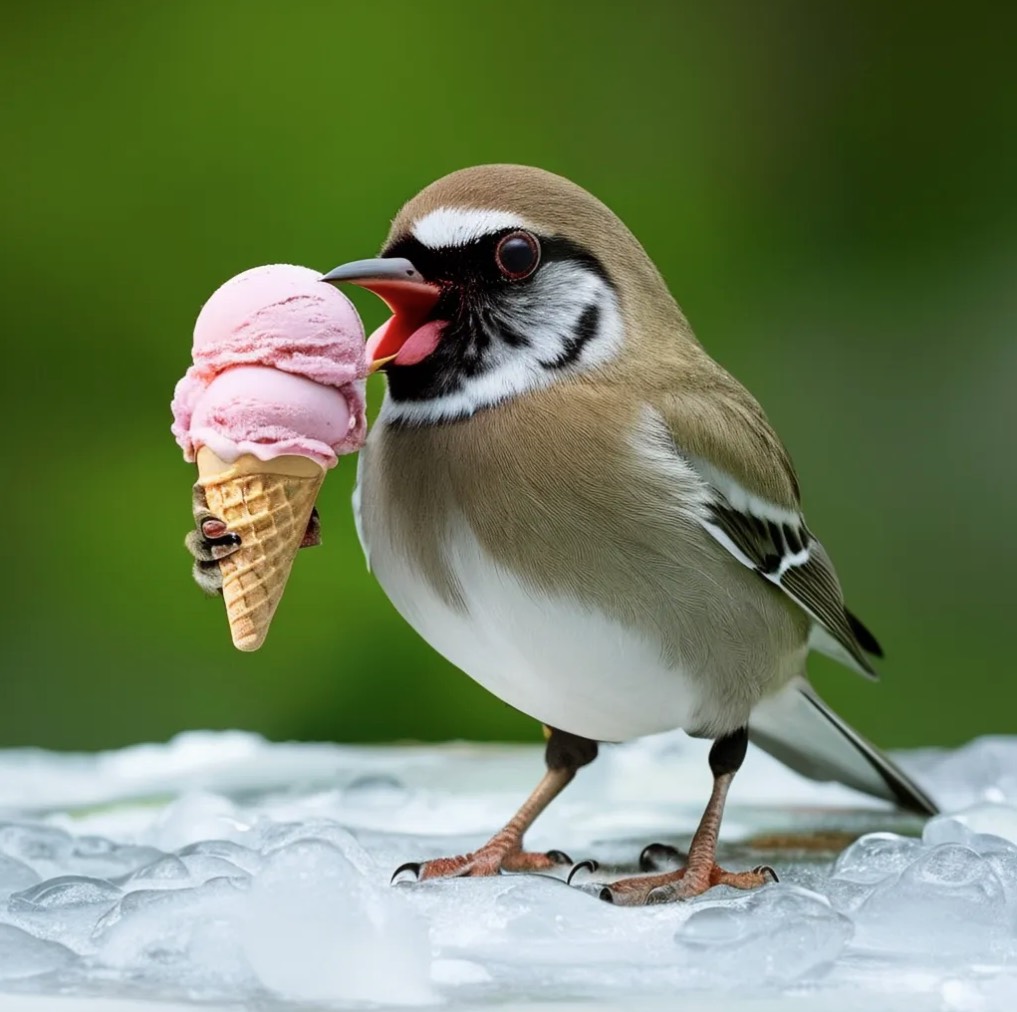} &
    \includegraphics[width=\imgw]{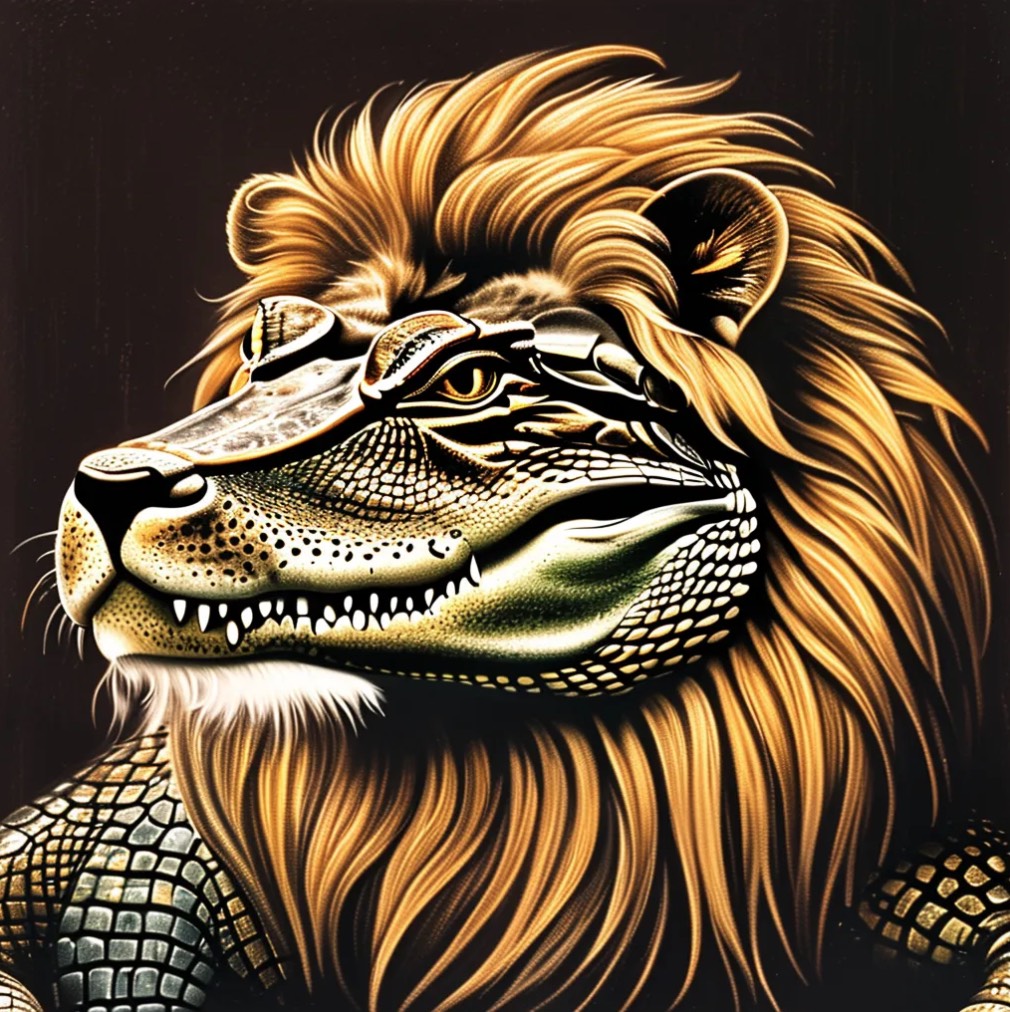} &
    \includegraphics[width=\imgw]{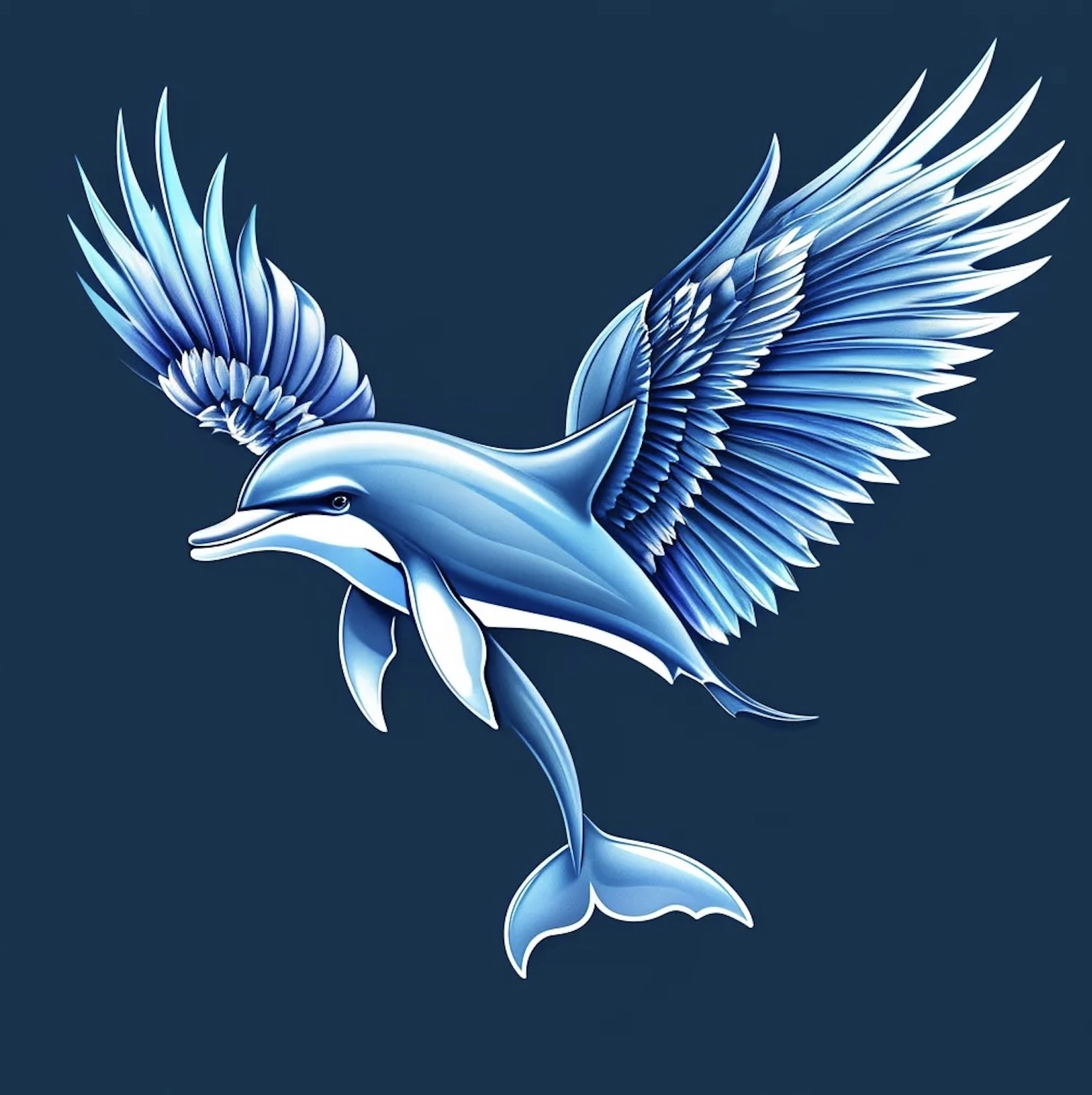} & \includegraphics[width=\imgw]{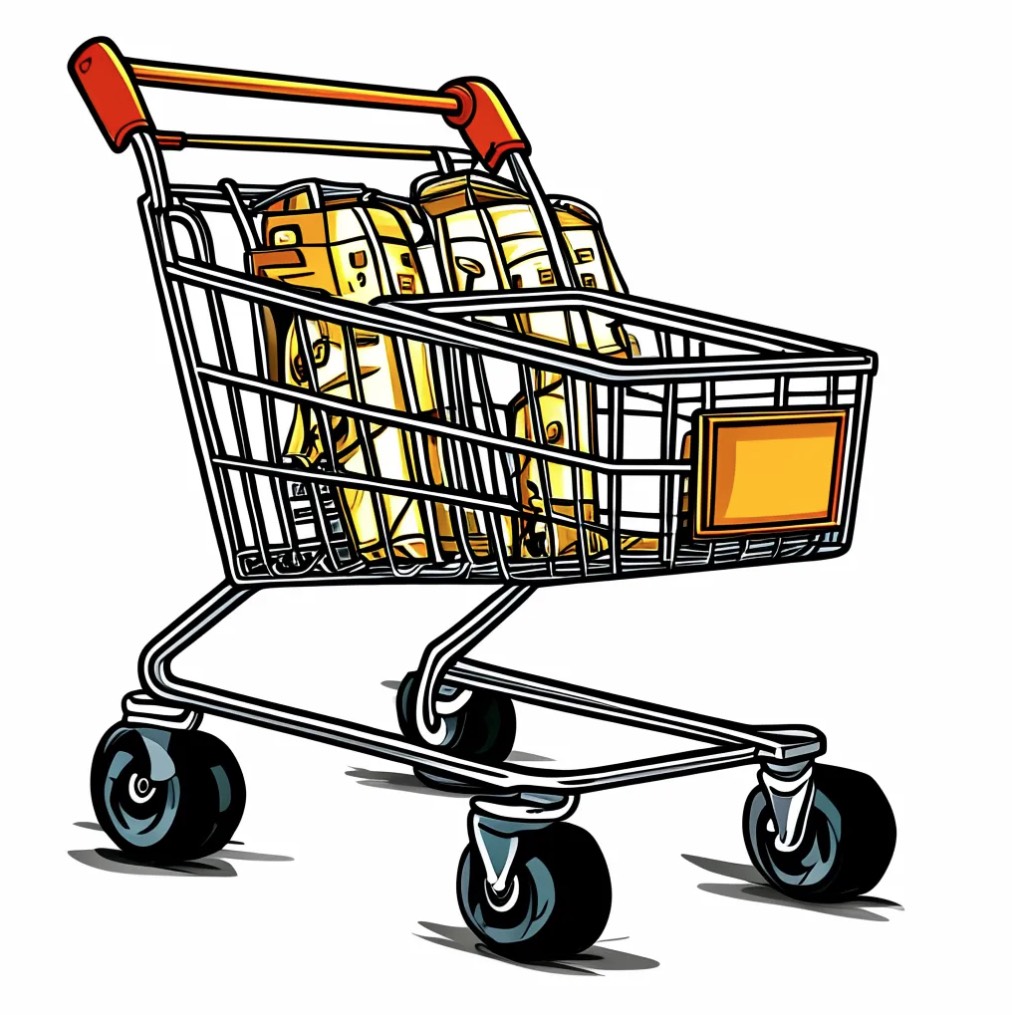} 
    \\

    \raisebox{0.8\height}{\rotatebox{90}{\textbf{Aligner}}} &
    \includegraphics[width=\imgw]{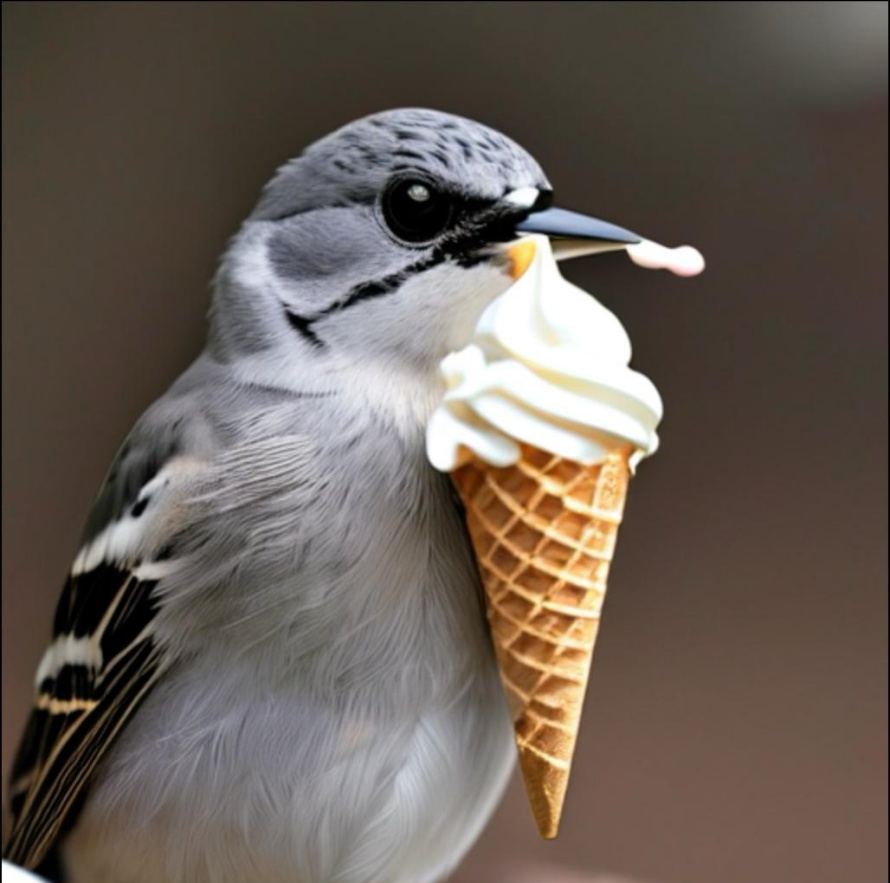} &
    \includegraphics[width=\imgw]{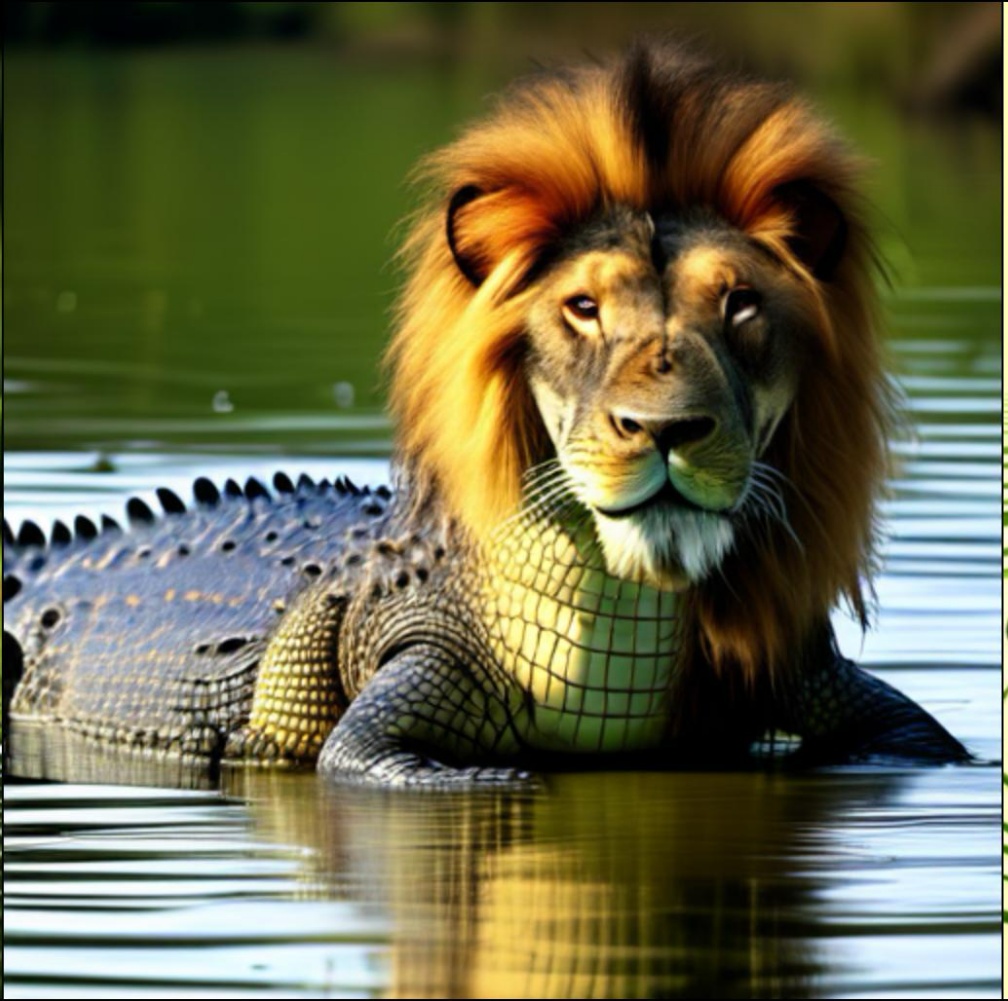} &
    \includegraphics[width=\imgw]{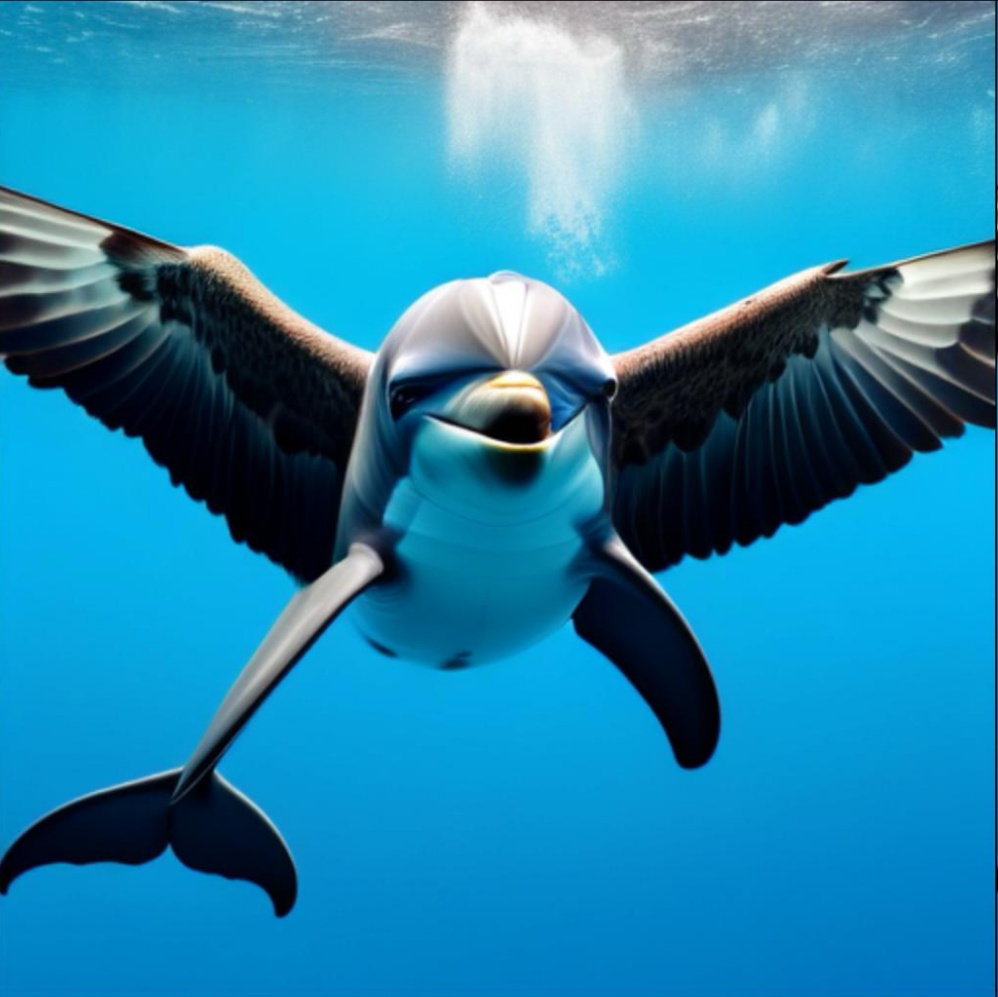} & \includegraphics[width=\imgw]{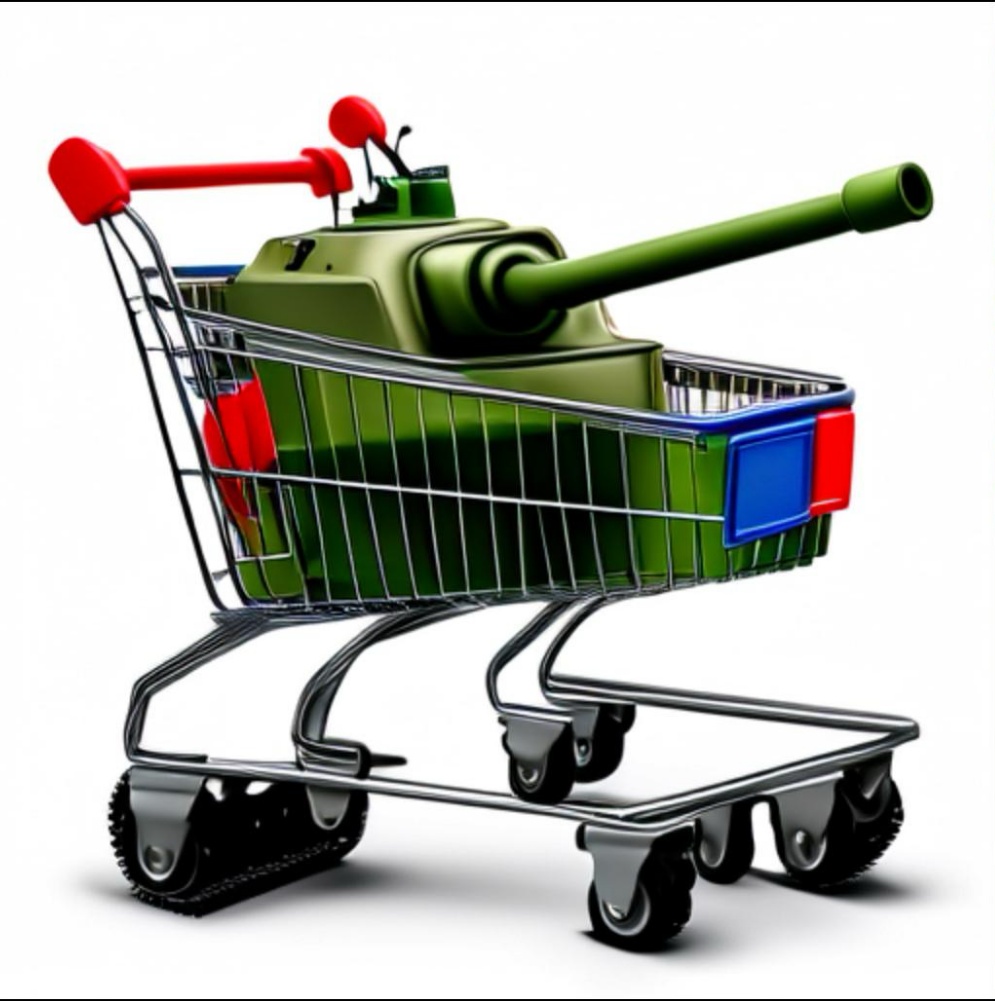} 
    \\

    \raisebox{0.8\height}{\rotatebox{90}{\textbf{SEDGE}}} &
    \includegraphics[width=\imgw]{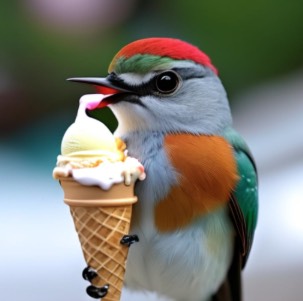} &
    \includegraphics[width=\imgw]{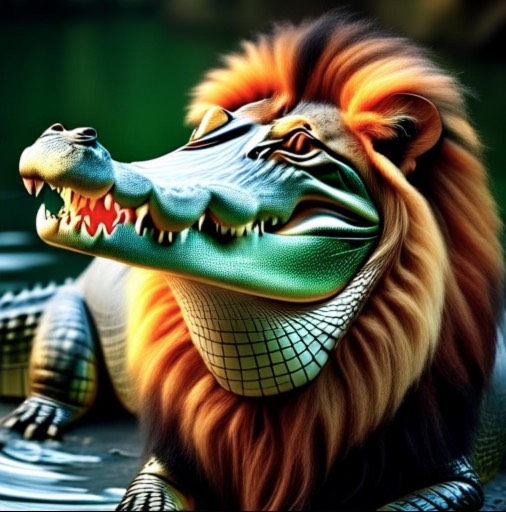} &
    \includegraphics[width=\imgw]{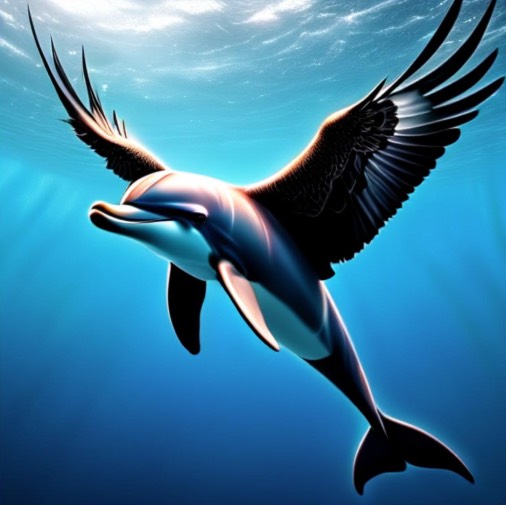} & \includegraphics[width=\imgw]{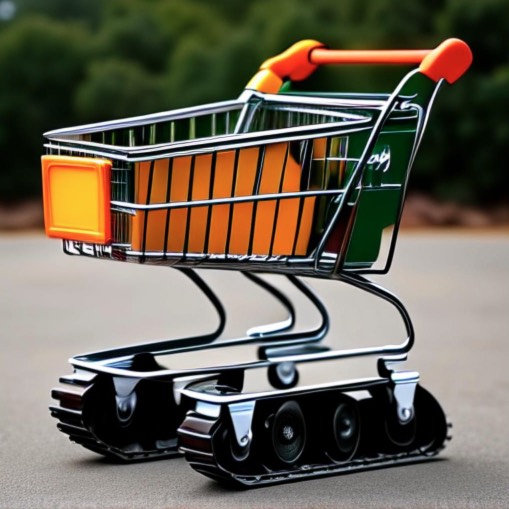} 
    \\
  \end{tabular}

  \caption{More qualitative comparisons under different prompts. }
  \label{app_fig:more_qual}
  \vskip -0.1in
\end{figure}

\begin{figure}[ht]
  \vskip 0.2in
  \centering

  \setlength{\tabcolsep}{8pt}   
  \renewcommand{\arraystretch}{1}

  \begin{tabular}{c c}
    {\bfseries Aligner} & {\bfseries Our results across different guidance values ($\eta\in\{0.1, 0.3, 0.5, 0.7, 1.0, 1.5\}$)} \\
    \addlinespace[3pt]

    \includegraphics[width=0.118\columnwidth]{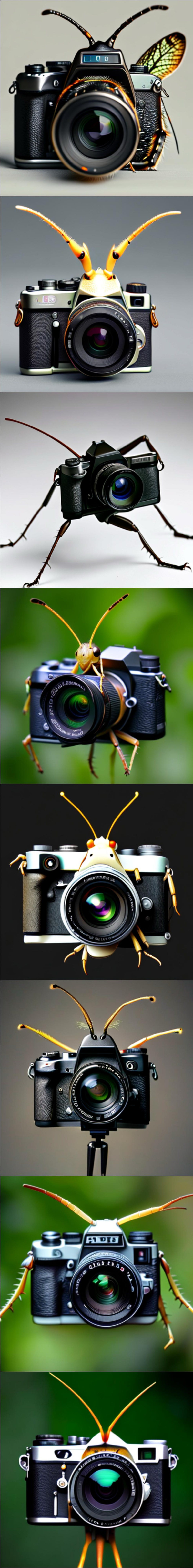}
    &
    \includegraphics[width=0.72\columnwidth]{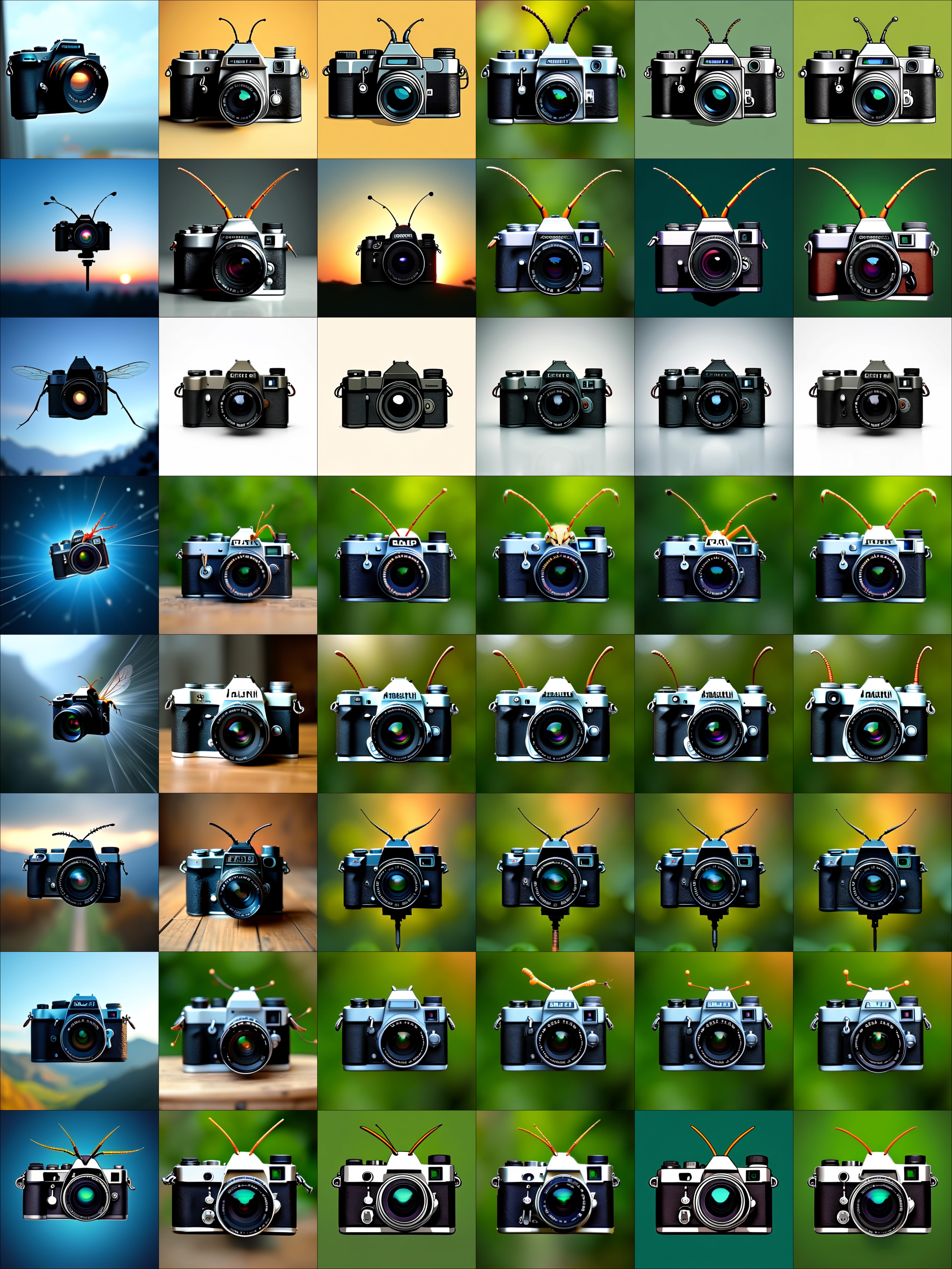}
  \end{tabular}

  \caption{Consistency of our generation. Given prompt \texttt{a camera with insect atennae}, we generated figures using both ours and concept aligner models acorss $8$ different seeds. As we can see in the figure, the generation of concept aligner has mixed up concepts like legs and bodies. Our generation results (especially with larger guidance values) makes consistent generation.}
  \label{app_fig:camera-aligner-vs-ours}
\end{figure}

\subsubsection{Quantitative Evaluation Results}
\label{app: T2I quantitative evalaution}
We constructed a new 75-prompt benchmark for extrapolative T2I for quantitative evaluation, and verified there is no significant train-test overlap via prompt similarity: full-Composition Coverage is 0 and Composition Novelty Score is 1, confirming genuine extrapolation.
We then compared SEDGE with SANA and a retrained Aligner using GPT-4o alignment scoring (higher is better). 

\begin{table}[h]
\centering
\caption{Model alignment comparison}
\begin{tabular}{lccc}
\hline
\textbf{Models} & \textbf{SANA} & \textbf{Aligner} & \textbf{Ours} \\
\hline
GPT-4o alignment & 51.9\% & 50.0\% & 53.0\% \\
\hline
\end{tabular}

\label{tab:alignment}
\end{table}



\end{document}